 \documentclass[10pt,journal]{IEEEtran}
\usepackage[ruled,linesnumbered,lined]{algorithm2e}
\usepackage{graphicx}
\usepackage{amsmath}
\usepackage{amsthm}

\usepackage{amssymb}
\usepackage{booktabs}

\usepackage{algorithmic}

\allowdisplaybreaks[4]
\usepackage{colortbl}
\usepackage{xcolor}
\usepackage{color}
\usepackage{framed}
\usepackage{enumitem}

\usepackage[pagebackref,breaklinks,colorlinks]{hyperref}

\usepackage{caption}
\usepackage{subcaption}
\def\ie{\emph{i.e.}}

\def\1{\bm{1}}

\DeclareMathAlphabet{\mathsfit}{\encodingdefault}{\sfdefault}{m}{sl}
\SetMathAlphabet{\mathsfit}{bold}{\encodingdefault}{\sfdefault}{bx}{n}

\def\sR{{\mathbb{R}}}

\newcommand{\R}{\mathbb{R}}

\DeclareMathOperator*{\argmax}{arg\,max}

\newcommand{\norm}[3][2]{\Vert #3 \Vert^{#1}_{#2}}
\newcommand{\od}[2]{#1 \odot #2}
\newcommand{\ip}[2]{\langle #1 , #2 \rangle}
\newcommand{\EE}[1][]{ \mathbb{E}_{#1}}

\newcommand{\defaultcolor}{\color{black}}

\newcommand{\pgreen}[1]{\color{green} #1 \defaultcolor}
\def\plemmalabelflag{0}
\newcommand{\plabel}[2][\plemmalabelflag]{
    \ifnum #1=1  
        \pgreen{#2}
    \else
    \fi
}

\newtheorem{theorem}{Theorem}

\newtheorem{lemma}{Lemma}
\newtheorem{assumption}{Assumption}
\newtheorem{remark}{Remark}
\newtheorem{corollary}{Corollary}

\ifCLASSOPTIONcompsoc
  \usepackage[nocompress]{cite}
\else
  \usepackage{cite}
\fi
\ifCLASSINFOpdf
\else
\fi
\usepackage{algorithmic}
\usepackage{multirow}
\usepackage{amsmath}

\begin{document}

\title{AdaSAM: Boosting Sharpness-Aware Minimization with Adaptive Learning Rate and Momentum for Training Deep Neural Networks}
%
%
%
%

\author{Hao~Sun,
        Li~Shen,
        Qihuang Zhong,
        Liang Ding, 
        Shixiang~Chen\\
        Jingwei~Sun,
        Jing Li,
        Guangzhong Sun, 
        and Dacheng Tao~\IEEEmembership{Fellow, IEEE}

\IEEEcompsocitemizethanks{
\IEEEcompsocthanksitem Hao Sun, Jingwei Sun, Jing Li and Guangzhong Sun are with School of Computer Science and Technology, University of Science and Technology of China, Hefei, China, 230000.
(E-mail: ustcsh@mail.ustc.edu.cn, sunjw@ustc.edu.cn, lj@ustc.edu.cn, gzsun@ustc.edu.cn.)
\IEEEcompsocthanksitem  Qihuang Zhong is with the School of
Computer Science, Wuhan University, Hubei, 430000.
(E-mail: zhongqihuang@whu.edu.cn)
\IEEEcompsocthanksitem Li Shen, Liang Ding, Shixiang Chen, and Dacheng Tao are with JD Explore Academy, Beijing, 100000.
(E-mail: mathshenli@gmail.com, liangding.liam@gmail.com, chenshxiang@gmail.com,  dacheng.tao@gmail.com.)
}
}

%
%

\markboth{Journal of \LaTeX\ Class Files,~Vol.~14, No.~8, August~2015}%
{Shell \MakeLowercase{\textit{et al.}}: Bare Demo of IEEEtran.cls for Computer Society Journals}
%




\IEEEtitleabstractindextext{%
\begin{abstract}
Sharpness aware minimization (SAM) optimizer has been extensively explored as it can generalize better for training deep neural networks via introducing extra perturbation steps to flatten the landscape of deep learning models.  Integrating SAM with adaptive learning rate and momentum acceleration, dubbed AdaSAM, has already been explored empirically to train large-scale deep neural networks without theoretical guarantee due to the triple difficulties in analyzing the coupled perturbation step, adaptive learning rate and momentum step. In this paper, we try to analyze the convergence rate of AdaSAM in the stochastic non-convex setting. We theoretically show that AdaSAM admits a $\mathcal{O}(1/\sqrt{bT})$ convergence rate, which achieves linear speedup property with respect to mini-batch size $b$. Specifically, to decouple the stochastic gradient steps with the adaptive learning rate and perturbed gradient, we introduce the delayed second-order momentum term to decompose them to make them independent while taking an expectation during the analysis. Then we bound them by showing the adaptive learning rate has a limited range, which makes our analysis feasible. To the best of our knowledge, we are the first to provide the non-trivial convergence rate of SAM with an adaptive learning rate and momentum acceleration. At last, we conduct several experiments on several NLP tasks, which show that AdaSAM could achieve superior performance compared with SGD, AMSGrad, and SAM optimizers. 
\end{abstract}

}

\maketitle

\IEEEdisplaynontitleabstractindextext

%
\IEEEpeerreviewmaketitle

\section{Introduction}

\IEEEPARstart{S}{harpness-aware} minimization (SAM) \cite{foret2020sharpness} is a powerful optimizer for training large-scale deep learning models by explicitly minimizing the gap between the training performance and generalization performance. It has achieved remarkable results in training various deep neural networks, including ResNet \cite{he2016deep,foret2020sharpness,mimake}, vision transformer \cite{dosovitskiy2021an,chen2022vision}, language models \cite{devlin2018bert,he2020deberta,zhong2022improving}, on extensive benchmarks.   

However, SAM-type methods suffer from several issues during training the deep neural networks, especially for huge computation costs and heavily hyper-parameter tuning procedure. In each iteration, SAM needs double gradients computation compared with classic optimizers, like SGD, Adam \cite{DBLP:journals/corr/KingmaB14}, AMSGrad \cite{j.2018on}, due to the extra perturbation step. Hence, SAM requires to forward and back propagate twice for one parameter update, resulting in one more computation cost than the classic optimizers.
Moreover, as there are two steps during the training process, it needs double hyper-parameters, which makes the learning rate tuning unbearable and costly.

Adaptive learning rate optimization methods \cite{DBLP:journals/tcyb/Iiduka22} scale the gradients based on the history gradient information to accelerate the convergence by tuning the learning rate automatically. These methods, such as Adagrad \cite{duchi2011adaptive}, Adam \cite{DBLP:journals/corr/KingmaB14}, and AMSGrad \cite{j.2018on}, have been proposed for solving the computer vision, natural language process, and generative neural networks tasks \cite{DBLP:journals/tcyb/Iiduka22,ruder2016overview,liao2022local,zhang2020adaptive}.
Recently, several works have tried to ease the learning rate tuning in SAM by inheriting the triplet advantages of SAM, adaptive learning rate, and momentum acceleration. For example, \cite{zhuang2022surrogate} and \cite{kwon2021asam} train ViT models and NLP models with adaptive learning rates and momentum acceleration, respectively. Although remarkable performance has been achieved, their convergences are still unknown since the adaptive learning rate and momentum acceleration are used in SAM. Directly analyzing its convergence is complicated and difficult due to the three coupled steps of optimization, i.e.,  the adaptive learning rate estimation is coupled with the momentum step and perturbation step of SAM.

In this paper, we analyze the convergence rate of SAM with an adaptive learning rate and momentum acceleration, dubbed AdaSAM, in the non-convex stochastic setting. To circumvent the difficulty in the analysis, we develop a novel technique to decouple the three-step training of SAM from the adaptive learning rate and momentum step. The analysis procedure is mainly divided into three parts. 
The first part is to analyze the procedure of the SAM. Then we analyze the second step that adopts the adaptive learning rate method.
We introduce a second-order momentum term from the previous iteration, which is related to the adaptive learning rate and independent of SAM while taking an expectation.
Then we can bound the term composed by the SAM and the previous second-order momentum due to the limited adaptive learning rate.
In the last part, we analysis the momentum acceleration that is combined with the SAM and the adaptive learning rate.  The momentum acceleration lead to an extra term in convergence analysis. Here, we introduce an auxiliary sequence to absorb it and show that their summation over the all iterations is controllable.
We prove that AdaSAM enjoys the property of linear speedup property with respect to the batch size, i.e. $\mathcal{O}(1/\sqrt{bT})$ where $b$ is the mini-batch size. 
Empirically, we apply AdaSAM to train RoBERTa model on the GLUE benchmark to evaluate our theoretical findings.  We show that AdaSAM achieves the best performance in experiments, where it wins 6 tasks of 8 tasks, and the linear speedup can be clearly observed.

In the end, we summarize  our contributions as follows:
\begin{itemize}
\item We present the first convergence guarantee of the adaptive SAM method with momentum acceleration under the stochastic non-convex setting. Our results suggest that a large mini-batch can help convergence due to the established linear speedup with respect to batch size.
 
\item We conduct a series of experiments on various tasks. The results show that AdaSAM outperforms most of the state-of-art optimizers and the linear speedup is verified.
\end{itemize}

\section{Preliminary and Related Work}

In this section, we first describe the basic problem setup and then introduce several related works on the SAM, adaptive learning rate and momentum steps. 

\subsection{Problem Setup}
In this work, we focus on stochastic nonconvex optimization 
\begin{equation}\label{problem:opt}
\min_{x \in \sR^{d}} f(x) := \EE[\xi \sim D]{f_{\xi}(x)} ,
\end{equation}
where $d$ is dimension of variable $x$, $D$ is the unknown distribution of the data samples, $f_{\xi}(x)$ is a smooth and possibly non-convex function, and $f_{\xi_i}(x)$ denotes the objective function at the sampled data point $\xi_{i}$ according to data distribution $D$. In machine learning, it covers empirical risk minimization as a special case and $f$ is the loss function when the dataset $D$ cover $N$ data points, i.e., $D = \{\xi_{i}, i = 1,2,\ldots, N\}$. Problem \eqref{problem:opt} reduces to the following finite-sum problem:
\begin{equation}\label{problem:erm}
\min_{x \in \sR^{d}} f(x) := \frac{1}{N}\sum_{i}{f_{\xi_{i}}(x)}.
\end{equation}

\paragraph{Notations.}\ Without additional declaration,  we represent $f_{i}(x)$ as $f_{\xi_{i}}(x)$ for simplification, 
which is the $i$-th loss function while $x \in \sR^{d}$ is the model parameter and $d$ is the parameter dimension.
We denote the $l_2$ norm  as $\norm[]{2}{\cdot}$.
A Hadamard product is denoted as $a \odot b$ where $a$,$b$ are two vectors. For a vector $a \in \sR^{d}$, $\sqrt{a}$ is denoted as a vector that the $j$-th value, $(\sqrt{a})_{(j)}$, is equal to the square  root of $a_j$.

\subsection{Related Work}

\paragraph{Sharpness-aware minimization}
Many works try to improve the generalization ability during training the deep learning model. Some methods such as dropout\cite{DBLP:journals/jmlr/SrivastavaHKSS14}, weight decay \cite{loshchilov2017decoupled}, and regularization methods \cite{DBLP:journals/tcyb/LiZGYX22,DBLP:journals/tcyb/LuZLZLZ22} provide an explicit way to improve generalization. Previous work shows that sharp minima may lead to poor generalization whereas flat minima perform better\cite{Jiang2020Fantastic,DBLP:conf/iclr/KeskarMNST17,he2019asymmetric}. Therefore, it is popular to consider sharpness to be closely related to the generalization.  
Sharpness-aware minimization (SAM) \cite{foret2020sharpness}  targets to find flat minimizers explicitly by minimizing the training loss uniformly in the entire neighborhood. Specifically, SAM aims to solve the following minimax saddle point problem:
\begin{equation}\label{eq:sam:quadratic}
    \min_{x} \max_{\norm[]{}{\delta} \leq \rho} f(x+\delta) + \lambda \norm[2]{2}{x},
\end{equation}
where $\rho\geq 0$ and $\lambda\geq 0$ are two hyperparameters. That is, the perturbed loss function of $f(x)$ in a neighborhood is minimized instead of the original loss function $f(x)$.  By using Taylor expansion of $ f(x+\delta)$ with respect to $\delta$, the inner max problem is approximately solved via  
\begin{align*}
\delta^{*}(x) &= \argmax_{\norm[]{}{\delta} \leq \rho} f(x+\delta ) \\
&\approx \argmax_{\norm[]{}{\delta}\leq \rho} f(x) + \delta^\top \nabla f(x) \\
&= \argmax_{\norm[]{}{\delta} \leq \rho} \delta^\top \nabla f(x) 
= \rho \frac{\nabla f(x)}{\norm[]{}{\nabla f(x)}}.
\end{align*}
By dropping the quadratic term,  \eqref{eq:sam:quadratic} is simplified  as the  following minimization problem 
\begin{equation}\label{eq:sam_surrogate}
    \min_{x} f\left(x+ \rho \frac{\nabla f(x)}{\norm[]{}{\nabla f(x)}}\right). 
\end{equation}
The stochastic gradient of $f\left(x+ \rho \frac{\nabla f(x)}{\norm[]{}{\nabla f(x)}}\right)$ on a batch data $b$ includes the Hessian-vector product, SAM further approximates the gradient by 
\[ \nabla_x f_b\left(x+ \rho \frac{\nabla f_b(x)}{\norm[]{}{\nabla f_b(x)}}\right) \approx \nabla_x f_b(x)|_{x+ \rho \frac{\nabla f_b(x)}{\norm[]{}{\nabla f_b(x)}}}. \]
Then, along the negative direction $-\nabla_x f_b(x)|_{x+ \rho \frac{\nabla f_b(x)}{\norm[]{}{\nabla f_b(x)}}}$, SGD is applied to solve the surrogate minimization problem \eqref{eq:sam_surrogate}. It is easy to see that SAM requires twice gradient back-propagation, i.e., $\nabla f_b(x)$ and $\nabla_x f_b(x)|_{x+ \rho \frac{\nabla f_b(x)}{\norm[]{}{\nabla f_b(x)}}}$. Due to the existence of hyperparameter $\rho,$ one needs to carefully tune both $\rho$ and learning rate in SAM.  In practice, $\rho$ is predefined to control the radius of the neighborhood.

Recently, Several variants of SAM are proposed to improve its performance. For example, \cite{zhuang2022surrogate,kwon2021asam,zhong2022improving} have empirically incorporated adaptive learning rate with SAM and shown impressive generalization accuracy, while their convergence analysis has never been studied. 
ESAM \cite{du2021efficient} proposes an efficient method by sparsifying the gradients to alleviate the double computation cost of backpropagation. 
ASAM \cite{kwon2021asam} modifies SAM by adaptively scaling the neighborhood so that the sharpness is invariant to parameters re-scaling. 
GSAM \cite{zhuang2022surrogate} simultaneously minimizes the perturbed function and a new defined surrogate gap function to further improve the flatness of minimizers. Liu et al. \cite{liu2022towards} also study SAM in large-batch training scenario and periodically update the perturbed gradient. Recently, \cite{mimake,zhong2022improving} improve the efficiency of SAM by adopting the sparse gradient perturbation technique. \cite{qu2022generalized,sunfedspeed} extend SAM to the federated learning setting setting with a significant performance gain.
On the other hand, there are some  works  analyzing the convergence of the SAM such as \cite{andriushchenko2022towards} without considering the normalization step, i.e., the normalization in $ \frac{\nabla f_b(x)}{\norm[]{}{\nabla f_b(x)}}$.

\paragraph{Adaptive optimizer} 
The adaptive optimizer can automatically adjust the learning rate based on the history gradients methods. The first adaptive method, Adagrad \cite{duchi2011adaptive}, can achieve a better result than other first-order methods under the convex setting. While training the deep neural network, Adagrad will decrease the learning rate rapidly with a degraded performance. Adadelta \cite{zeiler2012adadelta} is proposed to change this situation and introduces a learning rate based on the exponential average history gradients.
Adam \cite{DBLP:journals/corr/KingmaB14} additionally adds momentum step to stabilize the training process, and it shows great performance in many tasks. However, Reddi et al \cite{j.2018on} give a counterexample that it cannot converge even when the objective function is convex and propose an alternative method called AMSGrad with convergence guarantee. Then, many works \cite{zhou2018convergence,chen2018convergence,zaheer2018adaptive,ward2019adagrad,DBLP:journals/corr/abs-2003-02395,zou2019sufficient,chen2021towards,chen2021quantized,chen2022towards,chen2022efficient,zou2018weighted,iiduka2021appropriate,sun2019survey,sakai2021riemannian} have been proposed to study the convergence of adaptive methods and their variants in the nonconvex setting. However, their analysis techniques can not directly extend to establish the convergence of SAM with adaptive learning rate due to the  coupled perturbation step and adaptive learning rate.

\paragraph{Momentum acceleration} 
Momentum methods such as Polyak's heavy ball method \cite{polyak1964some}, Nestrov's accelerated gradient descent method \cite{nesterov2003introductory} and accelerated projected method \cite{DBLP:journals/focm/ODonoghueC15} are used to optimize the parameters of the deep neural network.
In practice, they have been used to accelerated for federated learning tasks \cite{liu2020accelerating},  non-negative latent factor model \cite{luo2018fast} and recommender systems \cite{shang2021alpha}. There are many theoretical works \cite{yang2016unified,DBLP:conf/nips/MannelliU21,gao2022global} that focus on analyzing the momentum acceleration for optimizing non-convex problem.  \cite{sutskever2013importance} shows that it is important for tuning momentum while training deep neural network. \cite{DBLP:conf/icml/CanGZ19} first points out linear convergence results for stochastic momentum method. \cite{DBLP:journals/jmlr/HuangGPH22} proposes a class of accelerated zeroth-order and first-order momentum method to solve mini-optimization and minimax-optimization problem.  \cite{bollapragada2022nonlinear} extend the momentum method by introducing an RNA scheme and a constrained formulation RNA which has nonlinear updates.
\cite{o2015adaptive} propose a heuristic adaptive restart method and \cite{DBLP:journals/siamis/WangNSBBO22} propose a scheduled restart momentum accelerated SGD method named SRSGD which helps reduce the training time.  \cite{liu2022convergence} adds one momentum term on to the distributed gradient algorithm.

\section{Methodology}

In this section, we introduce SAM with adaptive learning rate and momentum acceleration, dubbed AdaSAM, to stabilize the training process of SAM and ease the learning rate tuning. Then, we present the convergence results of AdaSAM. At last, we give the proof sketch for the main theorem.  

\begin{algorithm}[h]
\caption{AdaSAM: SAM with adaptive learning rate and momentum acceleration}
\label{alg:SAM-adaptive-final}
\Indentp{0.05em}
\KwIn{Initial parameters $x_{0}$, $m_{-1}=0$, $\hat{v}_{-1}=\epsilon^{2} $(a small
positive scalar to avoid the denominator diminishing), base learning rate $ \gamma $, neighborhood size $\rho$ and momentum parameters $ \beta_{1} $, $\beta_{2}$.}
\KwOut{Optimized parameter $x_{T+1}$}
\Indentp{0.05em}
\For{iteration t $\in$ $\{0, 1, 2, ..., T-1 \}$}{
    Sample mini-batch $ B=\{ \xi_{t_{1}},\xi_{t_{2}},...,\xi_{t_{|B|}} \}$\;
    Compute gradient $s_t=\nabla_{x} f_{B}(x)|_{x_{t}}=\frac{1}{b}\sum_{i\in B} \nabla f_{t_{i}}(x_{t})$\;
    Compute $\delta(x_{t})=\rho_{t} \frac{s_t}{\Vert s_t \Vert}$\;
    Compute SAM gradient $g_{t}=\nabla_{x} f_{B}(x)|_{x_{t} + \delta(x_{t})}$\;
    $m_{t} = \beta_1 m_{t-1} +(1- \beta_1 ) g_{t}$\;
    $v_{t} = \beta_2 v_{t-1} +(1- \beta_2) [g_{t}]^{2}$\;
    $\hat{v}_{t} = \max(\hat{v}_{t-1}, v_{t})$\;
    $\eta_{t}={1}{/}{\sqrt{\hat{v}_{t}}}$\;
    $x_{t+1} = x_{t} - \gamma m_{t} \odot \eta_{t}$\;
}
\end{algorithm}

\subsection{AdaSAM Algorithm}
AdaSAM for solving Problem \eqref{problem:opt} is described in Algorithm \ref{alg:SAM-adaptive-final}. In each iteration,  a mini-batch gradient estimation $g_{t}$ at point $x+\epsilon(x)$ with batchsize $b$ is computed, i.e., 
\[
g_{t} 
= \nabla_{x} f_{b}(x)|_{x_{t} + \epsilon(x_{t})} 
=\frac{1}{b}\sum_{i\in B} \nabla f_{\xi_{i}}(x_t+\delta(x_t)).
\]
Here, $\delta(x_{t})$ is the extra perturbed gradient step in SAM  that is given as  follows
\[
\delta(x_{t}) =\rho \frac{s_t}{\Vert s_t \Vert}, {\rm\ where}\ 
s_t=\nabla_{x} f_{b}(x)|_{x_{t}}=\frac{1}{b}\sum_{i\in B} \nabla f_{\xi_{i}}(x_{t}).
\]
Then, the momentum term of $g_t$ and the second-order moment term $[g_t]^2$ is accumulatively computed as $m_t$ and $v_t$, respectively.
AdaSAM then updates iterate along $-m_t$ with the adaptive learning rate $\gamma\eta_t$.

\begin{remark}
Below, we give several comments on AdaSAM: 
\begin{itemize}
    \item When $\beta_2 = 1$, the adaptive learning rate reduces to the diminishing one as SGD. Then, AdaSAM recovers the classic SAM optimizer. 
    \item If we drop out the 8-th line $\hat v_t = \max(\hat v_{t-1}, v_t),$ then our algorithm  becomes the variant of Adam. The counterexample that Adam does not converge in the \cite{j.2018on} also holds for the SAM variant, while AdaSAM can converge. 
    \end{itemize}
\end{remark}

\subsection{Convergence Analysis}
Before presenting the convergence results of the AdaSAM algorithm, we first introduce some necessary assumptions.

\begin{assumption}[\textbf{$L$-smooth}]\label{assumption-smooth}
$f_{i}$ and $f$ is differentiable with gradient Lipschitz property: 
$\norm[]{}{\nabla f_{i}(x) - \nabla f_{i}(y)} \leq L \norm[]{}{x-y}$,$\norm[]{}{\nabla f(x) - \nabla f(y)} \leq L \norm[]{}{x-y},
 \forall x,y \in \sR^{d}, i=1,2,...,N,$
which also implies the descent inequality, i.e., $f_{i}(y)\leq f_{i}(x) +\ip{\nabla f_{i}(x)}{y-x} + \frac{L}{2}\norm[2]{}{y-x}$.
\end{assumption}

\begin{assumption}[\textbf{Bounded variance}]\label{assumption-variance} 
The estimator of the gradient is unbiased and the variance of the stochastic gradient is bounded. i.e., 
$$\EE\nabla f_i(x) = \nabla f(x), \quad \EE \norm[2]{}{\nabla f_{i}(x) - \nabla f(x)} \leq \sigma^{2}.$$
When the mini-batch size $b$ is used, we have $\EE \norm[2]{}{\nabla f_{b}(x) - \nabla f(x)} \leq \frac{\sigma^{2}}{b}.$
\end{assumption}

\begin{assumption}[\textbf{Bounded stochastic gradients}]\label{assumption-gradient}
The stochastic gradient is uniformly bounded, \ie, $$\norm[]{\infty}{\nabla f_{i}(x)} \leq G, for\ any\ i=1,\ldots, N.$$
\end{assumption}
\begin{remark}
The above assumptions are commonly used in the proof of convergence for adaptive stochastic gradient methods such as \cite{cutkosky2019momentum,huang2021super,zhou2018convergence,chen2018convergence}.
\end{remark}

Below, we briefly explain the main idea of analyzing the convergence of the AdaSAM algorithm. 
First, we discuss the difficulty of applying the adaptive learning rate on SAM.
We notice that the main step which contains adaptive learning rate in convergence analysis is to estimate the expectation 
$\EE{[ x_{t+1} - x_{t}]} = -\EE \od{m_{t}}{\eta_{t}}=-\EE \od{(1-\beta_1)g_{t}}{\eta_{t}} -\EE \od{\beta_1 m_{t-1}}{\eta_{t}},$ 
which is conditioned on the filtration $\sigma(x_t)$. In this part, we consider the situation that $\beta_1 =0$ which does not include the momentum. Then, we apply delay technology to disentangle the dependence between $g_t$ and $\eta_t$, that is 
\begin{align*}
    \EE \od{g_{t}}{\eta_{t}} 
    &= \EE {[\od{g_{t}}{\eta_{t-1}}]} + \EE[]{ [ \od{g_{t}}{ (\eta_{t} - \eta_{t-1}) } ]}\\
    &= \od{\nabla f(x_t)}{\eta_{t-1}} + \EE[]{ [ \od{g_{t}}{ (\eta_{t} - \eta_{t-1}) } ]}.
\end{align*} 
The second term $\EE[]{ [ \od{g_{t}}{ (\eta_{t} - \eta_{t-1}) } ]}$ is dominated by the first term $\od{\nabla f(x_t)}{\eta_{t-1}} $. Then, it is not difficult to get the convergence result of the stochastic gradient descend with the adaptive learning rate such as AMSGrad.  However,  when we apply the same strategy to  AdaSAM, we find that $\EE \od{g_{t}}{\eta_{t-1}}$ cannot be handled similarly because $\EE g_t =\EE \nabla_x f_b\left(x+ \rho \frac{\nabla f_b(x)}{\norm[]{}{\nabla f_b(x)}}\right) \neq \nabla f(x_t)$. Inspired by \cite[Lemma 16]{andriushchenko2022towards}, our key observation is that 
\begin{align*}
    \EE  \nabla_x f_b\left(x+ \rho \frac{\nabla f_b(x)}{\norm[]{}{\nabla f_b(x)}}\right) 
    &\approx \EE \nabla_x f_b\left(x+ \rho \frac{\nabla f(x)}{\norm[]{}{\nabla f(x)}}\right) \\
    &= \nabla_x f\left(x+ \rho \frac{\nabla f(x)}{\norm[]{}{\nabla f(x)}}\right)
\end{align*} 
and we prove the other terms such as $ \EE \od{\left(\nabla_x f_b\left(x+ \rho \frac{\nabla f_b(x)}{\norm[]{}{\nabla f_b(x)}}\right) - \nabla_x f_b\left(x+ \rho \frac{\nabla f(x)}{\norm[]{}{\nabla f(x)}}\right)\right)}{\eta_{t-1}}$ have small values that do not dominate the convergence rate.

On the other hand, when we apply the momentum steps, we find that the term $\EE m_{t-1}\odot \eta_{t}$ cannot be ignored. By introducing an auxiliary sequence $z_t=x_t +\frac{\beta_1}{1-\beta_1}(x_{t}-x_{t-1})$, we have $\EE{[ z_{t+1} - z_{t}]}= -\EE{[\frac{\beta_1}{1-\beta_{1}}\gamma \od{m_{t-1}}{(\eta_{t-1}-\eta_{t})}-\gamma \od{g_t}{\eta_{t}}]}$.
The first term contains the momentum term which has a small value due to the difference of the adaptive learning rate $\eta_t$. Thus, it is diminishing without hurting the convergence rate.

\begin{theorem}\label{theorem}
Under the assumptions \ref{assumption-smooth},\ref{assumption-variance},\ref{assumption-gradient}, and $\gamma  $ is a fixed number satisfying $\gamma \leq \frac{\epsilon}{16L}$, for the sequence $\{x_t\}$ generated by Algorithm \ref{alg:SAM-adaptive-final}, we have the following convergence rate
\begin{align}\label{theorem-ineq}
\frac{1}{T}\sum_{t=0}^{T-1}\EE \norm[2]{2}{\nabla f (x_{t})} \!\leq\! \frac{2G(f(x_{0})\!-\!f^{*})}{\gamma  T} \!+\! \frac{8G \gamma L }{ \epsilon} \frac{\sigma^{2}}{b \epsilon} \!+\! \Phi
\end{align}
where 
\begin{align}
&\Phi = \frac{45G  L^2 \rho_{t}^2}{ \epsilon} + \frac{2G^{3}}{(1-\beta_1)T} 
d(\frac{1}{\epsilon}-\frac{1}{G}) + \frac{6 \gamma^2 L^2 \beta_{1}^2}{  (1-\beta_{1})^{2}}  \frac{d G^{3}}{\epsilon^{3}} \nonumber \\
&+ \frac{2(4+(\frac{\beta_1}{1-\beta_1})^2 )\gamma L G^3}{T}  d(\epsilon ^{-2}-G^{-2})+\frac{8G \gamma  L }{ \epsilon} \frac{L \rho_{t}^{2}}{\epsilon }, 
\end{align}
in which $T$ is the number of iteration, $f^{*}$ is the minimal value of the function $f$, $\gamma$ is the base learning rate, $b$ is the mini-batch size, d is the dimension of paramter $x$. $\beta_1$, $G$, $L$, $\epsilon$, $\sigma^{2}$, $d$ are fixed constants.
\end{theorem}

Theorem \ref{theorem} characterizes the convergence rate of the sequence $\{x_{t}\}$ generated by AdaSAM with respect to the stochastic gradient residual. The first two terms of the right hand side of Inequality \eqref{theorem-ineq} are the terms that dominate the convergence rate. Compared with the first two terms, $\Phi$ is a small value while we set neighborhood size $ \rho $ and learning rate $ \gamma $ as small values which are related to large iteration number $T$. Then, we obtain the following corollary directly.

\begin{corollary}[\textbf{Mini-batch linear speedup}] Under the same conditions of Theorem \ref{theorem}. Furthermore, when we choose the base learning rate $\gamma = O(\sqrt{\frac{b}{T}})$ and neighborhood size   $\rho = O(\sqrt{\frac{1}{bT}})$ , the following result holds:
\begin{align*}
\frac{1}{T}\sum_{t=0}^{T-1}\EE \norm[2]{2}{\nabla f (x_{t})} 
& =O\left(\frac{1}{\sqrt{bT}}\right) +O\left(\frac{1}{bT}\right) +O\left(\frac{1}{T}\right) \\
& +O\left(\frac{1}{b^{\frac{1}{2}} T^{\frac{3}{2}}}\right)+O\left(\frac{b^{\frac{1}{2}}}{T^{\frac{3}{2}}}\right)+O\left(\frac{b}{T}\right) .
\end{align*}
When $T$ is sufficiently large, we achieve the linear speedup convergence rate with respect to mini-batch size $b$, i.e., 
\begin{align}
\frac{1}{T}\sum_{t=0}^{T-1}\EE \norm[2]{2}{\nabla f (x_{t})}=O\left(\frac{1}{\sqrt{bT}}\right).
\end{align}
\end{corollary}

\begin{remark}
Two comments are given about the above results:
\begin{itemize}
\item 
To reach a $O(\delta)$ stationary point, when the batch size is 1, it needs $T=O(\frac{1}{\delta^{2}})$ iterations. 
When the batch size is $b$, we need to run $T=O(\frac{1}{b \delta^2})$ steps.
The method with batch size  $b$ is $b$ times faster than batch size of 1, which means that it has the mini-batch linear speedup property.  
\item According to \cite{li2014communication,li2014efficient,chen2021towards}, AdaSAM can be extended to  distributed version and achieves linear speedup property with respect to the number of works in the Parameter-Sever setting. 
 \end{itemize}
\end{remark}

 
\begin{table*}[ht]
\centering
\caption{Evaluating SGD, SAM, AMSGrad and AdaSAM on the GLUE benchmark with $\beta_1 =0.9$}
\begin{tabular}{lcccccccccc}
\toprule
\multicolumn{1}{c}{}                        & \textbf{CoLA}               & \textbf{SST-2}              & \textbf{MRPC}                    & \textbf{STS-B}                          & \textbf{RTE}                & \textbf{MNLI}               & \textbf{QNLI}                      & \textbf{QQP} &                        \\
\multicolumn{1}{c}{\multirow{-2}{*}{\textbf{Model}}} & {mcc.} & {Acc.} & {Acc./F1}  & {Pcor./Scor.} & {Acc.} & {m./mm.} & {Acc.} & F1/ Acc.           & \multirow{-2}{*}{\textbf{Avg.}} \\ 
\midrule
SGD & 9.25 & 50.92 & 68.38/ 81.22 & 3.22/ 1.9 & 55.6 & 84.94/ 84.87 & 63.61 & 85.6/ 80.14 & 55.8                   \\
SAM($\rho=$0.01) & 4.64 & 95.87 & 70.58/ 81.98 & 84.74/ 85.57 & 52.71 & 90.5/ 90.19 & 94.44 & 84.7/ 87.88 & 76.98                 \\
SAM($\rho=$0.005) & 66.76 & 95.76 & 68.38/ 81.22 & 2/ 2 & 52.71 & 90.42/ 89.74 & 94.6 & 86.72/ 89.94 & 68.35                  \\
SAM(best) & 66.76 & 95.87 & 70.58/ 81.98 & 84.74/ 85.57 & 52.71 & 90.5/ 90.19 & 94.6 & 86.72/ 89.94 & 82.51                  \\
AMSGrad                                     & 68.0                        & 96.33                        & 90.2/ 92.72                   & 91.72/ 91.48                         & 87.73                        & 90.67/ 90.41                      & 94.82                       & 88.7/ 91.41           & 89.52                  \\
\midrule
AdaSAM($\rho=$0.01)                                  & 65.29                        & 96.33                        & 91.18/ 93.64                   & 90.13/ 90.36                         & 84.84                        & 90.97/ 90.42                      & 94.65                       & 88.55/ 91.23           & 88.97                          \\
AdaSAM($\rho=$0.005)                                 & 68.74                        & 96.67                        & 90.93/ 93.36                   & 91.64/ 91.38                         & 87.73                        & 90.88/ 90.4                      & 94.56                       & 88.69/ 91.33           & 89.69                 \\
AdaSAM($\rho=$0.001)                                 & 67.3                        & 96.1                        & 90.2/ 92.96                   & 91.9/ 91.62                         & 85.92                        & 90.45/ 90.4                      & 94.56                       & 88.64/ 91.27           & 89.28                  \\
AdaSAM(best)                                  & 68.74                        & 96.67                        & 91.18/ 93.64                   & 91.9/ 91.62                         & 87.73                        & 90.97/ 90.42                      & 94.65                       & 88.69/ 91.33           & 89.8        
\\
\bottomrule
\end{tabular}
\label{tab:addlabel-1}
\end{table*}

\begin{figure*}[h]
   \centering
    \begin{subfigure}{0.23\linewidth}
    \includegraphics[width=\linewidth]{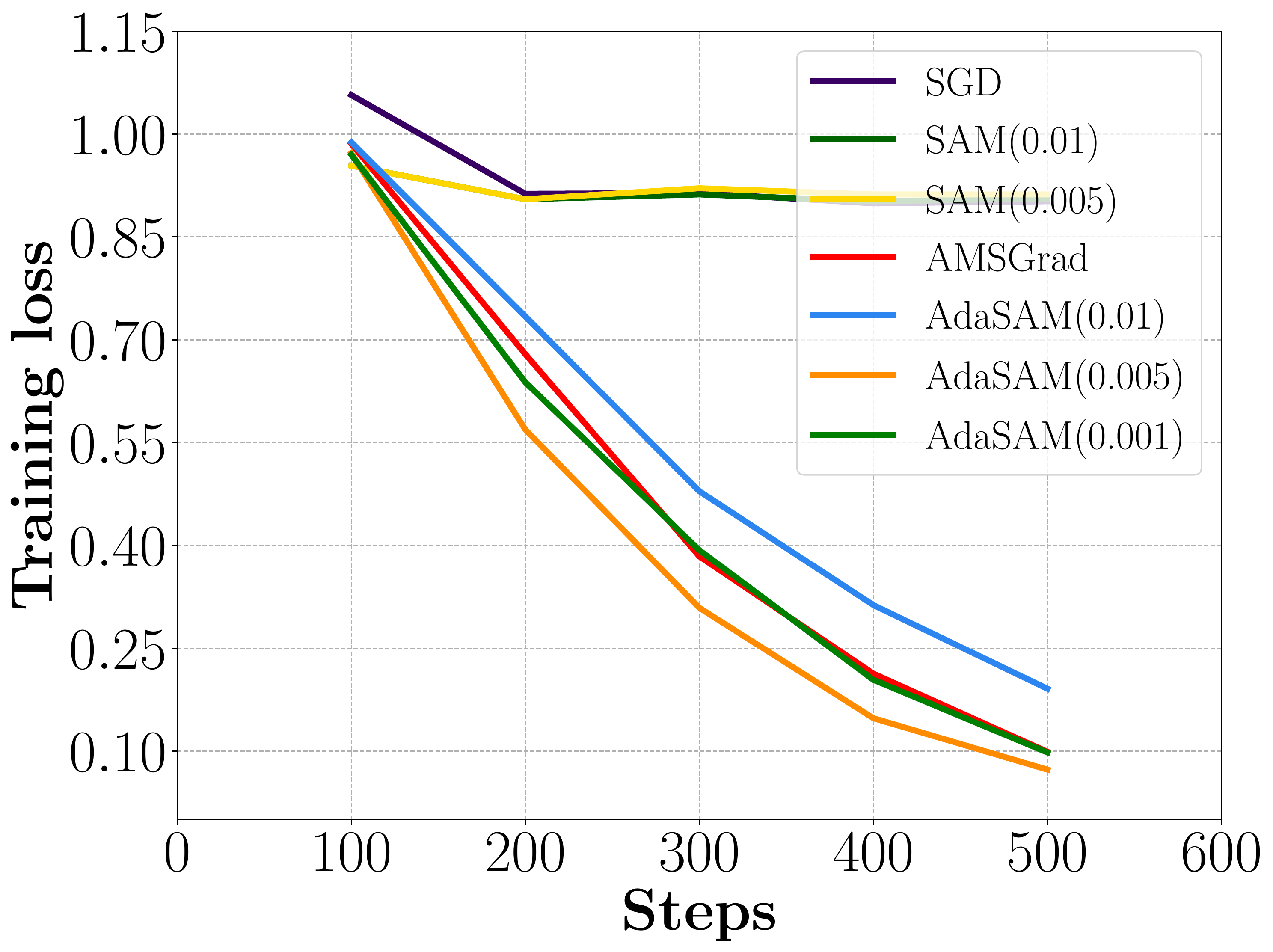}
    \caption{MRPC}
   \end{subfigure}\!\!
    \begin{subfigure}{0.23\linewidth}
    \includegraphics[width=\linewidth]{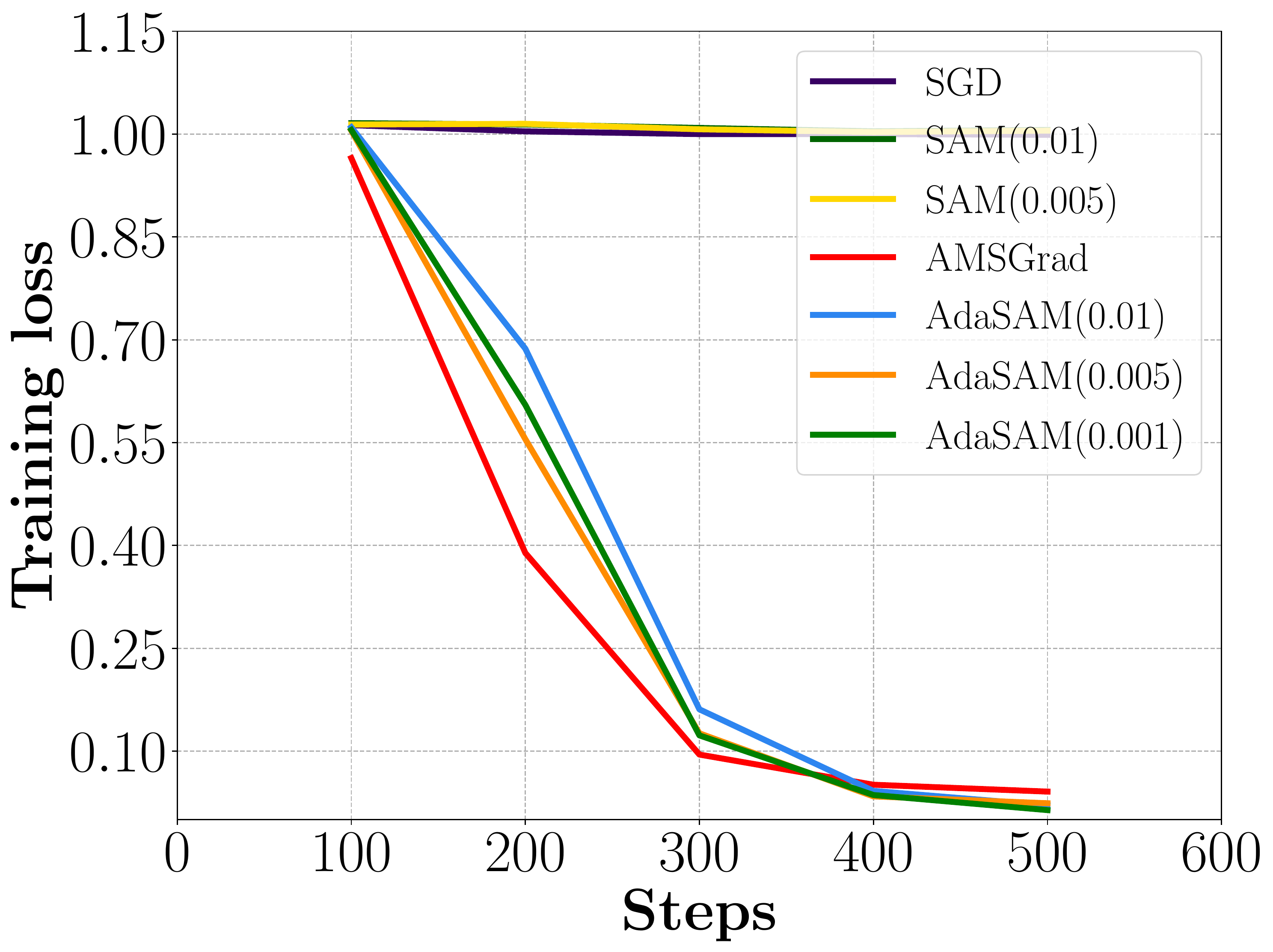}
    \caption{RTE}
   \end{subfigure}\!\!
    \begin{subfigure}{0.23\linewidth}
    \includegraphics[width=\linewidth]{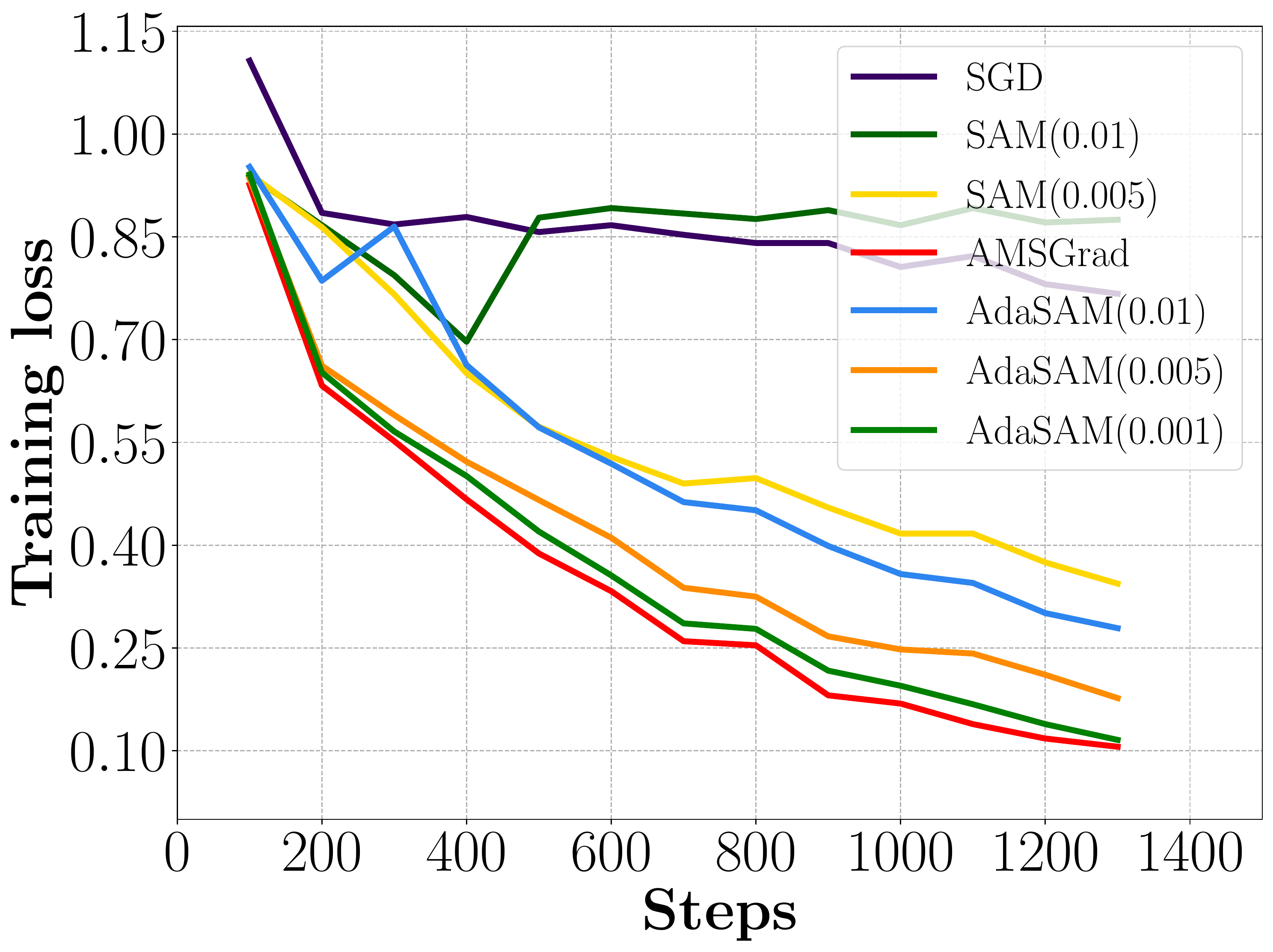}
    \caption{CoLA}
   \end{subfigure}
   \begin{subfigure}{0.23\linewidth}
    \includegraphics[width=\linewidth]{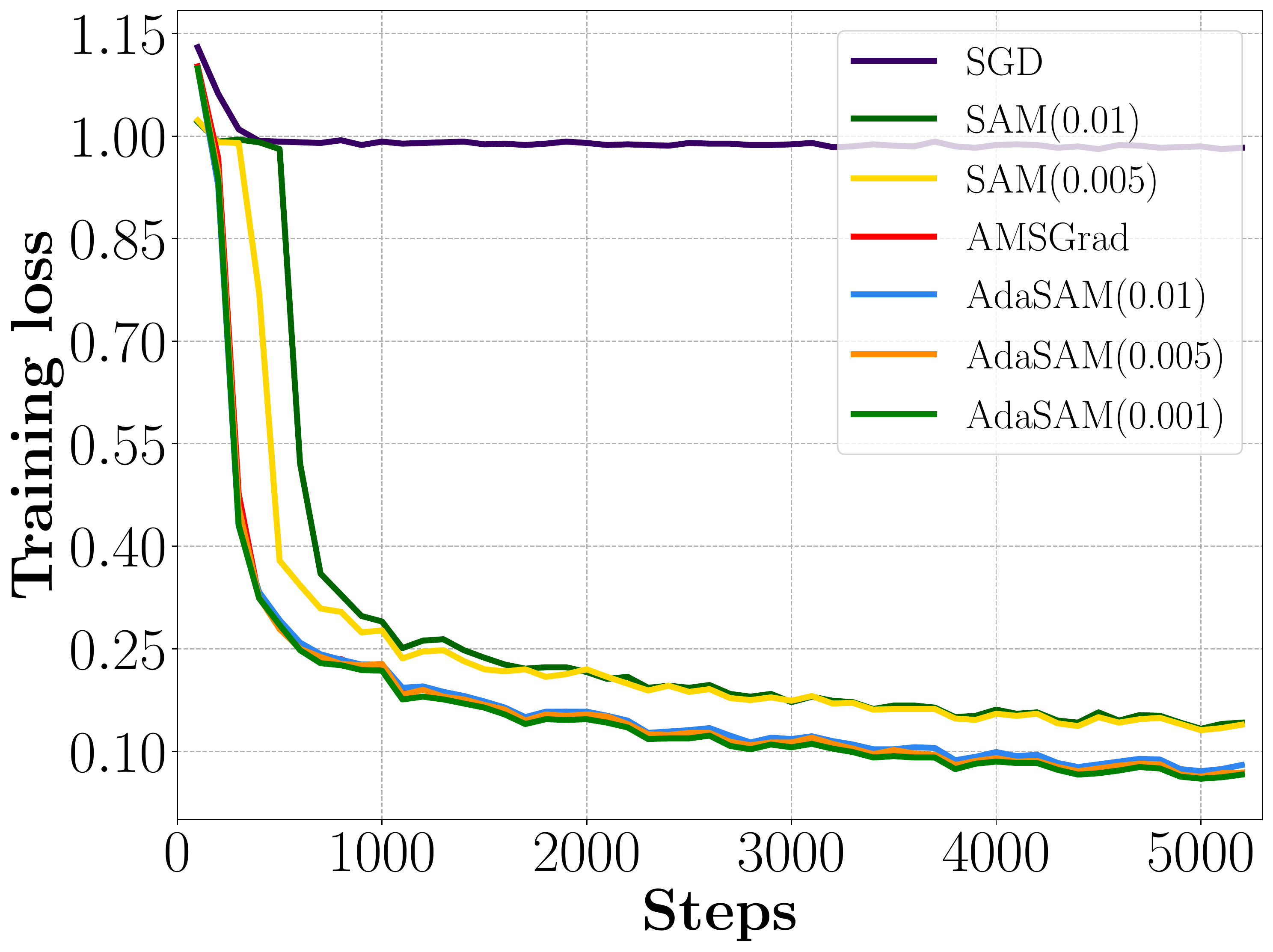}
    \caption{SST-2}
   \end{subfigure}\!\!
   
   \begin{subfigure}{0.23\linewidth}
    \includegraphics[width=\linewidth]{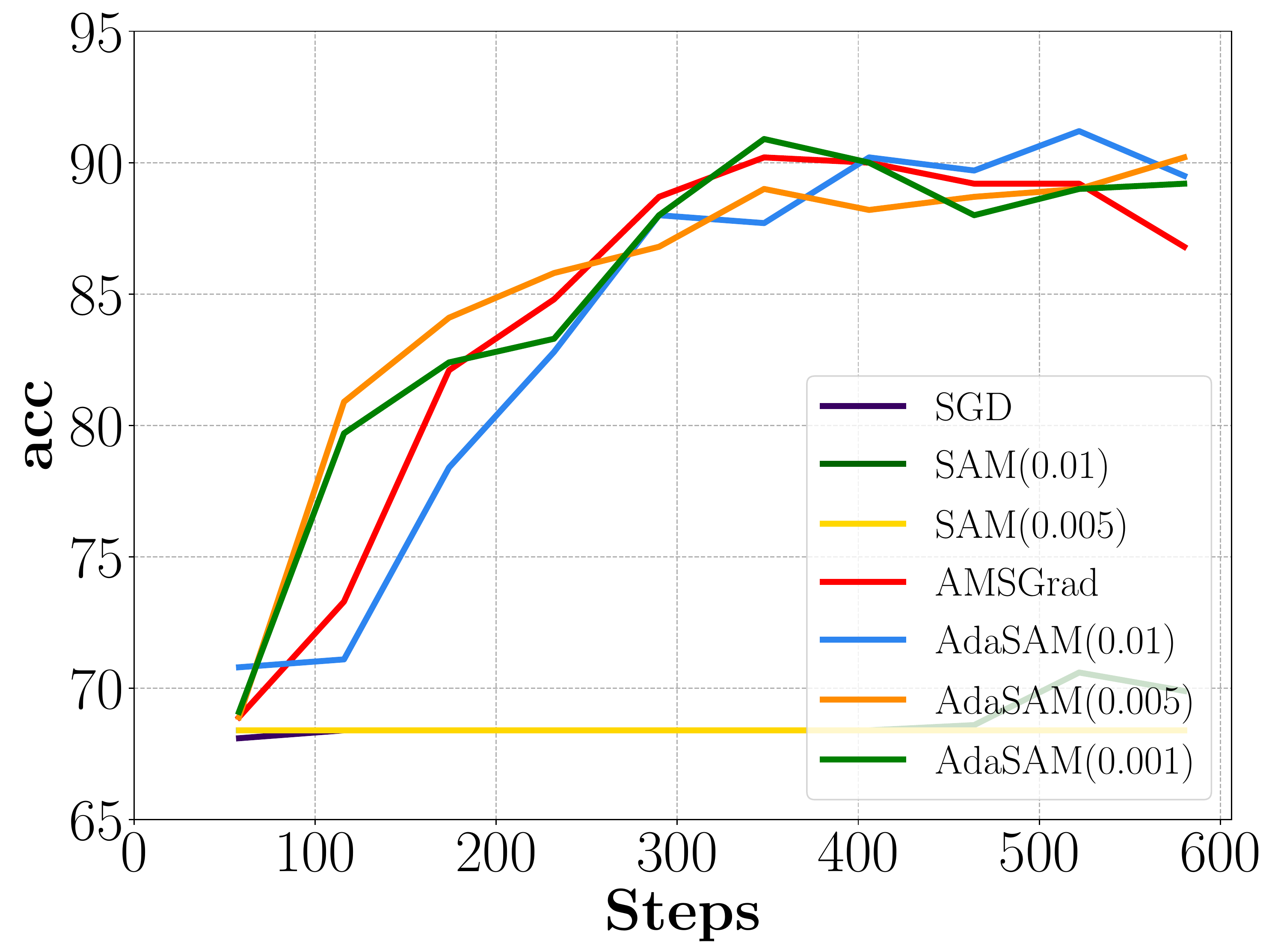}
    \caption{MRPC}
   \end{subfigure}\!\!
    \begin{subfigure}{0.23\linewidth}
    \includegraphics[width=\linewidth]{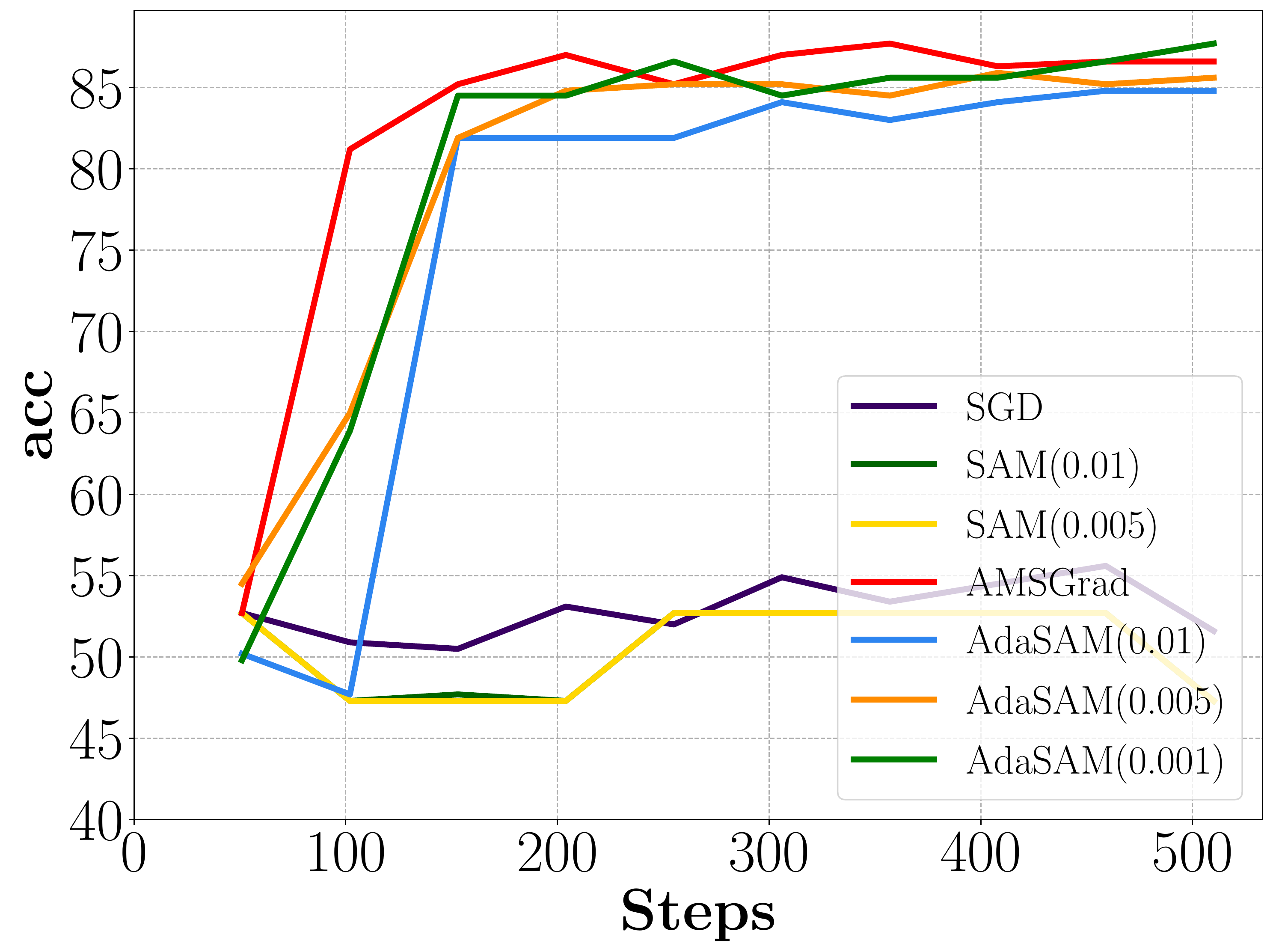}
    \caption{RTE}
   \end{subfigure}\!\!
    \begin{subfigure}{0.23\linewidth}
    \includegraphics[width=\linewidth]{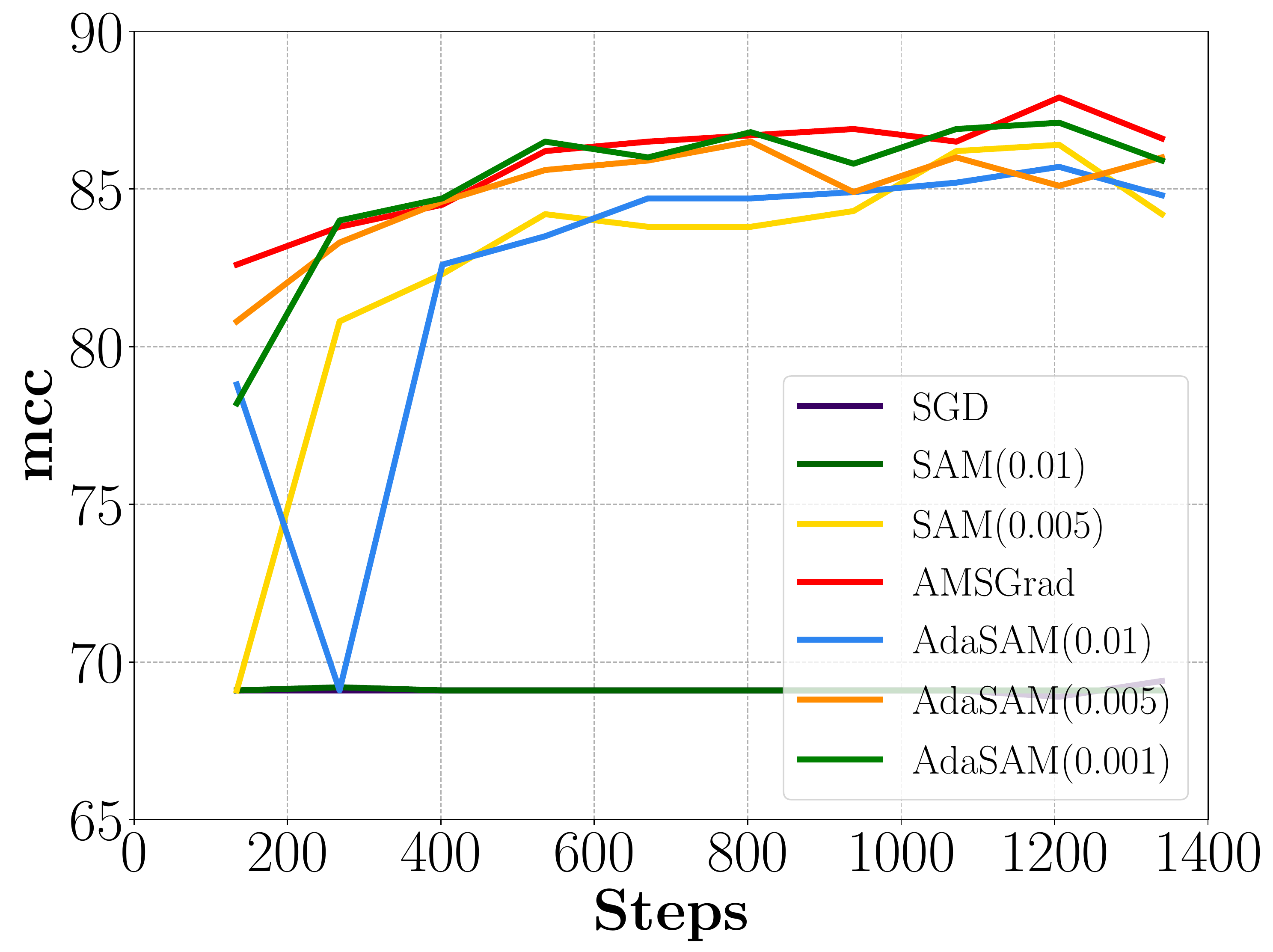}
    \caption{CoLA}
   \end{subfigure}
   \begin{subfigure}{0.23\linewidth}
    \includegraphics[width=\linewidth]{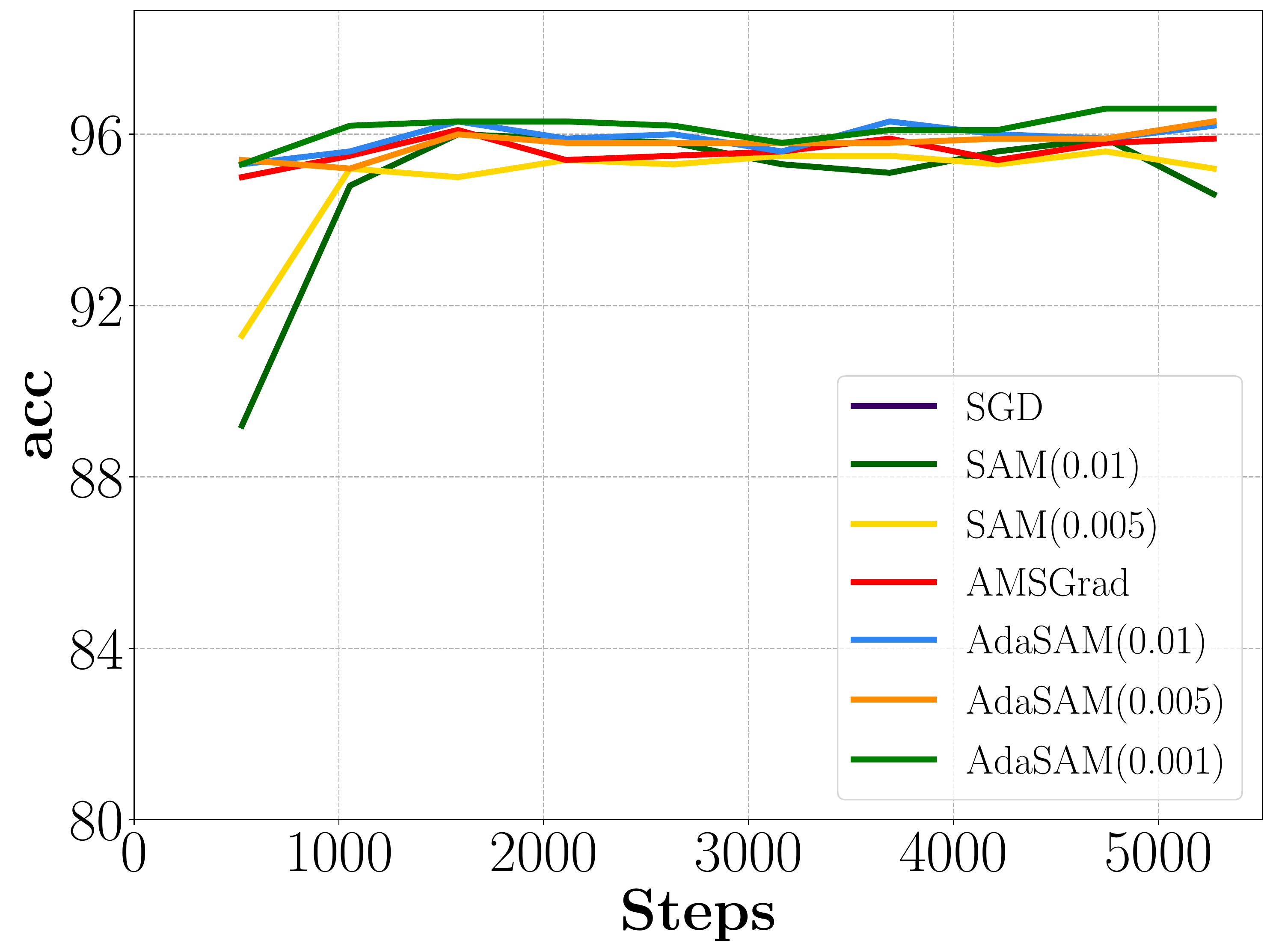}
    \caption{SST-2}
   \end{subfigure}\!\!

    \begin{subfigure}{0.23\linewidth}
    \includegraphics[width=\linewidth]{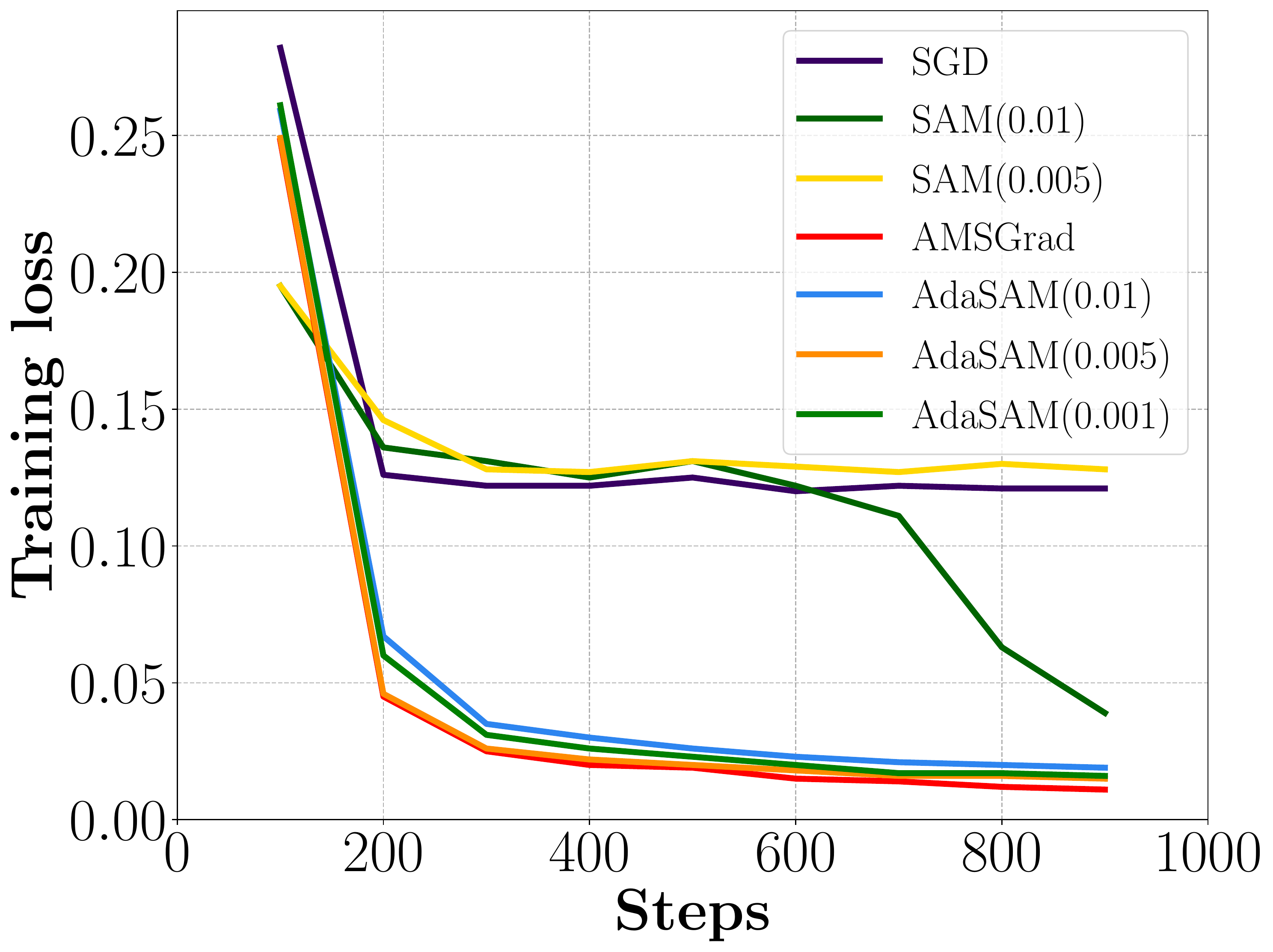}
    \caption{STS-B}
   \end{subfigure}\!\!
    \begin{subfigure}{0.23\linewidth}
    \includegraphics[width=\linewidth]{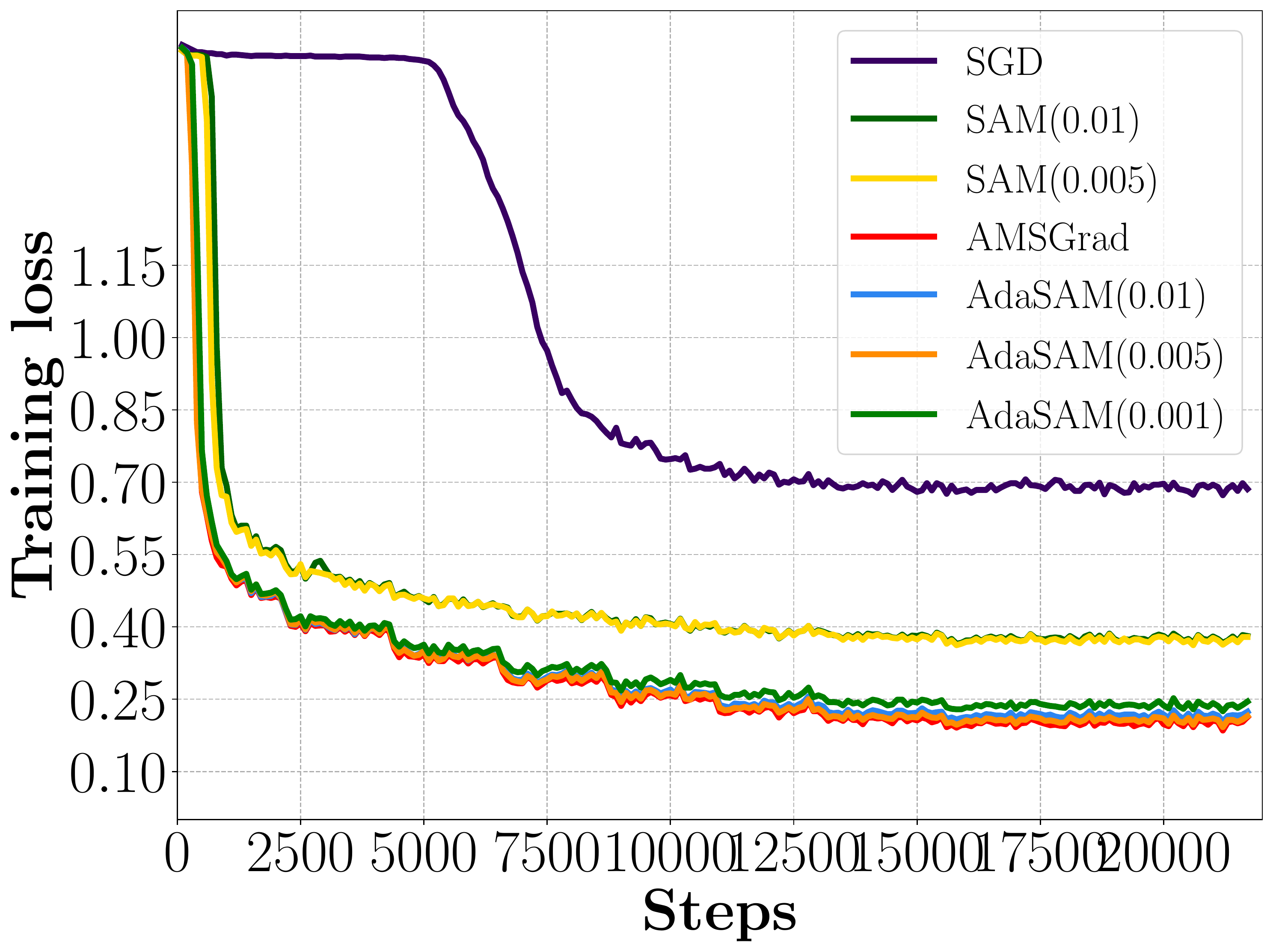}
    \caption{MNLI}
   \end{subfigure}
   \begin{subfigure}{0.23\linewidth}
    \includegraphics[width=\linewidth]{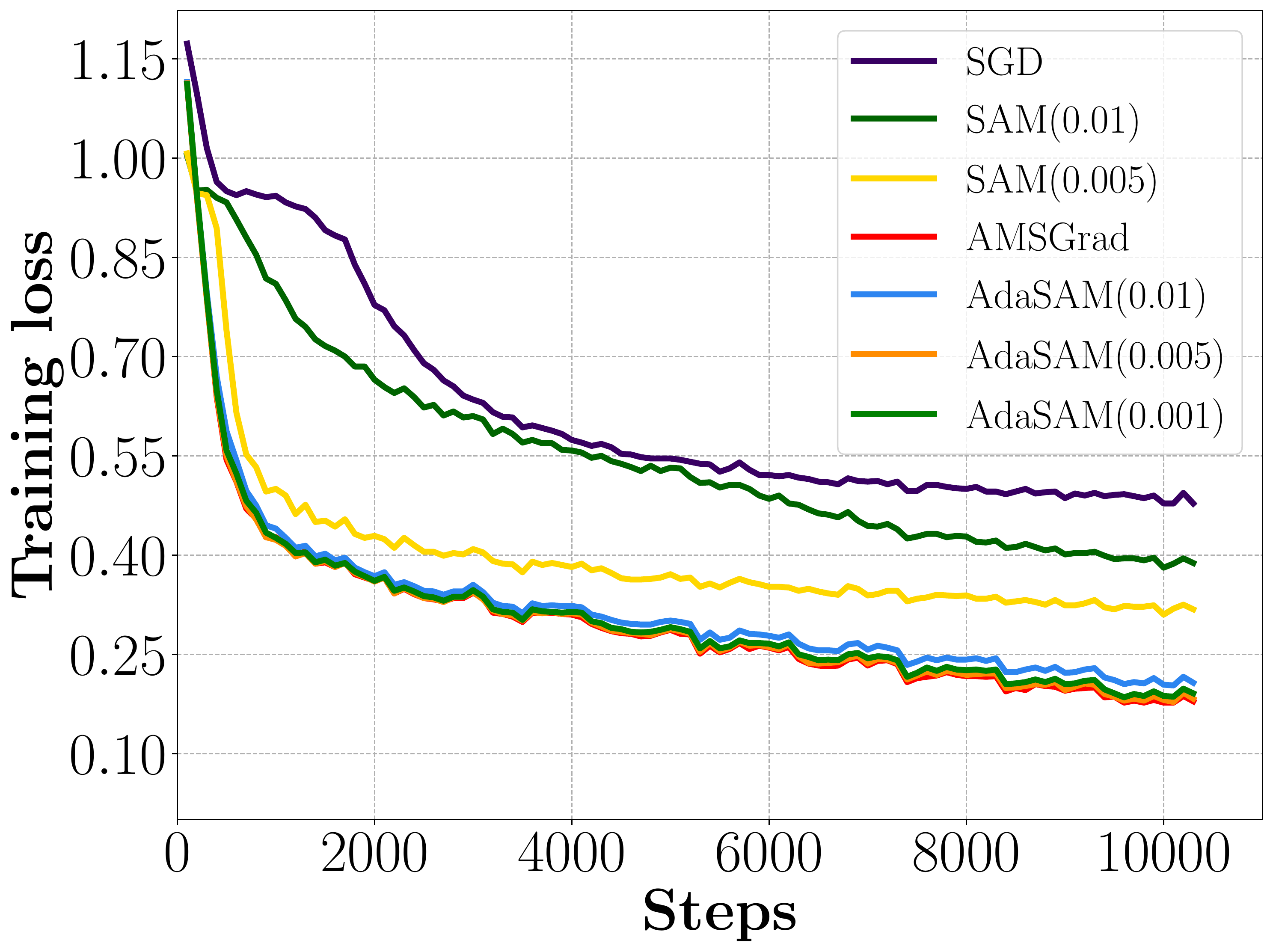}
    \caption{QQP}
  \end{subfigure}\!\!
    \begin{subfigure}{0.23\linewidth}
    \includegraphics[width=\linewidth]{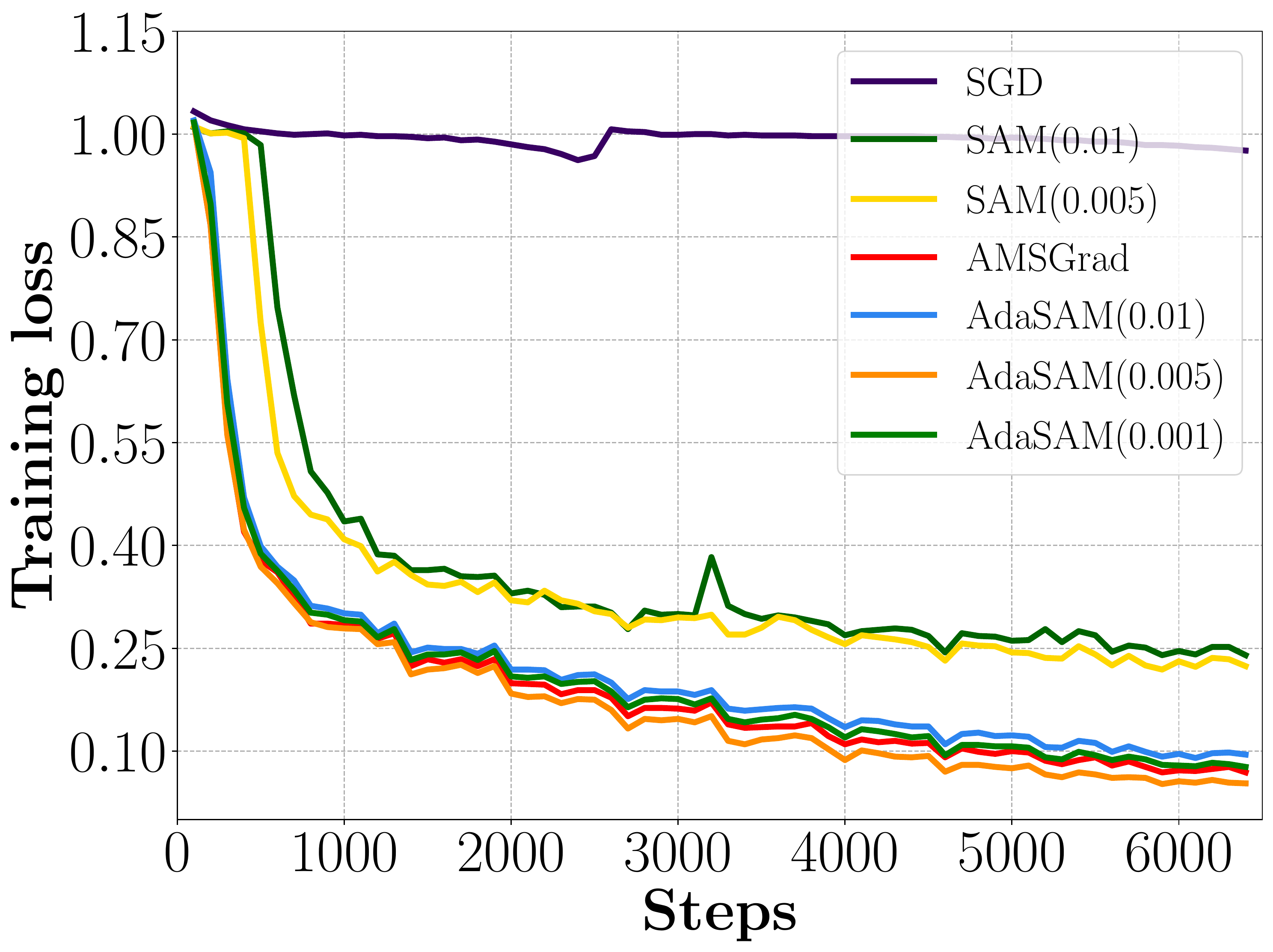}
    \caption{QNLI}
  \end{subfigure}\!\!
   
    \begin{subfigure}{0.23\linewidth}
    \includegraphics[width=\linewidth]{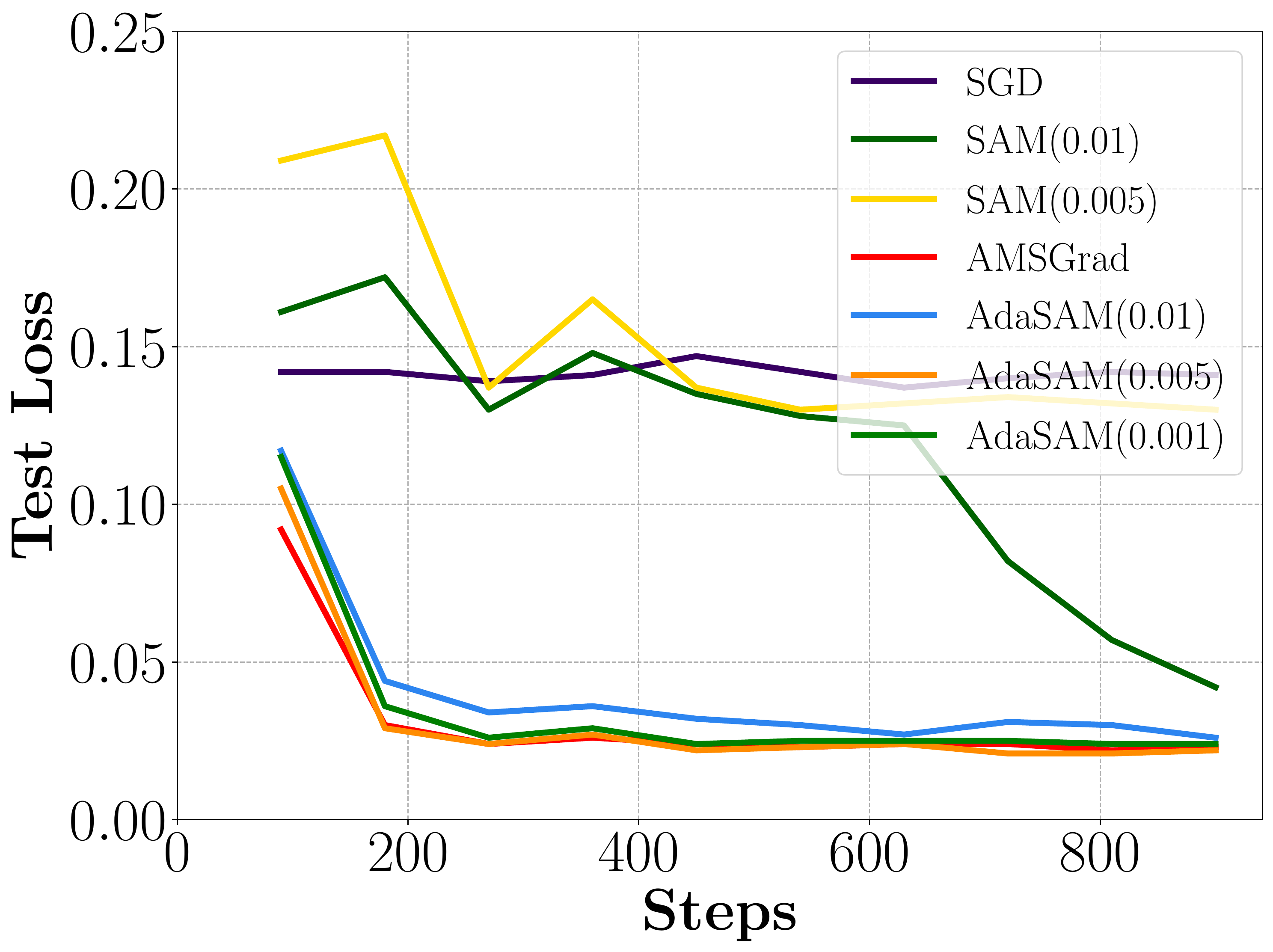}
    \caption{STS-B}
   \end{subfigure}\!\!
    \begin{subfigure}{0.23\linewidth}
    \includegraphics[width=\linewidth]{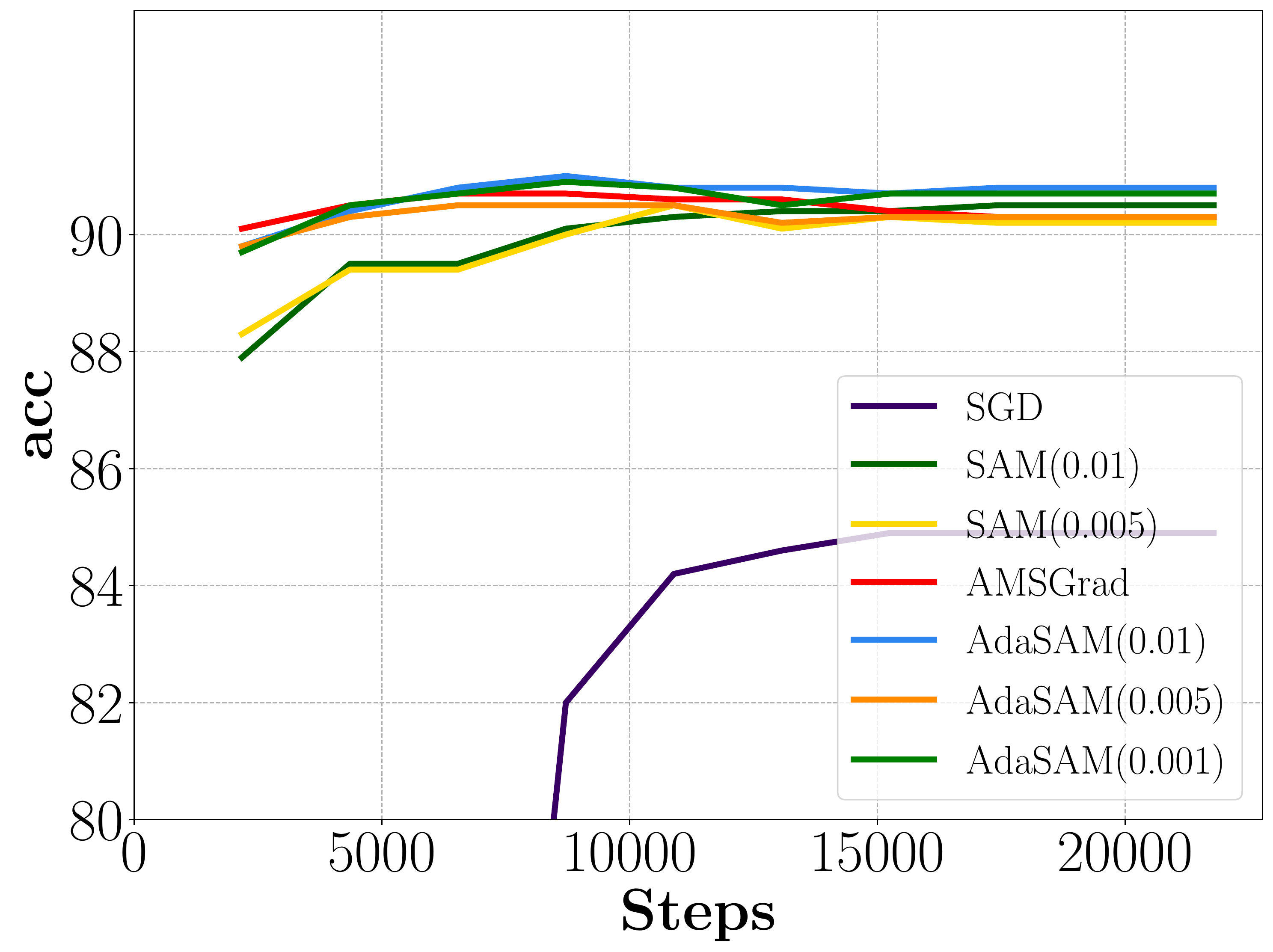}
    \caption{MNLI}
   \end{subfigure}
   \begin{subfigure}{0.23\linewidth}
    \includegraphics[width=\linewidth]{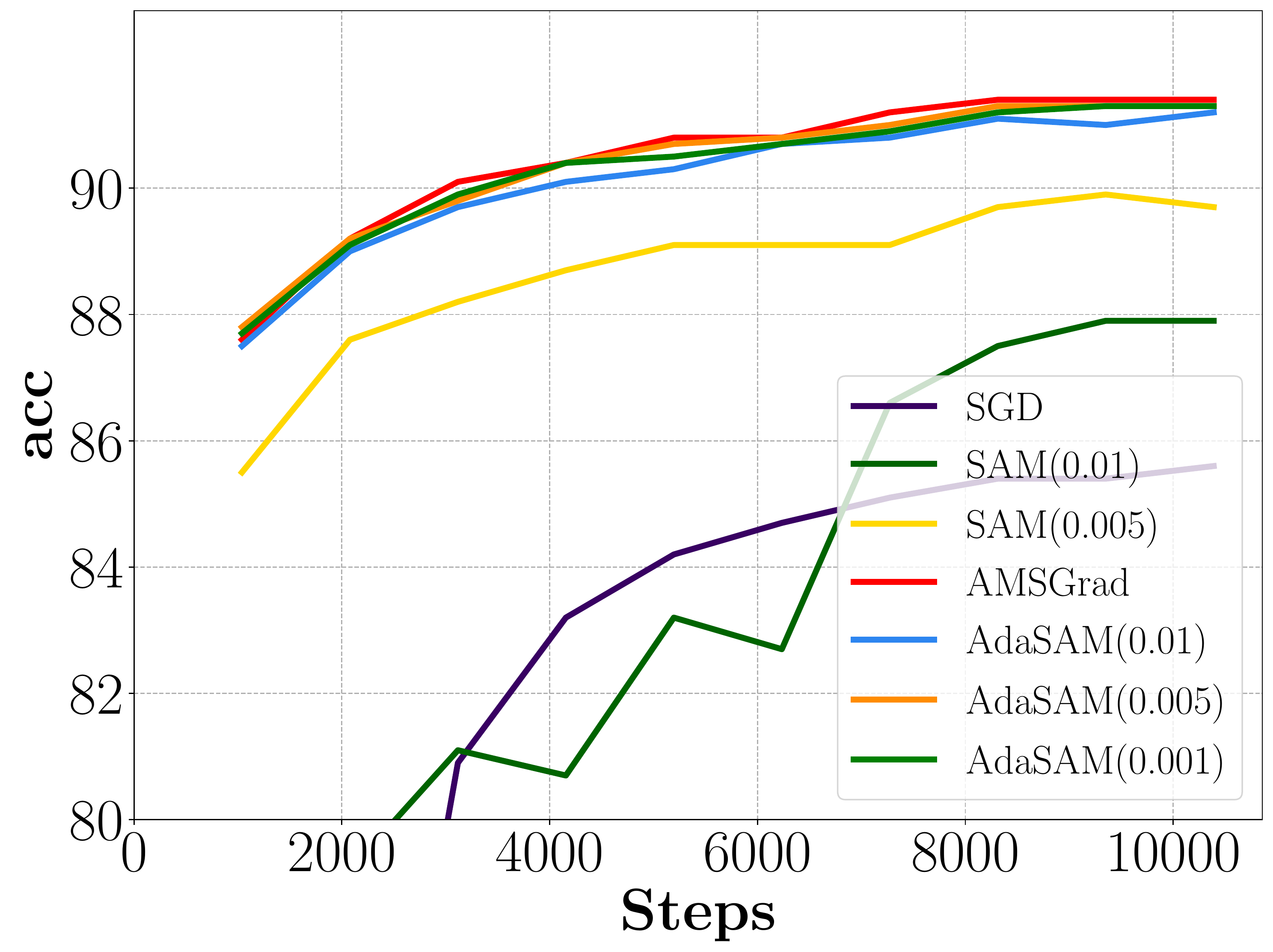}
    \caption{QQP}
  \end{subfigure}\!\!
    \begin{subfigure}{0.23\linewidth}
    \includegraphics[width=\linewidth]{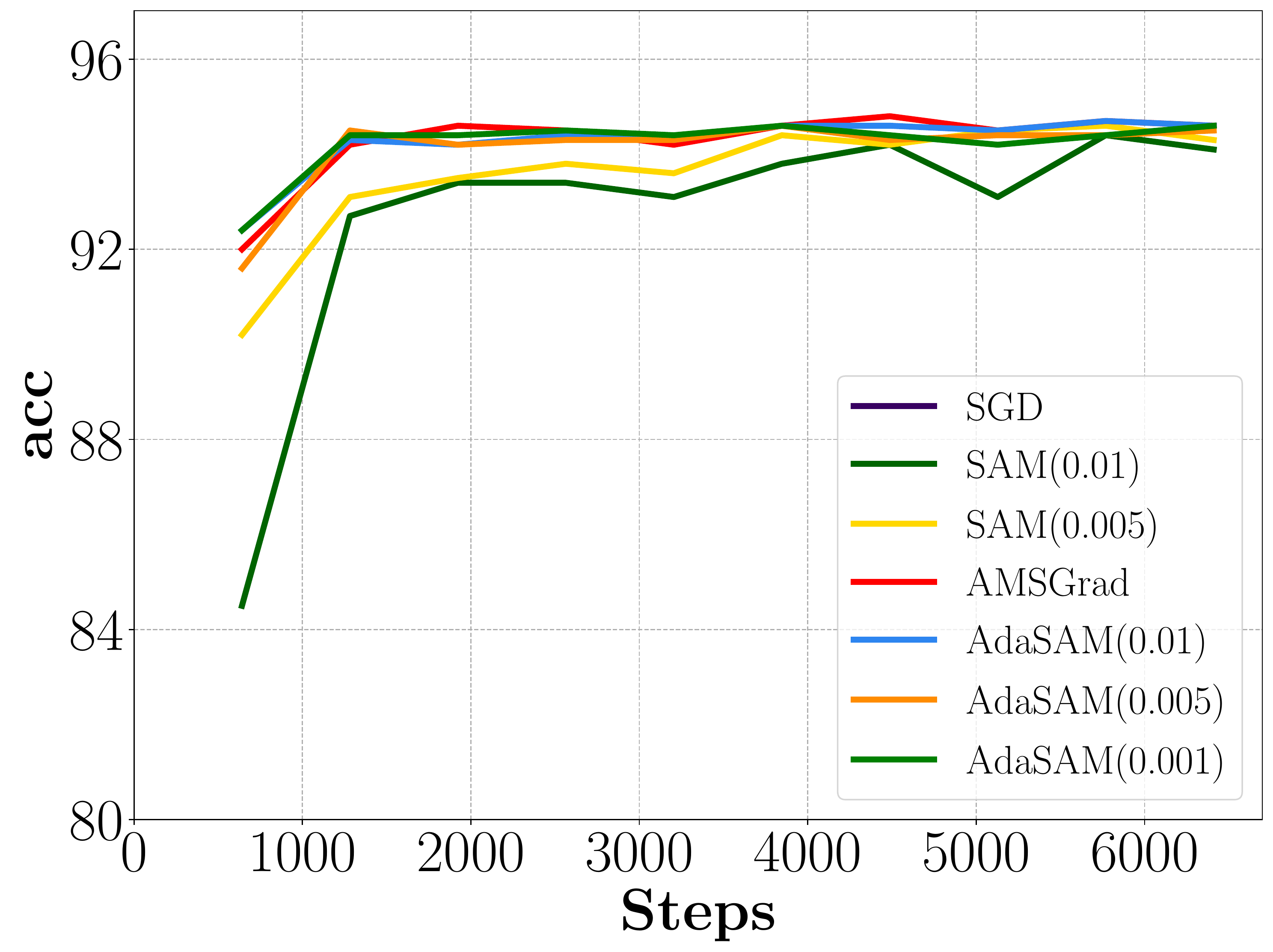}
    \caption{QNLI}
  \end{subfigure}\!\!
\caption{The loss and evaluation metric v.s. steps on MRPC, RTE, CoLA, SST-2, STS-B, MNLI, QQP, and QNLI.($\beta_1 = 0.9$)}
\label{fig:s_exp_3}
 \vspace{-0.0cm}
\end{figure*}

\subsection{Proof Sketch}
In this part, we give the proof sketch of the Theorem \ref{theorem}. For the complete proof, please see Appendix. Below, we first introduce an auxiliary sequence $z_{t}=x_{t}+\frac{\beta_1}{1-\beta_1}(x_{t}-x_{t-1})$. By applying $L$-smooth condition, we have
 \begin{align}
     f(z_{t+1}) \!\leq\! f(z_{t}) \!+\! \ip{\nabla f(z_{t})}{z_{t+1}-z_{t}} \!+\!\frac{L}{2}\norm[2]{}{z_{t+1}-z_{t}}.
 \end{align}
Applying it to the sequence $\{z_{t}\}$ and using the delay strategy yield 
 \begin{align}
& \;\;\;\;f(z_{t+1}) -f(z_{t}) \nonumber \\
&\leq \ip{\nabla f(z_{t})}{ \frac{\gamma \beta_1 }{1- \beta_1} \od{m_{t-1}}{(\eta_{t-1} - \eta_{t})} } + \frac{L}{2} \norm[2]{}{z_{t+1}- z_{t}}  \nonumber \\
& \;\;\;\;+\ip{\nabla f(z_{t})}{\frac{\gamma }{b} \sum_{i \in B } \nabla f_{i}(x_{t} + \rho_{t} \frac{s_{t} }{\norm[]{}{s_{t}} }) \odot (\eta_{t-1}-\eta_{t}) } \nonumber \\
&\;\;\;\;+\ip{ \nabla f(z_{t}) -\nabla f(x_t) }{-\frac{\gamma }{b} \sum_{i \in B } \nabla f_{i}(x_{t} + \rho_{t} \frac{s_{t} }{\norm[]{}{s_{t}} }) \odot \eta_{t-1} } \nonumber \\
&\;\;\;\; +\ip{\nabla f(x_{t})}{-\frac{\gamma }{b} \sum_{i \in B } \nabla f_{i}(x_{t} + \rho_{t} \frac{ \nabla f(x_{t}) }{\norm[]{}{ \nabla f(x_{t})} })  \odot \eta_{t-1}}\nonumber \\
&\;\;\;\;+\ip{\nabla f(x_{t})}{\frac{\gamma }{b} \sum_{i \in B } \nabla f_{i}(x_{t} + \rho_{t} \frac{ \nabla f(x_{t}) }{\norm[]{}{ \nabla f(x_{t})} }) \odot \eta_{t-1} \nonumber\\
&\;\;\;\;\;\;\;\;\;\;\;\; -\frac{\gamma }{b} \sum_{i \in B } \nabla f_{i}(x_{t} + \rho_{t} \frac{s_{t} }{\norm[]{}{s_{t}} }) \odot \eta_{t-1} } .\label{form:term1}
 \end{align}
 From the Lemma \ref{lemma:term1}, Lemma \ref{lemma:term1-1}, Lemma \ref{lemma:term2} in appendix, we can bound the above terms in \eqref{form:term1} as follows
\begin{align}
&\ip{\nabla f(z_{t})}{\frac{\gamma }{b} \sum_{i \in B } \nabla f_{i}(x_{t} + \rho_{t} \frac{s_{t} }{\norm[]{}{s_{t}} }) \odot (\eta_{t-1}-\eta_{t}) } \nonumber \\
&\leq \gamma G^{2} \norm[]{1}{\eta_{t-1} - \eta_{t}},\\
&\ip{\nabla f(z_{t})}{ \frac{\gamma \beta_1 }{1- \beta_1} \od{m_{t-1}}{(\eta_{t-1} - \eta_{t})} } \nonumber \\
&\leq \frac{\gamma \beta_1 }{1- \beta_1} G^{2} \norm[]{1}{\eta_{t-1} - \eta_{t}},\\
&\ip{\nabla f(x_{t})}{\frac{\gamma }{b} \sum_{i \in B } \nabla f_{i}(x_{t} + \rho_{t} \frac{ \nabla f(x_{t}) }{\norm[]{}{ \nabla f(x_{t})} })\odot \eta_{t-1} \nonumber \\
&\;\;\;\;-\frac{\gamma }{b} \sum_{i \in B } \nabla f_{i}(x_{t} + \rho_{t} \frac{s_{t} }{\norm[]{}{s_{t}} })\odot \eta_{t-1} } \nonumber \\
&\leq \frac{\gamma }{2 \mu^2} \norm[2]{}{\nabla f(x_{t})\odot \sqrt{\eta_{t-1}} } +\frac{2 \mu^2 \gamma  L^2 \rho_{t}^2}{\epsilon}.
\end{align}
Then we substitute them into the   \eqref{form:term1}, and take the conditional expectation to get
 \begin{align}
  &\EE f(z_{t+1}) - f(z_{t}) \nonumber \\
  &\leq \EE \ip{\nabla f(x_{t})}{-\frac{\gamma }{b} \sum_{i \in B } \nabla f_{i}(x_{t} + \rho_{t} \frac{ \nabla f(x_{t}) }{\norm[]{}{ \nabla f(x_{t})} })  \odot \eta_{t-1}}   \nonumber \\
 &+\frac{\gamma }{2 \mu^2} \norm[2]{}{\nabla f(x_{t})\odot \sqrt{\eta_{t-1}} } +  \frac{\gamma }{1- \beta_1 }G^{2} \norm[]{1}{\eta_{t-1} - \eta_{t}} \nonumber \\
 &+\EE \ip{ \nabla f(z_{t}) -\nabla f(x_t) }{-\frac{\gamma }{b} \sum_{i \in B } \nabla f_{i}(x_{t} + \rho_{t} \frac{s_{t} }{\norm[]{}{s_{t}} }) \odot \eta_{t-1} }\nonumber \\
 & +\frac{2 \mu^2 \gamma  L^2 \rho_{t}^2}{\epsilon} + \frac{L}{2} \EE \norm[2]{}{z_{t+1}- z_{t}}, \label{form:term2}
 \end{align}
 where $\mu>0$ is a  constant to be determined. Next, from the Lemma \ref{lemma:term3}, Lemma \ref{lemma:term4-0} and Lemma \ref{lemma:term-m} in Appendix, we have
 \begin{align}
&\EE \ip{\nabla f(x_{t})}{-\frac{\gamma }{b} \sum_{i \in B } \nabla f_{i}(x_{t} + \rho_{t} \frac{ \nabla f(x_{t}) }{\norm[]{}{ \nabla f(x_{t})} })  \odot \eta_{t-1}} \nonumber \\
&\leq - \gamma \norm[2]{}{\nabla f(x_{t}) \odot \sqrt{\eta_{t-1}}} +\EE \frac{\gamma }{2\alpha^2} \norm[2]{}{\nabla f(x_{t}) \odot \sqrt{\eta_{t-1}}} \nonumber \\
&\;\;\;\;+\frac{\gamma  \alpha^2 L^2 \rho_{t}^2 }{2  \epsilon},\\
&\frac{L}{2} \EE \norm[2]{}{z_{t+1}- z_{t}} \leq \frac{L G^2 \gamma^{2} \beta_{1}^{2}}{(1-\beta_1)^{2}}\EE \norm[2]{}{\eta_{t} - \eta_{t-1}}\nonumber \\
&\;\;\;\;+\gamma^{2}L (3\frac{1 + \beta}{\beta \epsilon} (\frac{L \rho_{t}^{2}}{\epsilon } + \frac{\sigma^{2}}{b \epsilon} + \EE \norm[2]{}{ \nabla f(x_t ) \odot  \sqrt{\eta_{t-1}}})\nonumber \\
& \;\;\;\;+ (1+ \beta)G^2 \EE \norm[2]{}{\eta_{t} -\eta_{t-1} }),\\
 &\EE \ip{ \nabla f(z_{t}) -\nabla f(x_t) }{-\frac{\gamma }{b} \sum_{i \in B } \nabla f_{i}(x_{t} + \rho_{t} \frac{s_{t} }{\norm[]{}{s_{t}} }) \odot \eta_{t-1} } \nonumber \\
&\leq \frac{\gamma^3 L^2 \beta_{1}^2}{2 \epsilon (1-\beta_{1})^{2}}(\frac{1}{\lambda_{1}^{2}}+\frac{1}{\lambda_{2}^{2}} +\frac{1}{\lambda_{3}^{2}})  \frac{d G_{\infty}^{2}}{\epsilon^{2}} + \frac{\gamma  L^2 \rho_{t}^2 }{2\epsilon}(\lambda_{2}^{2} +4  \lambda_{3}^2) \nonumber \\
&\;\;\;\;+\frac{\gamma \lambda_{1}^{2}}{2} \norm[2]{}{\nabla f(x_{t}) \odot \sqrt{\eta_{t-1} }} .
 \end{align}
Next, we substitute it into the \eqref{form:term2}. Taking the expectation over all history information yields
 \begin{align}
& \EE f(x_{t+1}) - \EE f(x_{t}) \nonumber \\
&\leq \!-\gamma (1\!-\! \frac{1}{2 \mu^2} \!-\!\frac{1}{2 \alpha^2}\!-\!\frac{3\gamma L(1\!+\!\beta)}{\beta \epsilon } \!-\! \frac{\lambda_{1}^{2}}{2})\EE \norm[2]{}{\nabla f(x_t) \odot  \sqrt{\eta_{t-1}}}\nonumber \\
&+\frac{2 \mu^2 \gamma  L^2 \rho_{t}^2}{\epsilon}  +  \frac{\gamma}{1-\beta_{1}} G^{2} \EE \norm[]{1}{\eta_{t-1} - \eta_{t}}   +\frac{\gamma  \alpha^2 L^2 \rho^2 }{2  \epsilon}\nonumber \\
&+\frac{\gamma^3 L^2 \beta_{1}^2}{2 \epsilon (1-\beta_{1})^{2}}(\frac{1}{\lambda_{1}^{2}}+\frac{1}{\lambda_{2}^{2}} +\frac{1}{\lambda_{3}^{2}})  \frac{d G_{\infty}^{2}}{\epsilon^{2}}  +\frac{\gamma  L^2 \rho_{t}^2 }{2\epsilon}(\lambda_{2}^{2} +4  \lambda_{3}^2) \nonumber \\
&   +\gamma^{2}L G^2 ((\frac{\beta_1}{1-\beta_1})^2 +1 +\beta) \EE \norm[2]{}{\eta_{t} -\eta_{t-1}}\nonumber \\
&+ \frac{3 \gamma^2 L (1+ \beta)}{ \beta \epsilon}(\frac{L \rho_{t}^{2}}{\epsilon }+ \frac{\sigma^{2}}{b \epsilon}).
 \end{align}
 We set $\mu^2 = \alpha^2 = 8$, $\beta = 3$, $\lambda_{1}^2 =\frac{1}{4} $, $\lambda_{2}^2 =\lambda_{3}^2 =1 $ and we choose $\frac{2 \gamma  L}{\epsilon} \leq \frac{1}{8}$.
Note that $ \eta_{t} $ is bounded. We have
 \begin{align}
&\frac{\gamma }{2 G} \EE \norm[2]{}{ \nabla f(x_t )} \leq \frac{\gamma }{2} \EE \norm[2]{}{ \nabla f(x_t ) \odot  \sqrt{\eta_{t-1}}}\\
&\leq - \EE f(x_{t+1}) + \EE f(x_{t}) + \frac{45 \gamma  L^2 \rho_{t}^2}{2 \epsilon} + \frac{4 \gamma^2 L }{ \epsilon}(\frac{L \rho_{t}^{2}}{\epsilon } + \frac{\sigma^{2}}{b \epsilon})  \nonumber \\
& +  \frac{\gamma}{1-\beta_1}G^{2} \EE \norm[]{1}{\eta_{t-1} - \eta_{t}}  +\frac{3 \gamma^3 L^2 \beta_{1}^2}{  (1-\beta_{1})^{2}}  \frac{d G_{\infty}^{2}}{\epsilon^{3}} \nonumber \\
&+(4+(\frac{\beta_1}{1-\beta_1})^2 )\gamma^{2}L G^2  \EE \norm[2]{}{\eta_{t} -\eta_{t-1}} .
 \end{align}
 Then,  telescoping it from $t=0$ to $t=T-1$, and   assuming $\gamma $ is a constant, it follows that
 \begin{align}
&\frac{1}{T}\sum_{t=0}^{T-1} \EE \norm[2]{}{ \nabla f(x_t )} \leq  \frac{2G(f(x_{0})-f^{*})}{\gamma  T}+ \frac{8G \gamma  L }{ \epsilon} \frac{\sigma^{2}}{b \epsilon}  \nonumber \\
&+ \frac{45G  L^2 \rho_{t}^2}{ \epsilon} + \frac{2G^{3}}{(1-\beta_1)T} 
d(\frac{1}{\epsilon}-\frac{1}{G}) + \frac{6 \gamma^2 L^2 \beta_{1}^2}{  (1-\beta_{1})^{2}}  \frac{d G^{3}}{\epsilon^{3}}\nonumber \\
&+\frac{8G \gamma  L }{ \epsilon} \frac{L \rho_{t}^{2}}{\epsilon }    + \frac{2(4+(\frac{\beta_1}{1-\beta_1})^2 )\gamma L G^3}{T}  d(\epsilon ^{-2}-G^{-2}),
 \end{align}
which completes the proof.

\begin{figure*}[htbp]
   \centering
    \begin{subfigure}{0.33\linewidth}
    \includegraphics[width=\linewidth]{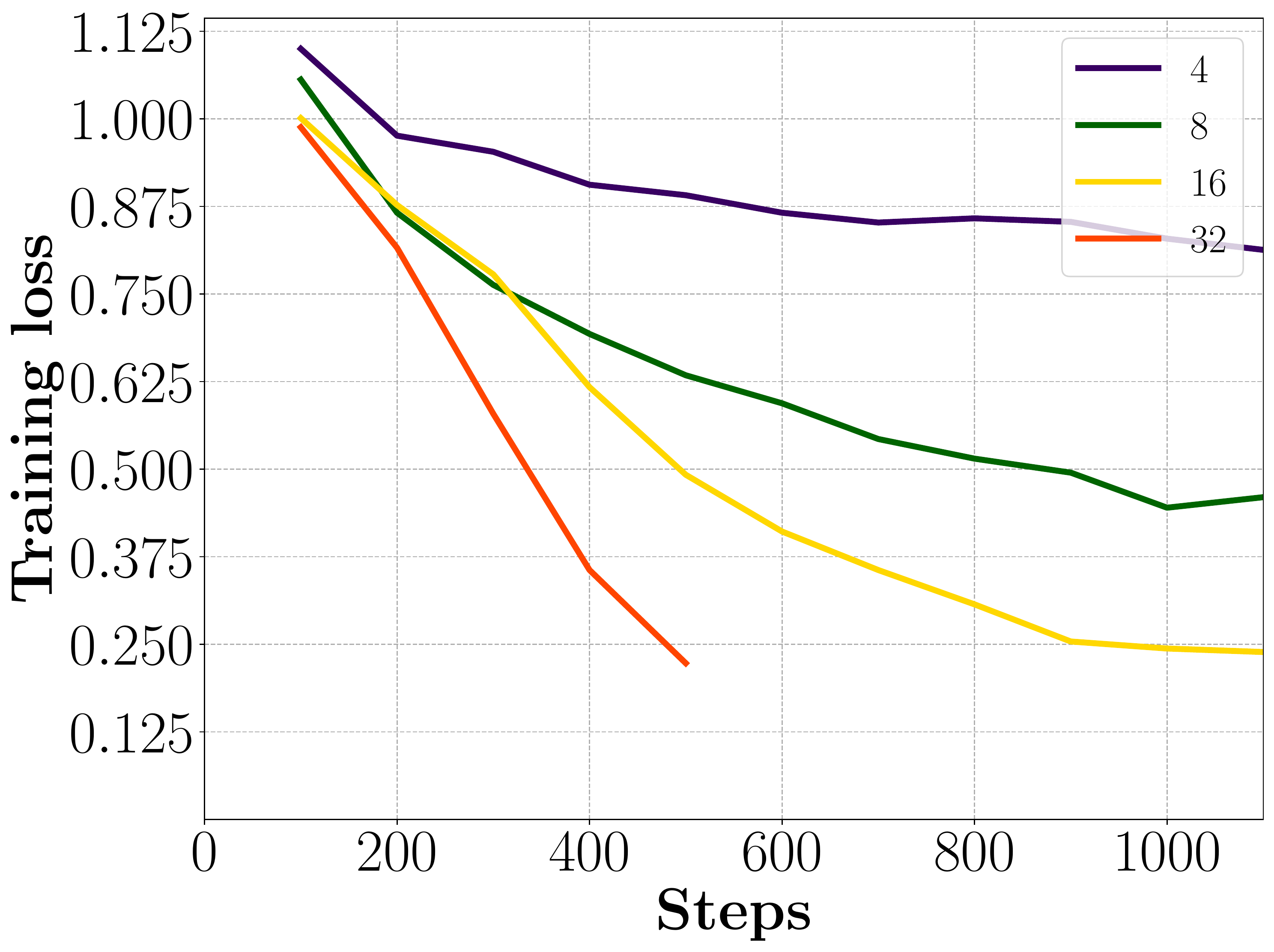}
    \caption{MRPC}
   \end{subfigure}\!\!
    \begin{subfigure}{0.33\linewidth}
    \includegraphics[width=\linewidth]{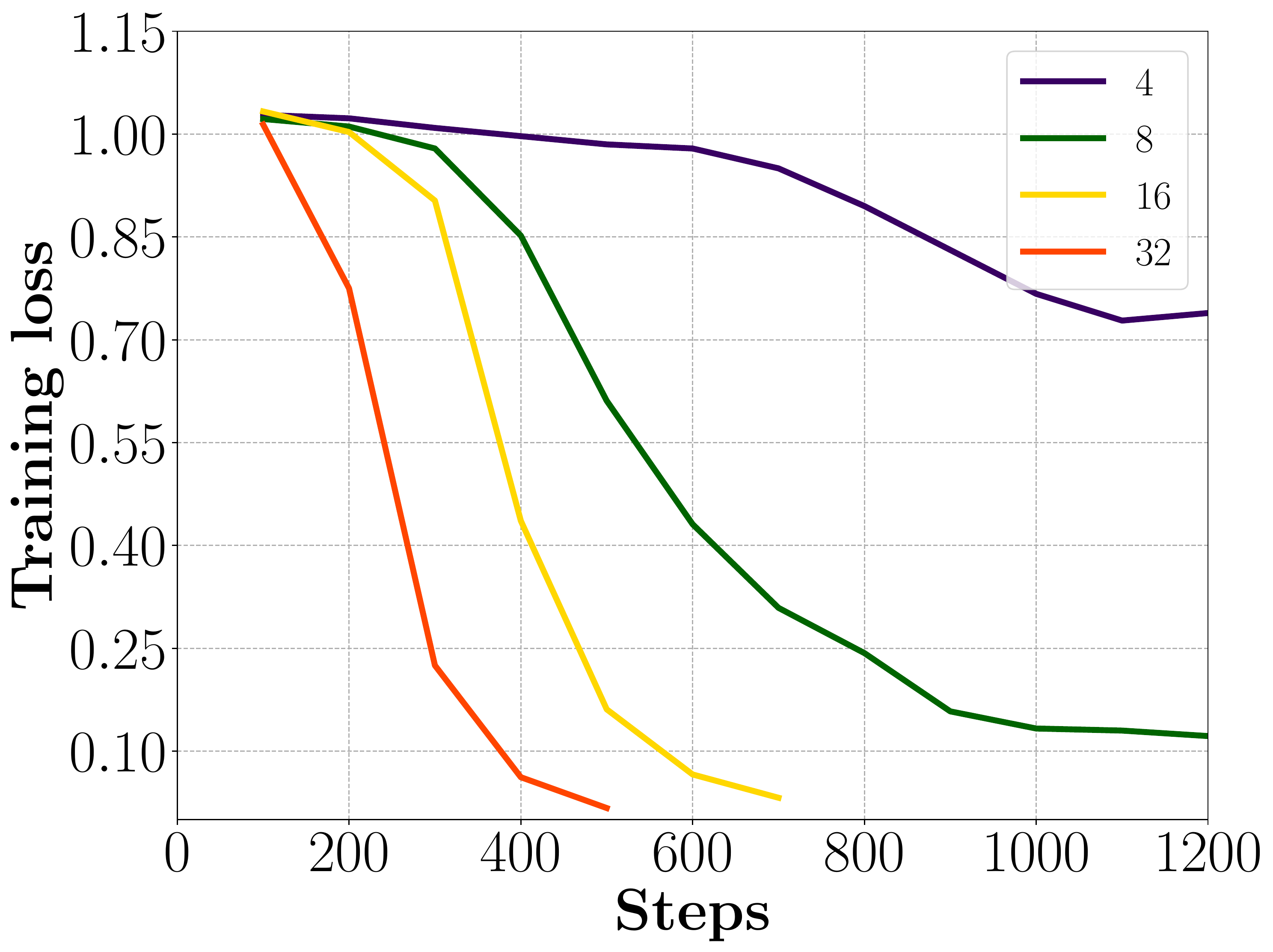}
    \caption{RTE}
   \end{subfigure}\!\!
    \begin{subfigure}{0.33\linewidth}
    \includegraphics[width=\linewidth]{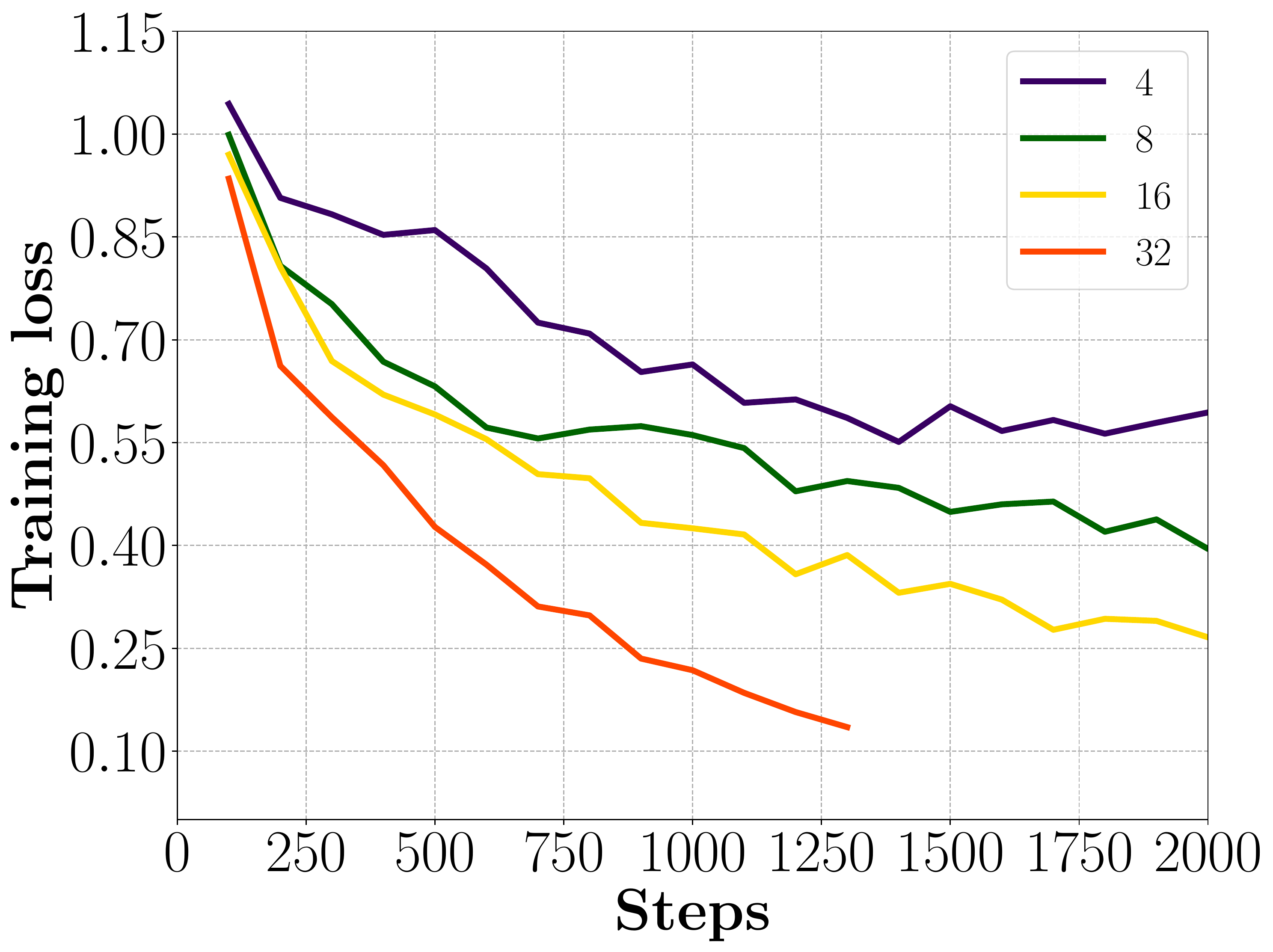}
    \caption{CoLA}
   \end{subfigure}
\caption{The linear speedup verification of AdaSAM with the number of batch size of 4, 8, 16, 32.}
\label{fig:linear}
\end{figure*}

\begin{table*}[htbp]
\centering
\caption{Results of SGD, SAM, AMSGrad and AdaSAM on the GLUE benchmark without momentum, i.e., $\beta_1 =0$}
\begin{tabular}{lcccccccccc}
\toprule
\multicolumn{1}{c}{}                        & \textbf{CoLA}               & \textbf{SST-2}              & \textbf{MRPC}                    & \textbf{STS-B}                          & \textbf{RTE}                & \textbf{MNLI}               & \textbf{QNLI}                      & \textbf{QQP} &                        \\
\multicolumn{1}{c}{\multirow{-2}{*}{\textbf{Model}}} & {mcc.} & {Acc.} & {Acc./F1}  & {Pcor./Scor.} & {Acc.} & {m./mm.} & {Acc.} & F1/ Acc.           & \multirow{-2}{*}{\textbf{Avg.}} \\ 
\midrule
SGD & 0 & 51.722 & 68.38/ 81.22 & 5.55/ 7.2 & 51.27 & 32.51/ 32.42 & 53.32 & 0/ 63.18 & 37.23                   \\
SAM($\rho=$0.01) & 41.91 & 95.3 & 68.38/ 81.22 & 9.21/ 10.38 & 53.07 & 87.99/ 87.8 & 51.24 & 83.44/ 87.27 & 63.1                 \\
SAM($\rho=$0.005) & 58.79 & 81.54 & 68.38/ 81.22 & 13.52/ 16.6 & 53.79 & 88.42/ 88.15 & 92.95 & 83.84/ 87.7 & 67.91                  \\
SAM(best) & 58.79 & 95.3 & 68.38/ 81.22 & 13.52/ 16.6 & 53.79 & 88.42/ 88.15 & 92.95 & 83.84/ 87.7 & 69.06                  \\
AMSGrad                                     & 63.78                        & 96.44                        & 89.71/ 92.44                   & 89.98/ 90.35                         & 87.36                        & 90.65/ 90.35                      & 94.53                       & 88.59/ 91.27           & 88.79                  \\
\midrule
AdaSAM($\rho=$0.01)                                 & 69.23                        & 96.22                        & 89.96/ 92.84                   & 88.83/ 89.07                         & 87                        & 90.83/ 90.41                      & 94.8                       & 88.67/ 91.38           & 89.1                 \\
AdaSAM($\rho=$0.005)                                 & 68.47                        & 96.22                        & 89.96/ 92.82                   & 91.59/ 91.22                         & 73.65                        & 90.75/ 90.42                      & 94.73                       & 88.72/ 91.46           & 88.33                  \\
AdaSAM(best)                                 & 69.23                        & 96.22                        & 89.96/ 92.84                   & 91.59/ 91.22                         & 87                        & 90.83/ 90.42                      & 94.8                       & 88.72/ 91.46           & 89.52       
\\
\bottomrule
\end{tabular}
\label{tab:addlabel-3}
\end{table*}

\section{Experiments}
In this section, we apply AdaSAM to train language models and compare it with SGD, AMSGrad, and SAM to show its effectiveness. Due to space limitations, more experiments, including visualization, task description, implementation details and results description, are placed in the Appendix.

\subsection{Experimental Setup}
\textbf{Tasks and Datasets.}\ 
We evaluate AdaSAM on a popular benchmark, \textit{i.e.} General Language Understanding Evaluation (GLUE) \cite{wang2018glue}, which consists of several language understanding tasks including sentiment analysis, question answering and textual entailment.
For a fair comparison, we report the results based on single-task, without multi-task or ensemble training. We evaluate the performance with Accuracy (``\textit{Acc}'') metric for most tasks, except the F1 scores for QQP and MRPC, the Pearson-Spearman correlations (``\textit{Pcor/Scor}'') for STS-B and the Matthew correlations (``\textit{Mcc}'') for CoLA.
The performance is better as the metric is higher.

\textbf{Implementations.}\  
We conduct our experiments using a widely-used pre-trained language model, RoBERTa-large\footnote{\url{https://dl.fbaipublicfiles.com/fairseq/models/roberta.large.tar.gz}} in the open-source toolkit fairseq\footnote{\url{https://github.com/facebookresearch/fairseq}}, with 24 transformer layers, a hidden size of 1024. For fine-tuning on each task, we use different combinations of hyper-parameters, including the learning rate, the number of epochs, the batch size, \textit{etc}~\footnote{Due to the space limitation, we show the details of the dataset and training setting in Appendix~\ref{expsetting}.}.
In particular, for RTE, STS-B and MRPC of GLUE benchmark, we first fine-tune the pre-trained RoBERTa-large model on the MNLI dataset and continue fine-tuning the RoBERTa-large-MNLI model on the corresponding single-task corpus for better performance, as many prior works did \cite{liu2019roberta, he2020deberta}.
All models are trained on NVIDIA DGX SuperPOD cluster, in which each machine contains 8$\times$40GB A100 GPUs.

\subsection{Results on GLUE Benchmark}
Table \ref{tab:addlabel-1} shows the performance of SGD, SAM, AMSGrad, and AdaSAM.  For the AdaSAM, we tune the neighborhood size of the perturbation parameter from 0.01, 0.005, and 0.001.
The result shows that AdaSAM outperforms AMSGrad on 6 tasks of 8 tasks except for QNLI and QQP.
Overall, it improves the 0.28 average score than AMSGrad.
On the other hand, Table \ref{tab:addlabel-1} indicates that SAM is better than SGD on 7 tasks of 8 tasks except for RTE. And SAM can significantly improve performance. Comparing the results of Table \ref{tab:addlabel-1}, we can find that the adaptive learning rate method is better than SGD tuned with handicraft learning rate.
AdaSAM achieves the best metric on 6 tasks which is CoLA, SST-2, MRPC, STS-B, RTE, QNLI, and MNLI.
In general, AdaSAM is better than the other methods.

In addition, Figure \ref{fig:s_exp_1} shows the convergence speed of the detailed loss and evaluation metrics vs. the number of steps during training, respectively. The loss curve of AdaSAM decreases faster than SAM and SGD in all tasks, and it has a similar decreasing speed as the AMSGrad. 
The evaluation metric curve of AdaSAM and AMSGrad show that the AdaSAM is better than SGD and SAM and decreases the loss value as faster as the AMSGrad in all tasks.

\subsection{Mini-batch Speedup}
In this part, we test the performance with different batch sizes to validate the linear speedup property. The experiments are conducted on the MRPC, RTE, and CoLA tasks. The batch size is set as 4, 8, 16, 32, respectively. We scale the learning rate as $\sqrt{N}$, which is similar as \cite{li2021distributed}, where $N$ is the batch size. The results show that the training loss decreases faster as the batchsize increases, and the loss curve with the batch size of 32 achieves nearly half iterations as the curve with the batch size of 16.

\subsection{Ablation Study}
In this subsection, we conduct the experiments the momentum hyper-parameter $\beta_1$ is set to 0 to evaluate the influence of the momentum acceleration and the adaptive learning rate.
Table \ref{tab:addlabel-3} shows that AdaSAM outperforms AMSGrad on 6 tasks of 8 tasks except for SST-2 and RTE.
In Table \ref{tab:addlabel-3}, we also compare SGD and SAM, and without the momentum, SAM outperforms SGD on all tasks.
Under this situation, AdaSAM without the momentum acceleration method is better than the other methods.

When comparing the result of Table \ref{tab:addlabel-1} and Table \ref{tab:addlabel-3}, we find that both the adaptive learning rate method and the momentum acceleration are helpful for the model's generalization ability. When there is no momentum term, SAM with an adaptive learning rate improves the 0.74 average score to AMSGrad. With a momentum term, AdaSAM improves the 0.28 average score to AMSGrad. It shows that the adaptive method can improve the performance with or without momentum acceleration and it achieves the best performance with momentum acceleration. And we can find that momentum acceleration improves the performance of SAM, AMSGrad and AdaSAM.

\section{Conclusion}

In this work, we study the convergence rate of Sharpness aware minimization optimizer with an adaptive learning rate and momentum acceleration, dubbed AdaSAM in the stochastic non-convex setting. To the best of our knowledge, we are the first to provide the non-trivial $\mathcal{O}(1/\sqrt{bT})$ convergence rate of AdaSAM, which achieves a linear speedup property with respect to mini-batch size $b$.
We have conducted extensive experiments on several NLP tasks, which verifies that AdaSAM could achieve superior performance compared with AMSGrad and SAM optimizers. Future works include extending AdaSAM to the distributed setting and reducing the twice gradient back-propagation cost.

\ifCLASSOPTIONcaptionsoff
  \newpage
\fi




\bibliographystyle{IEEEtran}
\bibliography{reference}

%



%


\newpage
\onecolumn

\appendices
In this supplementary material, we give additional discussion on this paper.  In  Appendix \ref{expsetting}, detailed experimental settings such as some  hyper-parameters are listed.
In Appendix \ref{mainproof}, we first give the proof, then we give some useful lemmas to help proving the main theorem.
In Appendix \ref{additional-exp}, we provide additional experiment illustration.

\section{Experimental Settings}
\label{expsetting}

\begin{table*}[ht]
\centering
\caption{Experimental settings and data divisions upon different downstream tasks. Notably, for each tasks in GLUE benchmark, we provide the number of classes (``classes"), the learning rate (``lr"), the batch size (``bsz"), the total number of updates (``total"), the number of warmup updates (``warmup") and the number of GPUs (``GPUs") during fine-tuning, respectively.}
\begin{tabular}{ccccccccc}
\toprule
\multicolumn{1}{c}{}                   & \textbf{MNLI}    & \textbf{QNLI}    & \textbf{QQP}     & \textbf{RTE}   & \textbf{SST-2}  & \textbf{MRPC}  & \textbf{CoLA}  & \multicolumn{1}{c}{\textbf{STS-B}}    \\ \hline
\multicolumn{9}{c}{\cellcolor{lightgray}\textit{experimental settings upon different downstream tasks}}  \\ \hline
\multicolumn{1}{c}{--classes}      & 3       & 2       & 2       & 2     & 2      & 2     & 2     & \multicolumn{1}{c}{1}            \\
\multicolumn{1}{c}{--lr}    & 1e-5    & 1e-5    & 1e-5    & 2e-5  & 1e-5   & 1e-5  & 1e-5  & \multicolumn{1}{c}{2e-5}       \\
\multicolumn{1}{c}{--bsz}       & 256     & 128      & 256      & 32    & 64    & 32    & 32    & \multicolumn{1}{c}{32}        \\
\multicolumn{1}{c}{--total} & 15,484   & 8,278   & 14,453   & 1,018  & 10,467   & 1,148  & 2,668  & \multicolumn{1}{c}{1,799}       \\
\multicolumn{1}{c}{--warmup}    & 929     & 496     & 867    & 61    & 628    & 68    & 160    & \multicolumn{1}{c}{107}         \\
\multicolumn{1}{c}{--GPUs}         & 4       & 4       & 8       & 2     & 2      & 2     & 2     & \multicolumn{1}{c}{2}          \\ \hline
\multicolumn{9}{c}{\cellcolor{lightgray}\textit{data divisions for each dataset}}                                 \\ \hline
\multicolumn{1}{c}{train}              & 392,720 & 104,743 & 363,870 & 2,491 & 67,350 & 5,801 & 8,551 & \multicolumn{1}{c}{5,749}  \\
\multicolumn{1}{c}{dev}                & 9,815   & 5,463   & 40,431  & 277   & 873    & 4,076 & 1,043 & \multicolumn{1}{c}{1,500}    \\
\multicolumn{1}{c}{test}               & 9,796   & 5,461   & 390,956 & 3,000 & 1,821  & 1,725 & 1,063 & \multicolumn{1}{c}{1,379}   \\ \bottomrule
\end{tabular}
\end{table*}

The GLUE benchmark contains 8 tasks, they are RTE, STS-B, CoLA, SST-2, MNLI, MRPC, QNLI and QQP.
 CoLA is a single sentence task.
 Each sentence has a label 1 and -1. 
 1 represents that it is a grammatical sentence, while -1 represents that it is illegal.
 Matthews correlation coefficient, dubbed {\bf mcc} is used as our evaluation metric.
 STS-B is a similarity and paraphrase task.
 Each sample has a pair of a paragraph. 
 People annotated the sample from 1 to 5 based on the similarity between the two paragraphs. 
 The metric is Pearson and Spearman, dubbed {\bf p/s} correlation coefficients.
 RTE is an inference task.
 Each sample has two sentences. 
 If two sentences have a relation of entailment, we view them as a positive sample.
 If not, they compose of a negative sample.
 In the RTE task, the metric is the accuracy, dubbed {\bf acc}. 
 SST-2 is a single sentence task and its metric is the accuracy.
 MNLI is a sentence-level task that has 3 classes. They are entailment, contradiction and neutral.
 MRPC is a task to classify whether the sentences in the pair are equivalent. 
 QNLI is a question-answering task. If the sentence contains the answer to the question, then it is a positive sample.
 QQP is a social question-answering task that consists of question pairs from Quora.
 It determines whether the questions are equivalent. 
 The metric of MNLI, MRPC, QNLI, QQP is accuracy.

\section{Proof of the Main Results}
\label{mainproof}
We set $z_t = x_t +\frac{\beta_1}{1-\beta_1}(x_{t}-x_{t-1})$ for $t \ge 0$ and we assume $x_{-1} = 0$ and $m_{-1}=0$.

 We have that
 \begin{align}
z_{t+1} - z_t = x_{t+1} +\frac{\beta_1}{1-\beta_1}(x_{t+1}-x_{t}) -x_t -\frac{\beta_1}{1-\beta_1}(x_{t}-x_{t-1}) \\
=\frac{1}{1-\beta_1}(x_{t+1}-x_t)-\frac{\beta_1}{1-\beta_1}(x_{t}-x_{t-1})\\
=-\frac{1}{1-\beta_1}\gamma \od{m_{t}}{\eta_t} +\frac{\beta_1}{1-\beta_1}(x_{t}-x_{t-1}) \gamma \od{m_{t-1}}{\eta_{t-1}}\\
=-\frac{1}{1-\beta_1} \gamma \od{ (\beta_{1} m_{t-1}+(1-\beta_1)g_t ) }{ \eta_{t} } +\frac{\beta_1}{1-\beta_1}(x_{t}-x_{t-1}) \gamma \od{m_{t-1}}{\eta_{t-1}} \\
=\frac{\beta_1}{1-\beta_{1}}\gamma \od{m_{t-1}}{(\eta_{t-1}-\eta_{t})}-\gamma \od{g_t}{\eta_{t}}
 \end{align}
 By applying L-smooth, we have
 \begin{align}
     f(z_{t+1}) \leq f(z_{t}) + \ip{\nabla f(z_{t})}{z_{t+1}-z_{t}} +\frac{L}{2}\norm[2]{}{z_{t+1}-z_{t}}
 \end{align}
We re-organize it, and we have 
 \begin{align}
&f(z_{t+1}) -f(z_{t}) \nonumber\\
&\leq \ip{\nabla f(z_{t})}{z_{t+1}-z_{t}} + \frac{L}{2} \norm[2]{}{z_{t+1}- z_{t}} \\
&=\ip{\nabla f(z_{t})}{ \frac{\gamma \beta_1 }{1- \beta_1} \od{m_{t-1}}{(\eta_{t-1} - \eta_{t})} } + \ip{\nabla f(z_{t})}{-\gamma \od{g_{t}}{\eta_{t}} } +\frac{L}{2} \norm[2]{}{z_{t+1}- z_{t}}\\
&= \ip{\nabla f(z_{t})}{ \frac{\gamma \beta_1 }{1- \beta_1} \od{m_{t-1}}{(\eta_{t-1} - \eta_{t})} } + \frac{L}{2} \norm[2]{}{z_{t+1}- z_{t}} \nonumber\\ 
& \;\;\;\;+\ip{\nabla f(z_{t})}{\frac{\gamma_{t}}{b} \sum_{i \in B } \nabla f_{i}(x_{t} + \rho_{t} \frac{s_{t} }{\norm[]{}{s_{t}} }) \odot (\eta_{t-1}-\eta_{t}) } \nonumber \\
&\;\;\;\;+\ip{\nabla f(z_{t})}{-\frac{\gamma_{t}}{b} \sum_{i \in B } \nabla f_{i}(x_{t} + \rho_{t} \frac{s_{t} }{\norm[]{}{s_{t}}} ) \odot \eta_{t-1} }\\
&= \ip{\nabla f(z_{t})}{ \frac{\gamma \beta_1 }{1- \beta_1} \od{m_{t-1}}{(\eta_{t-1} - \eta_{t})} } + \frac{L}{2} \norm[2]{}{z_{t+1}- z_{t}} \nonumber\\ 
& \;\;\;\;+\ip{\nabla f(z_{t})}{\frac{\gamma_{t}}{b} \sum_{i \in B } \nabla f_{i}(x_{t} + \rho_{t} \frac{s_{t} }{\norm[]{}{s_{t}} }) \odot (\eta_{t-1}-\eta_{t}) } \nonumber \\
&\;\;\;\;+\ip{ \nabla f(z_{t}) -\nabla f(x_t) }{-\frac{\gamma_{t}}{b} \sum_{i \in B } \nabla f_{i}(x_{t} + \rho_{t} \frac{s_{t} }{\norm[]{}{s_{t}} }) \odot \eta_{t-1} } \nonumber \\
&\;\;\;\;+\ip{\nabla f(x_{t})}{-\frac{\gamma_{t}}{b} \sum_{i \in B } \nabla f_{i}(x_{t} + \rho_{t} \frac{s_{t} }{\norm[]{}{s_{t}} }) \odot \eta_{t-1} }\\
&=  \ip{\nabla f(z_{t})}{ \frac{\gamma \beta_1 }{1- \beta_1} \od{m_{t-1}}{(\eta_{t-1} - \eta_{t})} } + \frac{L}{2} \norm[2]{}{z_{t+1}- z_{t}}  \nonumber \\
& \;\;\;\;+\ip{\nabla f(z_{t})}{\frac{\gamma_{t}}{b} \sum_{i \in B } \nabla f_{i}(x_{t} + \rho_{t} \frac{s_{t} }{\norm[]{}{s_{t}} }) \odot (\eta_{t-1}-\eta_{t}) } \nonumber \\
&\;\;\;\;+\ip{ \nabla f(z_{t}) -\nabla f(x_t) }{-\frac{\gamma_{t}}{b} \sum_{i \in B } \nabla f_{i}(x_{t} + \rho_{t} \frac{\sum _{} \nabla f_{i}(x_{t}) }{\norm[]{}{\sum _{} \nabla f_{i}(x_{t})} }) \odot \eta_{t-1} } \nonumber \\
&\;\;\;\;+\ip{\nabla f(x_{t})}{\frac{\gamma_{t}}{b} \sum_{i \in B } \nabla f_{i}(x_{t} + \rho_{t} \frac{ \nabla f(x_{t}) }{\norm[]{}{ \nabla f(x_{t})} }) \odot \eta_{t-1} -\frac{\gamma_{t}}{b} \sum_{i \in B } \nabla f_{i}(x_{t} + \rho_{t} \frac{s_{t} }{\norm[]{}{s_{t}} }) \odot \eta_{t-1} } \nonumber \\
&\;\;\;\; +\ip{\nabla f(x_{t})}{-\frac{\gamma_{t}}{b} \sum_{i \in B } \nabla f_{i}(x_{t} + \rho_{t} \frac{ \nabla f(x_{t}) }{\norm[]{}{ \nabla f(x_{t})} })  \odot \eta_{t-1}}.
 \end{align}
 
 From the Lemma \ref{lemma:term1}, Lemma \ref{lemma:term1-1}, Lemma \ref{lemma:term2}, we have 
 \begin{align}
&\ip{\nabla f(z_{t})}{\frac{\gamma_{t}}{b} \sum_{i \in B } \nabla f_{i}(x_{t} + \rho_{t} \frac{s_{t} }{\norm[]{}{s_{t}} }) \odot (\eta_{t-1}-\eta_{t}) } \leq \gamma_{t}G^{2} \norm[]{1}{\eta_{t-1} - \eta_{t}},\\
&\ip{\nabla f(z_{t})}{ \frac{\gamma \beta_1 }{1- \beta_1} \od{m_{t-1}}{(\eta_{t-1} - \eta_{t})} } \leq \frac{\gamma \beta_1 }{1- \beta_1} G^{2} \norm[]{1}{\eta_{t-1} - \eta_{t}},\\
&\ip{\nabla f(x_{t})}{\frac{\eta_{t}}{b} \sum_{i \in B } \nabla f_{i}(x_{t} + \rho_{t} \frac{ \nabla f(x_{t}) }{\norm[]{}{ \nabla f(x_{t})} })\odot \eta_{t-1} -\frac{\gamma_{t}}{b} \sum_{i \in B } \nabla f_{i}(x_{t} + \rho_{t} \frac{s_{t} }{\norm[]{}{s_{t}} })\odot \eta_{t-1} } \nonumber \\
&\leq \frac{\gamma_{t}}{2 \mu^2} \norm[2]{}{\nabla f(x_{t})\odot \sqrt{\eta_{t-1}} } +\frac{2 \mu^2 \gamma_{t} L^2 \rho_{t}^2}{\epsilon}.
\end{align}

Taking conditional expectation, we have
 \begin{align}
 &\EE f(z_{t+1}) - f(z_{t}) \\
 &\leq \EE \ip{\nabla f(x_{t})}{-\frac{\gamma_{t}}{b} \sum_{i \in B } \nabla f_{i}(x_{t} + \rho_{t} \frac{ \nabla f(x_{t}) }{\norm[]{}{ \nabla f(x_{t})} })  \odot \eta_{t-1}} + \frac{L}{2} \EE \norm[2]{}{z_{t+1}- z_{t}}  \nonumber \\
 &+\frac{\gamma_{t}}{2 \mu^2} \norm[2]{}{\nabla f(x_{t})\odot \sqrt{\eta_{t-1}} } +\frac{2 \mu^2 \gamma_{t} L^2 \rho_{t}^2}{\epsilon}  +  \frac{\gamma }{1- \beta_1 }G^{2} \norm[]{1}{\eta_{t-1} - \eta_{t}} \nonumber \\
 &+\EE \ip{ \nabla f(z_{t}) -\nabla f(x_t) }{-\frac{\gamma_{t}}{b} \sum_{i \in B } \nabla f_{i}(x_{t} + \rho_{t} \frac{s_{t}}{\norm[]{}{s_{t}} }) \odot \eta_{t-1} }
 \end{align}
 where $\mu>0$ is   to be determined.

 For the term 
 \begin{align}
\;\;\;\;\EE \ip{\nabla f(x_{t})}{-\frac{\gamma_{t}}{b} \sum_{i \in B } \nabla f_{i}(x_{t} + \rho_{t} \frac{ \nabla f(x_{t}) }{\norm[]{}{ \nabla f(x_{t})} })  \odot \eta_{t-1}}, 
 \end{align}
  the term 
 \begin{align}
 \frac{L}{2} \EE \norm[2]{}{z_{t+1}- z_{t}},
 \end{align}
 and the term 
 \begin{align}
 \EE \ip{ \nabla f(z_{t}) -\nabla f(x_t) }{-\frac{\gamma_{t}}{b} \sum_{i \in B } \nabla f_{i}(x_{t} + \rho_{t} \frac{s_{t} }{\norm[]{}{s_{t}} }) \odot \eta_{t-1} },
 \end{align}
 we introduce the Lemma \ref{lemma:term3}, the Lemma \ref{lemma:term4-0} and the Lemma \ref{lemma:term-m}.
We take the expectation over the whole processing and we have
 \begin{align}
& \EE f(z_{t+1}) - \EE f(z_{t}) \nonumber \\
&\leq  \frac{\gamma_{t}}{2 \mu^2} \EE \norm[2]{}{\nabla f(x_{t})\odot \sqrt{\eta_{t-1}} } +\frac{2 \mu^2 \gamma_{t} L^2 \rho_{t}^2}{\epsilon}  +  \frac{\gamma}{1 - \beta_1} G^{2} \EE \norm[]{1}{\eta_{t-1} - \eta_{t}} \nonumber  \\
&- \gamma_{t} \EE \norm[2]{}{\nabla f(x_{t}) \odot \sqrt{\eta_{t-1}}} +\EE \frac{\gamma_{t}}{2\alpha^2}   \EE \norm[2]{}{\nabla f(x_{t}) \odot \sqrt{\eta_{t-1}}} +\frac{\gamma_{t} \alpha^2 L^2 \rho^2 }{2  \epsilon} +\frac{L G^2 \gamma^{2} \beta_{1}^{2}}{(1-\beta_1)^{2}}\EE \norm[2]{}{\eta_{t} - \eta_{t-1}} \nonumber \\
&+ \gamma_{t}^{2}L (3\frac{1 + \beta}{\beta \epsilon}( \EE \norm[2]{}{ \nabla f(x_t ) \odot  \sqrt{\eta_{t-1}}} +\frac{L \rho_{t}^{2}}{\epsilon } + \frac{\sigma^{2}}{b \epsilon}) + (1+ \beta)G^2 \EE \norm[2]{}{\eta_{t} -\eta_{t-1} })\nonumber \\
&+ \frac{\gamma^3 L^2 \beta_{1}^2}{2 \epsilon (1-\beta_{1})^{2}}(\frac{1}{\lambda_{1}^{2}}+\frac{1}{\lambda_{2}^{2}} +\frac{1}{\lambda_{3}^{2}})  \frac{d G_{\infty}^{2}}{\epsilon^{2}} +\frac{\gamma \lambda_{1}^{2}}{2} \norm[2]{}{\nabla f(x_{t}) \odot \sqrt{\eta_{t-1} }} + \frac{\gamma  L^2 \rho_{t}^2 }{2\epsilon}(\lambda_{2}^{2} +4  \lambda_{3}^2) \\
&=-\gamma_{t}(1 - \frac{1}{2 \mu^2} -\frac{1}{2 \alpha^2}-\frac{3\gamma L(1+\beta)}{\beta \epsilon } - \frac{\lambda_{1}^{2}}{2})\EE \norm[2]{}{ \nabla f(x_t ) \odot  \sqrt{\eta_{t-1}}} +\frac{2 \mu^2 \gamma_{t} L^2 \rho_{t}^2}{\epsilon}  +  \frac{\gamma}{1-\beta_{1}} G^{2} \EE \norm[]{1}{\eta_{t-1} - \eta_{t}}\nonumber \\
& +\frac{\gamma_{t} \alpha^2 L^2 \rho^2 }{2  \epsilon} + \frac{3 \gamma_{t}^2 L (1+ \beta)}{ \beta \epsilon}(\frac{L \rho_{t}^{2}}{\epsilon } + \frac{\sigma^{2}}{b \epsilon}) +\gamma_{t}^{2}L G^2 ((\frac{\beta_1}{1-\beta_1})^2 +1 +\beta) \EE \norm[2]{}{\eta_{t} -\eta_{t-1}}\nonumber \\
&+\frac{\gamma^3 L^2 \beta_{1}^2}{2 \epsilon (1-\beta_{1})^{2}}(\frac{1}{\lambda_{1}^{2}}+\frac{1}{\lambda_{2}^{2}} +\frac{1}{\lambda_{3}^{2}})  \frac{d G_{\infty}^{2}}{\epsilon^{2}}  +\frac{\gamma  L^2 \rho_{t}^2 }{2\epsilon}(\lambda_{2}^{2} +4  \lambda_{3}^2).
 \end{align}
 We set $\mu^2 = \alpha^2 = 8$, $\beta = 3$, $\lambda_{1}^2 =\frac{1}{4} $, $\lambda_{2}^2 =\lambda_{3}^2 =1 $ and we choose $\frac{2 \gamma_{t} L}{\epsilon} \leq \frac{1}{8}$.
 So we have 
 \begin{align}
& \EE f(x_{t+1}) - \EE f(x_{t}) \nonumber \\
&\leq -\frac{\gamma_{t}}{2} \EE \norm[2]{}{ \nabla f(x_t ) \odot  \sqrt{\eta_{t-1}}} + \frac{16 \gamma_{t} L^2 \rho_{t}^2}{\epsilon}  +  \frac{\gamma}{1-\beta_1} G^{2} \EE \norm[]{1}{\eta_{t-1} - \eta_{t}}\nonumber \\
& +\frac{4 \gamma_{t}  L^2 \rho^2 }{ \epsilon} + \frac{4 \gamma_{t}^2 L }{ \epsilon}(\frac{L \rho_{t}^{2}}{\epsilon } + \frac{\sigma^{2}}{b \epsilon}) +(4+(\frac{\beta_1}{1-\beta_1})^2 )\gamma_{t}^{2}L G^2  \EE \norm[2]{}{\eta_{t} -\eta_{t-1}} \nonumber \\
& +\frac{3 \gamma^3 L^2 \beta_{1}^2}{ \epsilon (1-\beta_{1})^{2}}  \frac{d G_{\infty}^{2}}{\epsilon^{2}}  +\frac{5 \gamma  L^2 \rho_{t}^2 }{2\epsilon}
 \end{align}
 We re-arrange it and $ \eta_{t} $ is bounded. We have
 \begin{align}
&\frac{\gamma_{t}}{2 G} \EE \norm[2]{}{ \nabla f(x_t )} \leq \frac{\gamma_{t}}{2} \EE \norm[2]{}{ \nabla f(x_t ) \odot  \sqrt{\eta_{t-1}}}\\
&\leq - \EE f(x_{t+1}) + \EE f(x_{t}) + \frac{45 \gamma_{t} L^2 \rho_{t}^2}{2 \epsilon}  +  \frac{\gamma}{1-\beta_1}G^{2} \EE \norm[]{1}{\eta_{t-1} - \eta_{t}}\nonumber \\
&  + \frac{4 \gamma_{t}^2 L }{ \epsilon}(\frac{L \rho_{t}^{2}}{\epsilon } + \frac{\sigma^{2}}{b \epsilon}) +(4+(\frac{\beta_1}{1-\beta_1})^2 )\gamma_{t}^{2}L G^2  \EE \norm[2]{}{\eta_{t} -\eta_{t-1}} +\frac{3 \gamma^3 L^2 \beta_{1}^2}{  (1-\beta_{1})^{2}}  \frac{d G_{\infty}^{2}}{\epsilon^{3}}.
 \end{align}
 
 We summary it from $t=0$ to $t=T-1$, and we assume $\gamma_{t}$ is a constant, and we have
 \begin{align}
&\frac{1}{T}\sum_{t=0}^{T-1} \EE \norm[2]{}{ \nabla f(x_t )} \leq  2G\frac{\EE f(x_{0})- \EE f(x_{t+1})}{\gamma_{t} T} + \frac{45G  L^2 \rho_{t}^2}{ \epsilon}  +  \frac{2G^{3}}{(1-\beta_1)T} \EE \sum_{t=0}^{T-1}\norm[]{1}{\eta_{t-1} - \eta_{t}}\nonumber \\
& + \frac{8G \gamma_{t} L }{ \epsilon}(\frac{L \rho_{t}^{2}}{\epsilon } + \frac{\sigma^{2}}{b \epsilon}) + \frac{2(4+(\frac{\beta_1}{1-\beta_1})^2 )\gamma_{t}L G^3}{T}\EE \sum_{t=0}^{T-1}\norm[2]{}{\eta_{t} -\eta_{t-1}} +\frac{6 \gamma^2 L^2 \beta_{1}^2}{  (1-\beta_{1})^{2}}  \frac{d G^{3}}{\epsilon^{3}}\\
&\leq \frac{2G(f(x_{0})-f^{*})}{\gamma_{t} T} + \frac{45G  L^2 \rho_{t}^2}{\epsilon}  +  \frac{2G^{3}}{(1-\beta_1) T} 
d(\frac{1}{\epsilon}-\frac{1}{G})+ \frac{8G \gamma_{t} L }{ \epsilon}(\frac{L \rho_{t}^{2}}{\epsilon } + \frac{\sigma^{2}}{b \epsilon})\nonumber \\
&  + \frac{2(4+(\frac{\beta_1}{1-\beta_1})^2 )\gamma_{t}L G^3}{T}  d(\epsilon ^{-2}-G^{-2})+ \frac{6 \gamma^2 L^2 \beta_{1}^2}{  (1-\beta_{1})^{2}}  \frac{d G^{3}}{\epsilon^{3}}\\
&=\frac{2G(f(x_{0})-f^{*})}{\gamma_{t} T} + \frac{8G \gamma_{t} L }{ \epsilon} \frac{\sigma^{2}}{b \epsilon} + \frac{45G  L^2 \rho_{t}^2}{ \epsilon} + \frac{2G^{3}}{(1-\beta_1)T} 
d(\frac{1}{\epsilon}-\frac{1}{G}) +\frac{8G \gamma_{t} L }{ \epsilon} \frac{L \rho_{t}^{2}}{\epsilon } \nonumber \\
& + \frac{2(4+(\frac{\beta_1}{1-\beta_1})^2 )\gamma_{t}L G^3}{T}  d(\epsilon ^{-2}-G^{-2})+ \frac{6 \gamma^2 L^2 \beta_{1}^2}{  (1-\beta_{1})^{2}}  \frac{d G^{3}}{\epsilon^{3}}.
 \end{align}

 \subsection{Technical Lemma}
 \begin{lemma}
Given two vectors $a$, $b\in \R^{d}$, we have
$\langle a , b\rangle \leq \frac{\lambda^2}{2} \Vert a \Vert^{2} +\frac{1}{2 \lambda^{2}} \Vert b \Vert^{2} $  for parameter $\lambda$, $\forall \lambda \in (1, + \infty )$. 
\label{lemma:innerineq}
\end{lemma}
\begin{proof}
\begin{align}
    RHS =  \frac{\lambda^2}{2}  \sum_{j=1}^{d} (a)_{j}^2 + \frac{1}{2 \lambda^{2}} \sum_{j=1}^{d} (b)_{j}^2 \ge \sum_{j=1}^{d} 2\sqrt{ \frac{\lambda^2}{2} (a)_{j}^2 \times    \frac{1}{2 \lambda^{2}} (b)_{j}^2}= \sum_{j=1}^{d}   |(a)_{j}| \times |(b)_{j}| \ge LHS.
\end{align}
\end{proof}

\begin{lemma}
For any vector $x$,$y\in \R^{d}$, we have
\begin{align}
    \Vert x\odot y \Vert^2 \leq \Vert x \Vert^{2} \times \Vert y \Vert_{\infty}^{2}  \leq \norm{}{x}\times \norm{}{y}.
\end{align}
\label{lemma:baseineq2}
\end{lemma}
\begin{proof}
The first inequality can be derived from that $\sum_{i=1}^{d} (x_{i}^{2}y_{i}^{2}) \leq \sum_{i=1}^{d} (x_{i}^{2} \Vert y \Vert_{\infty}^{2}) $.
The second inequality follows from that $ \Vert y \Vert_{\infty}^{2}  \leq \norm{}{y}$. 
\end{proof}

\begin{lemma}
$\eta$  is bounded, i.e.,  $ \frac{1}{G_{\infty}} \leq (\eta_{t})_{j} \leq \frac{1}{\epsilon}$.
\label{lemma:boundeta}
\end{lemma}
\begin{proof}
As the gradient is bounded by $G$ and $(\eta_{t})_{j} = \frac{1}{\sqrt{( \hat{v}_{t})_{j}}}$. Follow the update rule, we have $ \frac{1}{G_{\infty}} \leq (\eta_{t})_{j} \leq \frac{1}{\epsilon}$.
\end{proof}

\begin{lemma}
For the term defined in the algorithm, we have 
\begin{align}
    &\frac{1}{T} \EE \sum_{t=0}^{T-1} \norm[1]{}{\eta_{t-1} - \eta_{t}} \leq \frac{d}{T}(\frac{1}{\epsilon} - \frac{1}{G})
\end{align}
\end{lemma}
\begin{proof}
$(\eta_{t})_{i}$, the i-th dimension of $\eta_{t}$ deceases as t increases. So we have
\begin{align}
    &\frac{1}{T} \EE \sum_{t=0}^{T-1} \norm[1]{}{\eta_{t-1} - \eta_{t}}=\EE \frac{1}{T} \sum_{i=1}^{d} \sum_{t=0}^{T-1} |(\eta_{t-1})_{i} - (\eta_{t})_{i}|\nonumber \\
    &\leq \EE \frac{1}{T} \sum_{i=1}^{d} ((\eta_{-1})_{i} - (\eta_{T-1})_{i}) \leq \EE \frac{1}{T} \sum_{i=1}^{d}(\frac{1}{\epsilon}-\frac{1}{G})=\frac{d}{T}(\frac{1}{\epsilon}-\frac{1}{G})
\end{align}
\end{proof}

\begin{lemma}
For the term defined in the algorithm, we have
 \begin{align}
 &\ip{\nabla f(z_{t})}{\frac{\gamma_{t}}{b} \sum_{i \in B } \nabla f_{i}(x_{t} + \rho_{t} \frac{s_{t}}{\norm[]{}{s_{t}} }) \odot (\eta_{t-1}-\eta_{t}) }\leq \gamma_{t}G^{2} \norm[]{1}{\eta_{t-1} - \eta_{t}}
 \end{align}
 \label{lemma:term1}
\end{lemma}
\begin{proof}
\begin{align}
 &\ip{\nabla f(z_{t})}{\frac{\gamma_{t}}{b} \sum_{i \in B } \nabla f_{i}(x_{t} + \rho_{t} \frac{s_{t} }{\norm[]{}{s_{t}} }) \odot (\eta_{t-1}-\eta_{t}) } \nonumber \\
 &\leq\gamma_{t} \sum_{j=1}^{d}|(\nabla f(z_{t}))_{(j)}| \times |(\frac{1}{b} \sum_{i \in B } \nabla f_{i}(x_{t} + \rho_{t} \frac{\sum _{} \nabla f_{i}(x_{t}) }{\norm[]{}{\sum _{} \nabla f_{i}(x_{t})} }) \odot (\eta_{t-1}-\eta_{t}))_{(j)} | \\
  &\leq\gamma_{t}G \sum_{j=1}^{d}  |((\frac{1}{b} \sum_{i \in B } \nabla f_{i}(x_{t} + \rho_{t} \frac{\sum _{} \nabla f_{i}(x_{t}) }{\norm[]{}{\sum _{} \nabla f_{i}(x_{t})} }) \odot (\eta_{t-1}-\eta_{t}))_{(j)} | \\
 & \leq \frac{\gamma_{t}G}{b} \sum_{j=1}^{d}  \sum_{i \in B } |( ( \nabla f_{i}(x_{t} + \rho_{t} \frac{\sum _{} \nabla f_{i}(x_{t}) }{\norm[]{}{\sum _{} \nabla f_{i}(x_{t})} }) \odot (\eta_{t-1}-\eta_{t}))_{(j)} | \\
 &=\frac{\gamma_{t}G}{b} \sum_{j=1}^{d}  \sum_{i \in B } | ( \nabla f_{i}(x_{t} + \rho_{t} \frac{\sum _{} \nabla f_{i}(x_{t}) }{\norm[]{}{\sum _{} \nabla f_{i}(x_{t})} })_{(j)} \times (\eta_{t-1}-\eta_{t})_{(j)} |\\
 &\leq \frac{\gamma_{t}G^2}{b} \sum_{j=1}^{d}  \sum_{i \in B } |  (\eta_{t-1}-\eta_{t})_{(j)} | \\
 &= \gamma_{t}G^{2} \norm[]{1}{\eta_{t-1} - \eta_{t}}
 \end{align}
\end{proof}

\begin{lemma}\label{lemma:term1-1}
 For the term defined in the algorithm, we have
 \begin{align}
\ip{\nabla f(z_{t})}{ \frac{\gamma \beta_1 }{1- \beta_1} \od{m_{t-1}}{(\eta_{t-1} - \eta_{t})} } \leq \frac{\gamma \beta_1 }{1- \beta_1} G^{2} \norm[]{1}{\eta_{t-1} - \eta_{t}}
 \end{align}
\end{lemma}
\begin{proof}
\begin{align}
 &\ip{\nabla f(z_{t})}{ \frac{\gamma \beta_1 }{1- \beta_1} \od{m_{t-1}}{(\eta_{t-1} - \eta_{t})} } \nonumber \\
 &\leq\frac{\gamma \beta_1 }{1- \beta_1} \sum_{j=1}^{d}|(\nabla f(z_{t}))_{(j)}| \times |(\od{m_{t-1}}{(\eta_{t-1} - \eta_{t})})_{(j)} | \\
  &\leq\frac{\gamma \beta_1 }{1- \beta_1} G \sum_{j=1}^{d}  |(\od{m_{t-1}}{(\eta_{t-1} - \eta_{t})})_{(j)} | \\
 &=\frac{\gamma \beta_1 }{1- \beta_1}  \sum_{j=1}^{d}   | ( m_{t-1} )_{(j)} \times (\eta_{t-1}-\eta_{t})_{(j)} |\\
 &\leq \frac{\gamma \beta_1 }{1- \beta_1} G^2 \sum_{j=1}^{d}   |  (\eta_{t-1}-\eta_{t})_{(j)} | \\
 &= \frac{\gamma \beta_1 }{1- \beta_1} G^{2} \norm[]{1}{\eta_{t-1} - \eta_{t}}
 \end{align}
\end{proof}

\begin{lemma}\label{lemma:term2}
 For the term defined in the algorithm, we have
 \begin{align}
& \ip{\nabla f(x_{t})}{\frac{\gamma_{t}}{b} \sum_{i \in B } \nabla f_{i}(x_{t} + \rho_{t} \frac{ \nabla f(x_{t}) }{\norm[]{}{ \nabla f(x_{t})} })\odot \eta_{t-1} -\frac{\gamma_{t}}{b} \sum_{i \in B } \nabla f_{i}(x_{t} + \rho_{t} \frac{s_{t} }{\norm[]{}{s_{t}} })\odot \eta_{t-1} }  \nonumber \\
&\leq \frac{\gamma_{t}}{2 \mu^2} \norm[2]{}{\nabla f(x_{t})\odot \sqrt{\eta_{t-1}} } +\frac{2 \mu^2 \gamma_{t} L^2 \rho_{t}^2}{\epsilon} .
 \end{align}
\end{lemma}

 \begin{proof}
 \begin{align}
& \ip{\nabla f(x_{t})}{\frac{\gamma_{t}}{b} \sum_{i \in B } \nabla f_{i}(x_{t} + \rho_{t} \frac{ \nabla f(x_{t}) }{\norm[]{}{ \nabla f(x_{t})} })\odot \eta_{t-1} -\frac{\gamma_{t}}{b} \sum_{i \in B } \nabla f_{i}(x_{t} + \rho_{t} \frac{s_{t} }{\norm[]{}{s_{t}} })\odot \eta_{t-1} }  \nonumber \\
&=\ip{\nabla f(x_{t})\odot \sqrt{\eta_{t-1}} }{\frac{\gamma_{t}}{b} \sum_{i \in B } (\nabla f_{i}(x_{t} + \rho_{t} \frac{ \nabla f(x_{t}) }{\norm[]{}{ \nabla f(x_{t})} }) - \nabla f_{i}(x_{t} + \rho_{t} \frac{\sum _{i \in B} \nabla f_{i}(x_{t}) }{\norm[]{}{\sum _{i \in B} \nabla f_{i}(x_{t})} }) )\odot \sqrt{\eta_{t-1}} } \\
&\leq \frac{\mu^2 \gamma_{t}}{2 b^2} \norm[2]{}{\sum (\nabla f_{i}(x_{t} + \rho_{t} \frac{ \nabla f(x_{t}) }{\norm[]{}{ \nabla f(x_{t})} }) - \nabla f_{i}(x_{t} + \rho_{t} \frac{\sum _{i \in B} \nabla f_{i}(x_{t}) }{\norm[]{}{\sum _{i \in B} \nabla f_{i}(x_{t})} }) ) \odot \sqrt{\eta_{t-1}}} \nonumber \\
&+\frac{\gamma_{t}}{2 \mu^2} \norm[2]{}{\nabla f(x_{t}) \odot \sqrt{\eta_{t-1}} } \\
& \leq+\frac{\mu^2 \gamma_{t}}{2 b} \sum  \norm[2]{}{\nabla f_{i}(x_{t} + \rho_{t} \frac{ \nabla f(x_{t}) }{\norm[]{}{ \nabla f(x_{t})} }) - \nabla f_{i}(x_{t} + \rho_{t} \frac{\sum _{i \in B} \nabla f_{i}(x_{t}) }{\norm[]{}{\sum _{i \in B} \nabla f_{i}(x_{t})} }) \odot \sqrt{\eta_{t-1}}} \nonumber\\
& +  \frac{\gamma_{t}}{2 \mu^2} \norm[2]{}{\nabla f(x_{t})\odot \sqrt{\eta_{t-1}} } \\
& \leq+\frac{\mu^2 \gamma_{t}}{2 b} \sum  \norm[2]{}{\nabla f_{i}(x_{t} + \rho_{t} \frac{ \nabla f(x_{t}) }{\norm[]{}{ \nabla f(x_{t})} }) - \nabla f_{i}(x_{t} + \rho_{t} \frac{\sum _{i \in B} \nabla f_{i}(x_{t}) }{\norm[]{}{\sum _{i \in B} \nabla f_{i}(x_{t})} }) }\times \norm[2]{\infty}{\sqrt{\eta_{t-1}}} \nonumber\\
& +  \frac{\gamma_{t}}{2 \mu^2} \norm[2]{}{\nabla f(x_{t})\odot \sqrt{\eta_{t-1}} } \\
&\leq \frac{\gamma_{t}}{2 \mu^2} \norm[2]{}{\nabla f(x_{t})\odot \sqrt{\eta_{t-1}} } +\frac{\mu^2 \gamma_{t} L^2 \rho_{t}^2 }{2 b \epsilon} \sum  \norm[2]{}{  \frac{ \nabla f(x_{t}) }{\norm[]{}{ \nabla f(x_{t})} }-   \frac{\sum _{i \in B} \nabla f_{i}(x_{t}) }{\norm[]{}{\sum _{i \in B} \nabla f_{i}(x_{t})} } }\\
&\leq \frac{\gamma_{t}}{2 \mu^2} \norm[2]{}{\nabla f(x_{t})\odot \sqrt{\eta_{t-1}} } +\frac{2 \mu^2 \gamma_{t} L^2 \rho_{t}^2}{\epsilon} .
 \end{align}
 \end{proof}
 
\begin{lemma} \label{lemma:term3}
For the term defined in the algorithm, we have
\begin{align}
&\;\;\;\;\EE \ip{\nabla f(x_{t})}{-\frac{\gamma_{t}}{b} \sum_{i \in B } \nabla f_{i}(x_{t} + \rho_{t} \frac{ \nabla f(x_{t}) }{\norm[]{}{ \nabla f(x_{t})} })  \odot \eta_{t-1}} \nonumber \\
&\leq - \gamma_{t}\norm[2]{}{\nabla f(x_{t}) \odot \sqrt{\eta_{t-1}}} +\EE \frac{\gamma_{t}}{2\alpha^2} \norm[2]{}{\nabla f(x_{t}) \odot \sqrt{\eta_{t-1}}} +\frac{\gamma_{t} \alpha^2 L^2 \rho_{t}^2 }{2  \epsilon}
 \end{align}
\end{lemma}
\begin{proof}
 \begin{align}
&\;\;\;\;\EE \ip{\nabla f(x_{t})}{-\frac{\gamma_{t}}{b} \sum_{i \in B } \nabla f_{i}(x_{t} + \rho_{t} \frac{ \nabla f(x_{t}) }{\norm[]{}{ \nabla f(x_{t})} })  \odot \eta_{t-1}} \nonumber \\
& = - \gamma_{t}\norm[2]{}{\nabla f(x_{t}) \odot \sqrt{\eta_{t-1}}} +\EE \ip{\nabla f(x_{t})}{\frac{\gamma_{t}}{b} \sum_{i \in B } (\nabla f(x_t) - \nabla f_{i}(x_{t} + \rho_{t} \frac{ \nabla f(x_{t}) }{\norm[]{}{ \nabla f(x_{t})} }))  \odot \eta_{t-1}} \\
& = - \gamma_{t}\norm[2]{}{\nabla f(x_{t}) \odot \sqrt{\eta_{t-1}}} +\EE \ip{\nabla f(x_{t})}{\frac{\gamma_{t}}{b} \sum_{i \in B } (\nabla f_{i}(x_t) - \nabla f_{i}(x_{t} + \rho_{t} \frac{ \nabla f(x_{t}) }{\norm[]{}{ \nabla f(x_{t})} }))  \odot \eta_{t-1}}\\
&\leq - \gamma_{t}\norm[2]{}{\nabla f(x_{t}) \odot \sqrt{\eta_{t-1}}} +\EE \frac{\gamma_{t}}{2\alpha^2} \norm[2]{}{\nabla f(x_{t}) \odot \sqrt{\eta_{t-1}}} \nonumber \\
&+\frac{\gamma_{t} \alpha^2}{2} \EE \norm[2]{}{\frac{1}{b} \sum_{i \in B } (\nabla f_{i}(x_t) - \nabla f_{i}(x_{t} + \rho_{t} \frac{ \nabla f(x_{t}) }{\norm[]{}{ \nabla f(x_{t})} }))  \odot \sqrt{\eta_{t-1}}}\\
&\leq - \gamma_{t}\norm[2]{}{\nabla f(x_{t}) \odot \sqrt{\eta_{t-1}}} +\EE \frac{\gamma_{t}}{2\alpha^2} \norm[2]{}{\nabla f(x_{t}) \odot \sqrt{\eta_{t-1}}} \nonumber \\
&+\frac{\gamma_{t} \alpha^2}{2 \epsilon} \EE \norm[2]{}{\frac{1}{b} \sum_{i \in B } (\nabla f_{i}(x_t) - \nabla f_{i}(x_{t} + \rho_{t} \frac{ \nabla f(x_{t}) }{\norm[]{}{ \nabla f(x_{t})} }))}\\
&\leq - \gamma_{t}\norm[2]{}{\nabla f(x_{t}) \odot \sqrt{\eta_{t-1}}} +\EE \frac{\gamma_{t}}{2\alpha^2} \norm[2]{}{\nabla f(x_{t}) \odot \sqrt{\eta_{t-1}}} \nonumber \\
&+\frac{\gamma_{t} \alpha^2}{2 b \epsilon} \EE \sum_{i \in B }\norm[2]{}{ (\nabla f_{i}(x_t) - \nabla f_{i}(x_{t} + \rho_{t} \frac{ \nabla f(x_{t}) }{\norm[]{}{ \nabla f(x_{t})} }))}\\
&\leq - \gamma_{t}\norm[2]{}{\nabla f(x_{t}) \odot \sqrt{\eta_{t-1}}} +\EE \frac{\gamma_{t}}{2\alpha^2} \norm[2]{}{\nabla f(x_{t}) \odot \sqrt{\eta_{t-1}}} +\frac{\gamma_{t} \alpha^2 L^2 \rho_{t}^2 }{2 b \epsilon} \EE \sum_{i \in B }\norm[2]{}{  \frac{ \nabla f(x_{t}) }{\norm[]{}{ \nabla f(x_{t})} }}\\
&=- \gamma_{t}\norm[2]{}{\nabla f(x_{t}) \odot \sqrt{\eta_{t-1}}} +\EE \frac{\gamma_{t}}{2\alpha^2} \norm[2]{}{\nabla f(x_{t}) \odot \sqrt{\eta_{t-1}}} +\frac{\gamma_{t} \alpha^2 L^2 \rho_{t}^2 }{2  \epsilon}
 \end{align}
\end{proof}

\begin{lemma} \label{lemma:term-m}
For the term defined in the algorithm, we have
\begin{align}
&\EE \ip{ \nabla f(z_{t}) -\nabla f(x_t) }{-\frac{\gamma_{t}}{b} \sum_{i \in B } \nabla f_{i}(x_{t} + \rho_{t} \frac{s_{t} }{\norm[]{}{s_{t}} }) \odot \eta_{t-1} } \nonumber \\
&\leq \frac{\gamma^3 L^2 \beta_{1}^2}{2 \epsilon (1-\beta_{1})^{2}}(\frac{1}{\lambda_{1}^{2}}+\frac{1}{\lambda_{2}^{2}} +\frac{1}{\lambda_{3}^{2}})  \frac{d G_{\infty}^{2}}{\epsilon^{2}} +\frac{\gamma \lambda_{1}^{2}}{2} \norm[2]{}{\nabla f(x_{t}) \odot \sqrt{\eta_{t-1} }} + \frac{\gamma  L^2 \rho_{t}^2 }{2\epsilon}(\lambda_{2}^{2} +4  \lambda_{3}^2).
\end{align}
\end{lemma}
 \begin{proof}
 \begin{align}
&\EE \ip{ \nabla f(z_{t}) -\nabla f(x_t) }{-\frac{\gamma_{t}}{b} \sum_{i \in B } \nabla f_{i}(x_{t} + \rho_{t} \frac{s_{t} }{\norm[]{}{s_{t}} }) \odot \eta_{t-1} } \\
&=\gamma \EE  \ip{( \od{\nabla f(x_t)-\nabla f(z_{t})) }{\sqrt{\eta_{t-1}}} }{\frac{1}{b} \sum_{i \in B } \nabla f_{i}(x_{t} + \rho_{t} \frac{\sum _{i \in B} \nabla f_{i}(x_{t}) }{\norm[]{}{\sum _{i \in B} \nabla f_{i}(x_{t})} }) \odot \sqrt{\eta_{t-1} }}\\
&=\gamma \EE \ip{( \od{\nabla f(x_t)-\nabla f(z_{t})) }{\sqrt{\eta_{t-1}}} }{ \nabla f(x_{t}) \odot \sqrt{\eta_{t-1} }}\nonumber \\
&+\gamma \EE \ip{( \od{\nabla f(x_t)-\nabla f(z_{t})) }{\sqrt{\eta_{t-1}}} }{\frac{1}{b} \sum_{i \in B } (\nabla f_{i}(x_{t} + \rho_{t} \frac{ \nabla f(x_{t}) }{\norm[]{}{ \nabla f(x_{t})} }) - \nabla f_{i}(x_{t}) ) \odot \sqrt{\eta_{t-1} }} \nonumber \\
&+\gamma \EE \ip{( \od{\nabla f(x_t)-\nabla f(z_{t})) }{\sqrt{\eta_{t-1}}} }{\frac{1}{b} \sum_{i \in B } (\nabla f_{i}(x_{t} + \rho_{t} \frac{\sum _{i \in B} \nabla f_{i}(x_{t}) }{\norm[]{}{\sum _{i \in B} \nabla f_{i}(x_{t})} }) -\nabla f_{i}(x_{t} + \rho_{t} \frac{ \nabla f(x_{t}) }{\norm[]{}{ \nabla f(x_{t})} })  \odot \sqrt{\eta_{t-1} }}\\
&\leq \frac{\gamma}{2}(\frac{1}{\lambda_{1}^{2}}+\frac{1}{\lambda_{2}^{2}} +\frac{1}{\lambda_{3}^{2}}) \EE \norm[2]{}{\od{(\nabla f(x_t)-\nabla f(z_{t})) }{\sqrt{\eta_{t-1}}}} +\frac{\gamma \lambda_{1}^{2}}{2} \norm[2]{}{\nabla f(x_{t}) \odot \sqrt{\eta_{t-1} }} \nonumber \\
&+\frac{\gamma \lambda_{2}^{2}}{2} \EE \norm[2]{}{\frac{1}{b} \sum_{i \in B } (\nabla f_{i}(x_{t} + \rho_{t} \frac{ \nabla f(x_{t}) }{\norm[]{}{ \nabla f(x_{t})} }) - \nabla f_{i}(x_{t}) ) \odot \sqrt{\eta_{t-1} }} \nonumber \\
&+\frac{\gamma \lambda_{3}^{2}}{2} \EE \norm[2]{}{\frac{1}{b} \sum_{i \in B } (\nabla f_{i}(x_{t} + \rho_{t} \frac{\sum _{i \in B} \nabla f_{i}(x_{t}) }{\norm[]{}{\sum _{i \in B} \nabla f_{i}(x_{t})} }) -\nabla f_{i}(x_{t} + \rho_{t} \frac{ \nabla f(x_{t}) }{\norm[]{}{ \nabla f(x_{t})} })  \odot \sqrt{\eta_{t-1} }}\\
&\leq  \frac{\gamma}{2}(\frac{1}{\lambda_{1}^{2}}+\frac{1}{\lambda_{2}^{2}} +\frac{1}{\lambda_{3}^{2}}) \EE \norm[2]{}{\od{(\nabla f(x_t)-\nabla f(z_{t})) }{\sqrt{\eta_{t-1}}}} +\frac{\gamma \lambda_{1}^{2}}{2} \norm[2]{}{\nabla f(x_{t}) \odot \sqrt{\eta_{t-1} }} \nonumber \\
&+ \frac{\gamma \lambda_{2}^{2} L^2 \rho_{t}^2 }{2\epsilon} +\frac{2 \lambda_{3}^2 \gamma L^2 \rho_{t}^{2}}{\epsilon}\\
&\leq  \frac{\gamma L^2}{2 \epsilon}(\frac{1}{\lambda_{1}^{2}}+\frac{1}{\lambda_{2}^{2}} +\frac{1}{\lambda_{3}^{2}}) \EE \norm[2]{}{z_{t}-x_{t} } +\frac{\gamma \lambda_{1}^{2}}{2} \norm[2]{}{\nabla f(x_{t}) \odot \sqrt{\eta_{t-1} }} \nonumber \\
&+ \frac{\gamma \lambda_{2}^{2} L^2 \rho_{t}^2 }{2\epsilon} +\frac{2 \lambda_{3}^2 \gamma L^2 \rho_{t}^{2}}{\epsilon}\\
&=  \frac{\gamma^3 L^2 \beta_{1}^2}{2 \epsilon (1-\beta_{1})^{2}}(\frac{1}{\lambda_{1}^{2}}+\frac{1}{\lambda_{2}^{2}} +\frac{1}{\lambda_{3}^{2}})  \norm[2]{}{m_{t-1} \odot \eta{t-1} } +\frac{\gamma \lambda_{1}^{2}}{2} \norm[2]{}{\nabla f(x_{t}) \odot \sqrt{\eta_{t-1} }} \nonumber \\
&+ \frac{\gamma \lambda_{2}^{2} L^2 \rho_{t}^2 }{2\epsilon} +\frac{2 \lambda_{3}^2 \gamma L^2 \rho_{t}^{2}}{\epsilon}\\
&\leq \frac{\gamma^3 L^2 \beta_{1}^2}{2 \epsilon (1-\beta_{1})^{2}}(\frac{1}{\lambda_{1}^{2}}+\frac{1}{\lambda_{2}^{2}} +\frac{1}{\lambda_{3}^{2}})  \frac{d G_{\infty}^{2}}{\epsilon^{2}} +\frac{\gamma \lambda_{1}^{2}}{2} \norm[2]{}{\nabla f(x_{t}) \odot \sqrt{\eta_{t-1} }} + \frac{\gamma  L^2 \rho_{t}^2 }{2\epsilon}(\lambda_{2}^{2} +4  \lambda_{3}^2) .
\end{align}
 \end{proof}

\begin{lemma} \label{lemma:term4-0}
 For the term defined in the algorithm, we have
 \begin{align}
 &\frac{L}{2} \EE \norm[2]{}{z_{t+1}- z_{t}} \leq \frac{L G^2 \gamma^{2} \beta_{1}^{2}}{(1-\beta_1)^{2}}\EE \norm[2]{}{\eta_{t} - \eta_{t-1}} \nonumber \\
 &+\gamma_{t}^{2}L (3\frac{1 + \beta}{\beta \epsilon}( \EE \norm[2]{}{ \nabla f(x_t ) \odot  \sqrt{\eta_{t-1}}} +\frac{L \rho_{t}^{2}}{\epsilon } + \frac{\sigma^{2}}{b \epsilon}) + (1+ \beta)G^2 \EE \norm[2]{}{\eta_{t} -\eta_{t-1} })
  \end{align}
\end{lemma}
 \begin{proof}
 \begin{align}
 &\frac{L}{2} \EE \norm[2]{}{z_{t+1}- z_{t}} \nonumber \\
 &=\frac{L}{2} \EE \norm[2]{}{\frac{\gamma \beta_1 }{1-\beta_1} \od{m_{t-1}}{(\eta_{t} - \eta_{t-1})} - \gamma g_{t} \odot \eta_{t} }\\
 &\leq \frac{L \gamma^{2} \beta_{1}^{2}}{(1-\beta_1)^{2}}\EE \norm[2]{}{\od{m_{t-1}}{(\eta_{t} - \eta_{t-1})}} +L \EE \norm[2]{}{\frac{\gamma_{t}}{b} \sum(\nabla f_{i}(x_t + \rho_{t} \frac{s_{t}}{\norm[]{}{s_{t}}})) \odot  \eta_{t} }\\
 &\leq \frac{L G^2 \gamma^{2} \beta_{1}^{2}}{(1-\beta_1)^{2}}\EE \norm[2]{}{\eta_{t} - \eta_{t-1}} +L \EE \norm[2]{}{\frac{\gamma_{t}}{b} \sum(\nabla f_{i}(x_t + \rho_{t} \frac{s_{t}}{\norm[]{}{s_{t}}})) \odot  \eta_{t} }\\
 &=\gamma_{t}^{2}L \EE \norm[2]{}{\frac{1}{b} \sum(\nabla f_{i}(x_t + \rho_{t} \frac{s_{t}}{\norm[]{}{s_{t}}})) \odot  \eta_{t-1} + \frac{1}{b} \sum(\nabla f_{i}(x_t + \rho_{t} \frac{s_{t}}{\norm[]{}{s_{t}}})) \odot  (\eta_{t} -\eta_{t-1}) }\nonumber \\
 &+\frac{L G^2 \gamma^{2} \beta_{1}^{2}}{(1-\beta_1)^{2}}\EE \norm[2]{}{\eta_{t} - \eta_{t-1}} \\
 &\leq \frac{L G^2 \gamma^{2} \beta_{1}^{2}}{(1-\beta_1)^{2}}\EE \norm[2]{}{\eta_{t} - \eta_{t-1}} + \gamma_{t}^{2}L ((1 + \frac{1}{\beta})\EE \norm[2]{}{\frac{1}{b} \sum(\nabla f_{i}(x_t + \rho_{t} \frac{s_{t}}{\norm[]{}{s_{t}}})) \odot  \eta_{t-1}} \nonumber \\
 &\;\;\;\;\;\;\;+ (1+\beta)\EE \norm[2]{}{\frac{1}{b} \sum(\nabla f_{i}(x_t + \rho_{t} \frac{s_{t}}{\norm[]{}{s_{t}}})) \odot  (\eta_{t} -\eta_{t-1}) })\\
 &\leq \gamma_{t}^{2}L ((1 + \frac{1}{\beta})\EE \norm[2]{}{\frac{1}{b} \sum(\nabla f_{i}(x_t + \rho_{t} \frac{s_{t}}{\norm[]{}{s_{t}}})) \odot  \eta_{t-1}} + (1+ \beta)G^2 \EE \norm[2]{}{\eta_{t} -\eta_{t-1} })\nonumber \\
 &+\frac{L G^2 \gamma^{2} \beta_{1}^{2}}{(1-\beta_1)^{2}}\EE \norm[2]{}{\eta_{t} - \eta_{t-1}} \\
 & \leq \gamma_{t}^{2}L ((1 + \frac{1}{\beta})\EE \norm[2]{}{\frac{1}{b} \sum(\nabla f_{i}(x_t + \rho_{t} \frac{s_{t}}{\norm[]{}{s_{t}}})) \odot  \sqrt{\eta_{t-1}}} \times \norm[2]{\infty}{\sqrt{\eta_{t-1}}} \nonumber \\
 &+ (1+ \beta)G^2 \EE \norm[2]{}{\eta_{t} -\eta_{t-1} }) +\frac{L G^2 \gamma^{2} \beta_{1}^{2}}{(1-\beta_1)^{2}}\EE \norm[2]{}{\eta_{t} - \eta_{t-1}}\\
 &\leq \gamma_{t}^{2}L (\frac{1 + \beta}{\beta \epsilon}\EE \norm[2]{}{\frac{1}{b} \sum(\nabla f_{i}(x_t + \rho_{t} \frac{s_{t}}{\norm[]{}{s_{t}}})) \odot  \sqrt{\eta_{t-1}}}  + (1+ \beta)G^2 \EE \norm[2]{}{\eta_{t} -\eta_{t-1} }) \nonumber \\
 &+\frac{L G^2 \gamma^{2} \beta_{1}^{2}}{(1-\beta_1)^{2}}\EE \norm[2]{}{\eta_{t} - \eta_{t-1}}\\
 &\leq \gamma_{t}^{2}L (3\frac{1 + \beta}{\beta \epsilon}\EE (\norm[2]{}{ \nabla f(x_t ) \odot  \sqrt{\eta_{t-1}}} + \norm[2]{}{(\frac{1}{b} \sum \nabla f_{i}(x_t )- \nabla f(x_t)) \odot  \sqrt{\eta_{t-1}}} \nonumber \\
 &+\norm[2]{}{\frac{1}{b} \sum(\nabla f_{i}(x_t + \rho_{t} \frac{\sum_{i \in B} \nabla f_{i}(x_t)}{\norm[]{}{\sum_{i \in B} \nabla f_{i}(x_t)}}) -\nabla f_{i}(x_t)) \odot  \sqrt{\eta_{t-1}}})+ (1+ \beta)G^2 \EE \norm[2]{}{\eta_{t} -\eta_{t-1} })\nonumber \\
 &+\frac{L G^2 \gamma^{2} \beta_{1}^{2}}{(1-\beta_1)^{2}}\EE \norm[2]{}{\eta_{t} - \eta_{t-1}} \\
 &\leq \gamma_{t}^{2}L (3\frac{1 + \beta}{\beta \epsilon}( \EE \norm[2]{}{ \nabla f(x_t ) \odot  \sqrt{\eta_{t-1}}} +\EE \norm[2]{}{\frac{1}{b} \sum(\nabla f_{i}(x_t + \rho_{t} \frac{\sum_{i \in B} \nabla f_{i}(x_t)}{\norm[]{}{\sum_{i \in B} \nabla f_{i}(x_t)}}) -\nabla f_{i}(x_t)) \odot  \sqrt{\eta_{t-1}}}  \nonumber \\
 &+ \frac{\sigma^{2}}{b \epsilon})+ (1+ \beta)G^2 \EE \norm[2]{}{\eta_{t} -\eta_{t-1} }) +\frac{L G^2 \gamma^{2} \beta_{1}^{2}}{(1-\beta_1)^{2}}\EE \norm[2]{}{\eta_{t} - \eta_{t-1}}\\
 &\leq \gamma_{t}^{2}L (3\frac{1 + \beta}{\beta \epsilon}( \EE \norm[2]{}{ \nabla f(x_t ) \odot  \sqrt{\eta_{t-1}}} +\frac{1}{\epsilon}\EE \norm[2]{}{\frac{1}{b} \sum(\nabla f_{i}(x_t + \rho_{t} \frac{\sum_{i \in B} \nabla f_{i}(x_t)}{\norm[]{}{\sum_{i \in B} \nabla f_{i}(x_t)}}) -\nabla f_{i}(x_t))}  \nonumber \\
 &+ \frac{\sigma^{2}}{b \epsilon})+ (1+ \beta)G^2 \EE \norm[2]{}{\eta_{t} -\eta_{t-1} }) + \frac{L G^2 \gamma^{2} \beta_{1}^{2}}{(1-\beta_1)^{2}}\EE \norm[2]{}{\eta_{t} - \eta_{t-1}}\\
 &\leq \gamma_{t}^{2}L (3\frac{1 + \beta}{\beta \epsilon}( \EE \norm[2]{}{ \nabla f(x_t ) \odot  \sqrt{\eta_{t-1}}} +\frac{1}{\epsilon b}\EE \sum\norm[2]{}{ \nabla f_{i}(x_t + \rho_{t} \frac{\sum_{i \in B} \nabla f_{i}(x_t)}{\norm[]{}{\sum_{i \in B} \nabla f_{i}(x_t)}}) -\nabla f_{i}(x_t)}  \nonumber \\
 &+ \frac{\sigma^{2}}{b \epsilon})+ (1+ \beta)G^2 \EE \norm[2]{}{\eta_{t} -\eta_{t-1} }) + \frac{L G^2 \gamma^{2} \beta_{1}^{2}}{(1-\beta_1)^{2}}\EE \norm[2]{}{\eta_{t} - \eta_{t-1}}\\
 &\leq \gamma_{t}^{2}L (3\frac{1 + \beta}{\beta \epsilon}( \EE \norm[2]{}{ \nabla f(x_t ) \odot  \sqrt{\eta_{t-1}}} +\frac{L \rho_{t}^{2}}{\epsilon } + \frac{\sigma^{2}}{b \epsilon}) + (1+ \beta)G^2 \EE \norm[2]{}{\eta_{t} -\eta_{t-1} }) \nonumber \\
 &+\frac{L G^2 \gamma^{2} \beta_{1}^{2}}{(1-\beta_1)^{2}}\EE \norm[2]{}{\eta_{t} - \eta_{t-1}}.
  \end{align}
 \end{proof}

\section{Additional Experiment Illustrations}\label{additional-exp}
\subsection{Experiment Illustrations}

\begin{figure*}[h]
   \centering
    \begin{subfigure}{0.23\linewidth}
    \includegraphics[width=\linewidth]{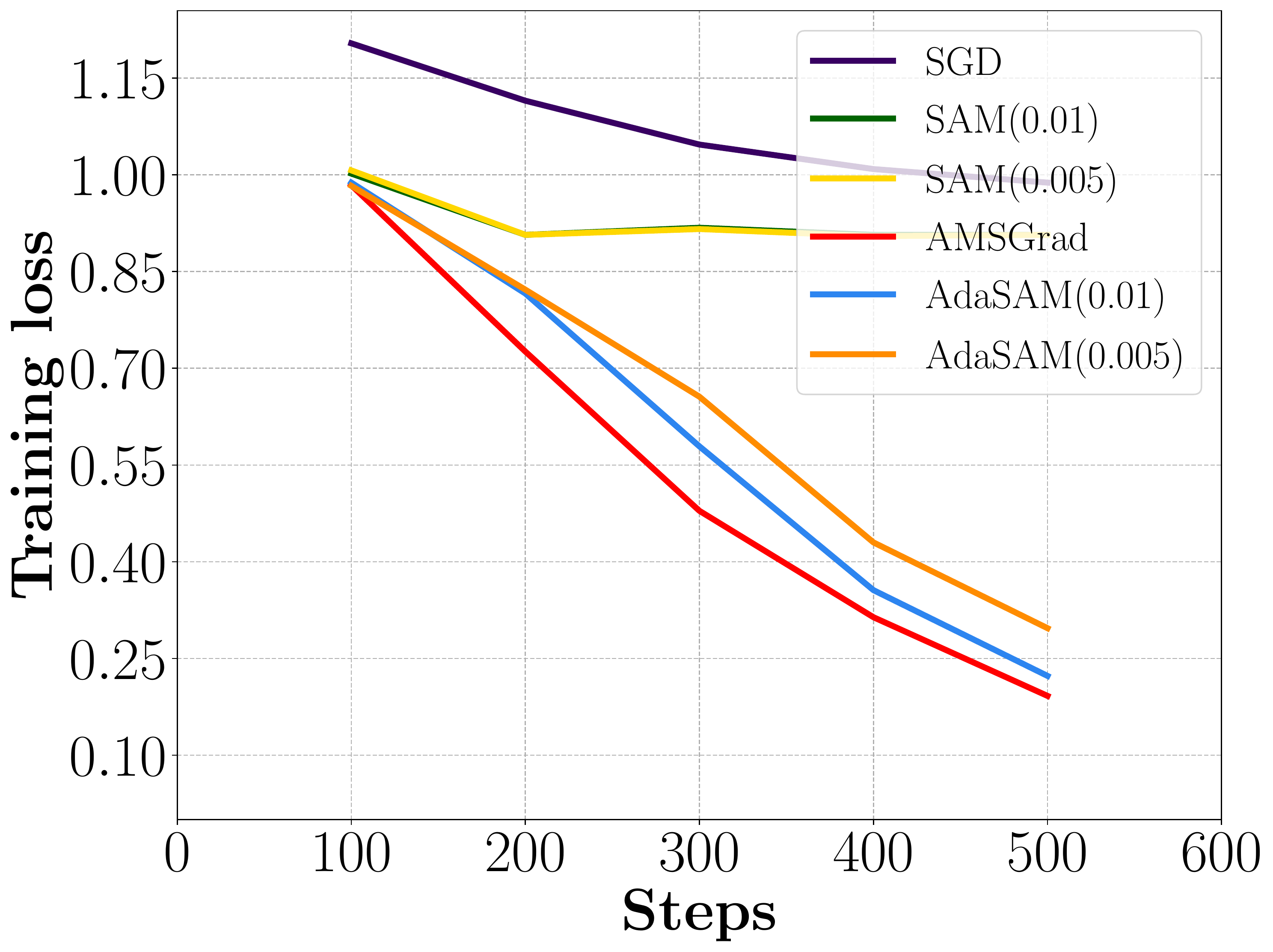}
    \caption{MRPC}
   \end{subfigure}\!\!
    \begin{subfigure}{0.23\linewidth}
    \includegraphics[width=\linewidth]{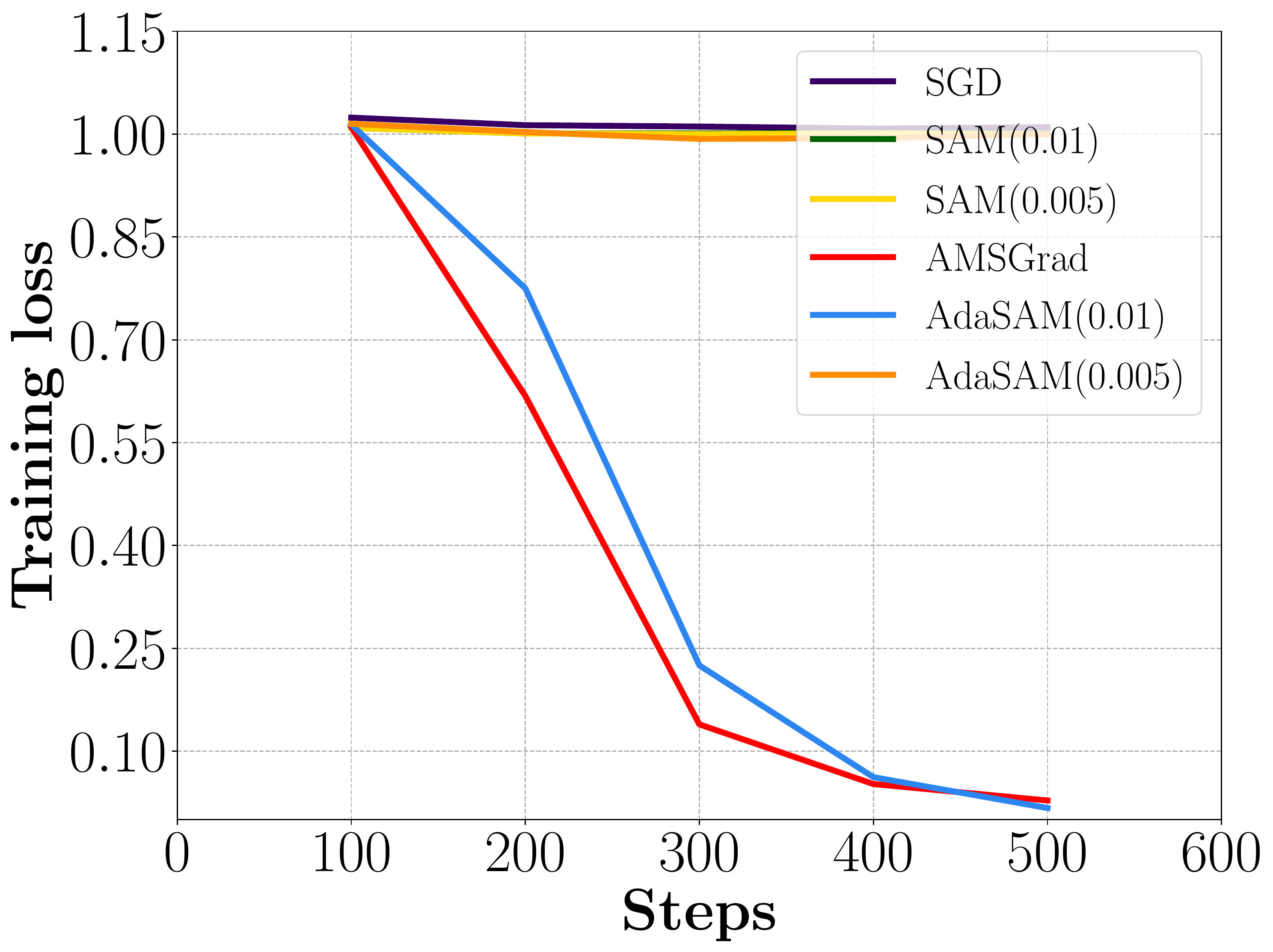}
    \caption{RTE}
   \end{subfigure}\!\!
    \begin{subfigure}{0.23\linewidth}
    \includegraphics[width=\linewidth]{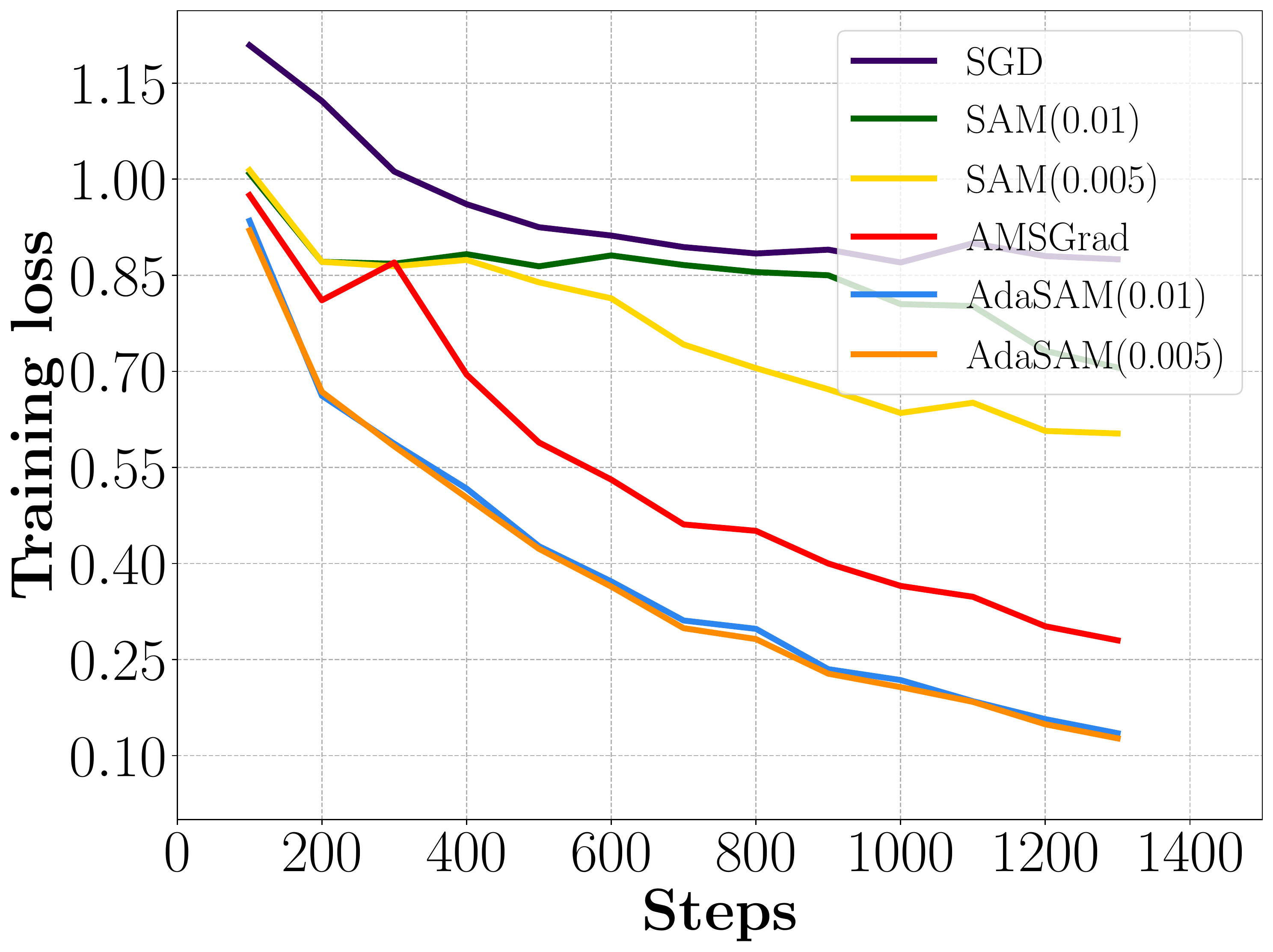}
    \caption{CoLA}
   \end{subfigure}
   \begin{subfigure}{0.23\linewidth}
    \includegraphics[width=\linewidth]{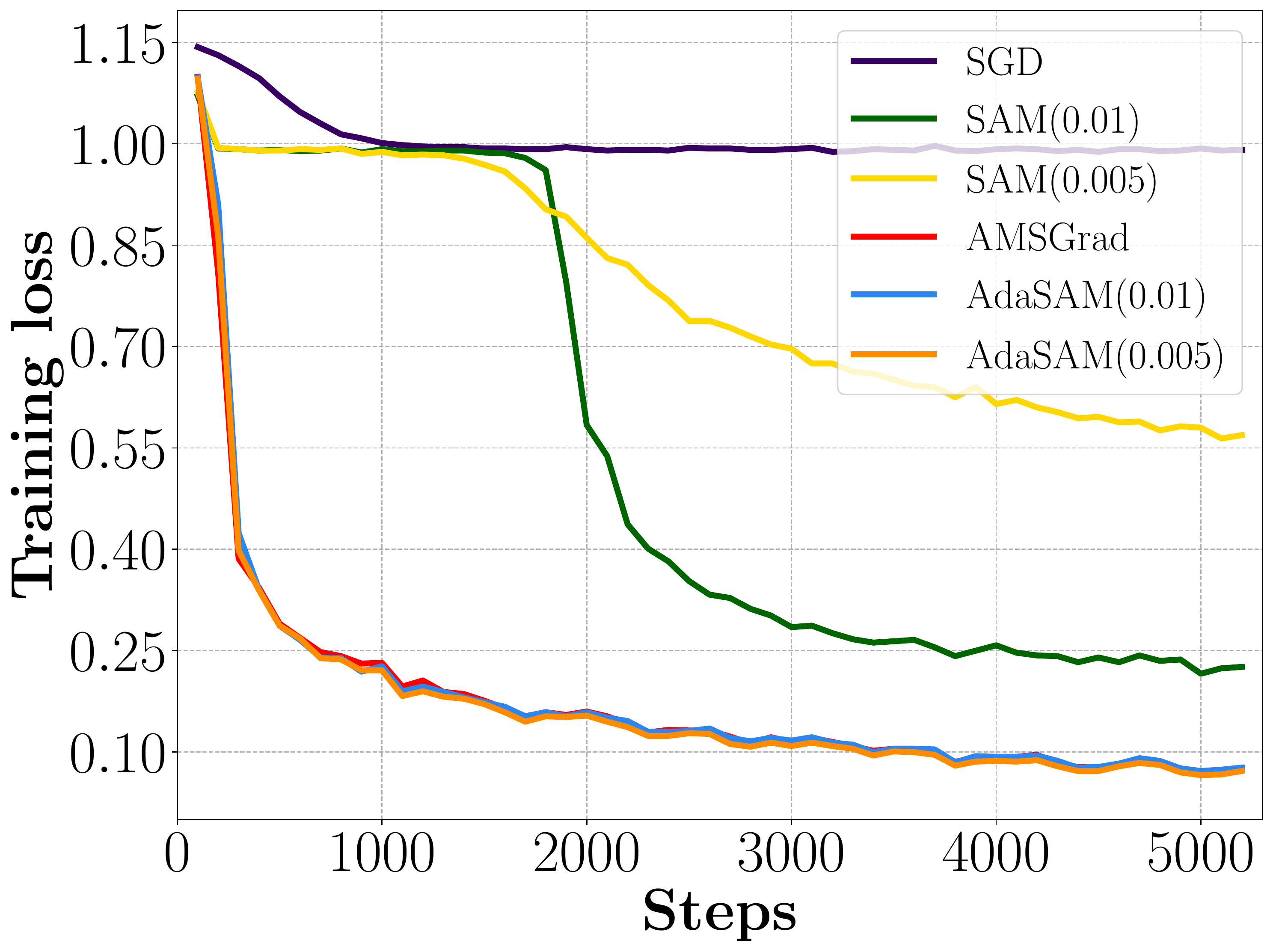}
    \caption{SST-2}
   \end{subfigure}\!\!
   
   \begin{subfigure}{0.23\linewidth}
    \includegraphics[width=\linewidth]{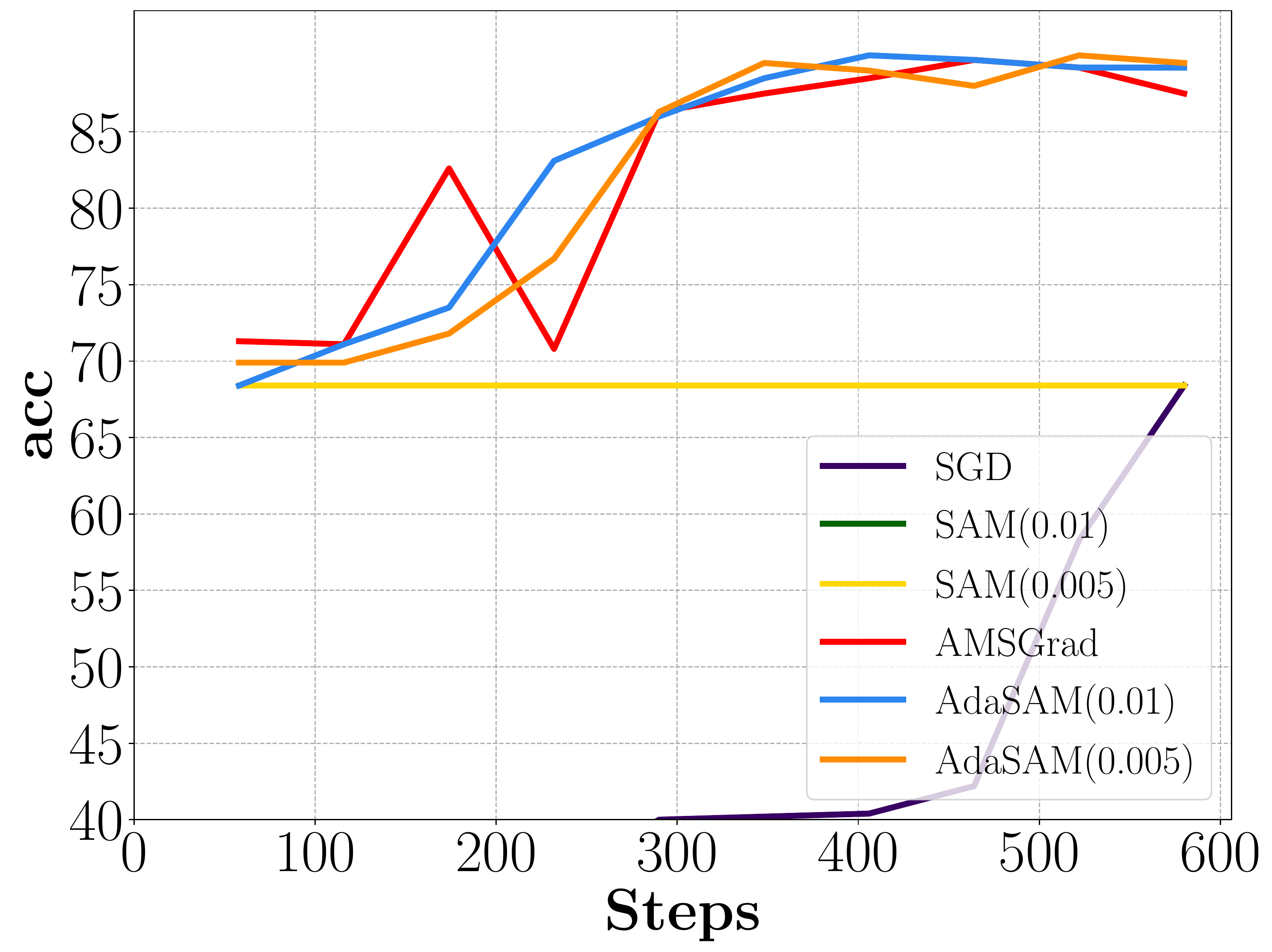}
    \caption{MRPC}
   \end{subfigure}\!\!
    \begin{subfigure}{0.23\linewidth}
    \includegraphics[width=\linewidth]{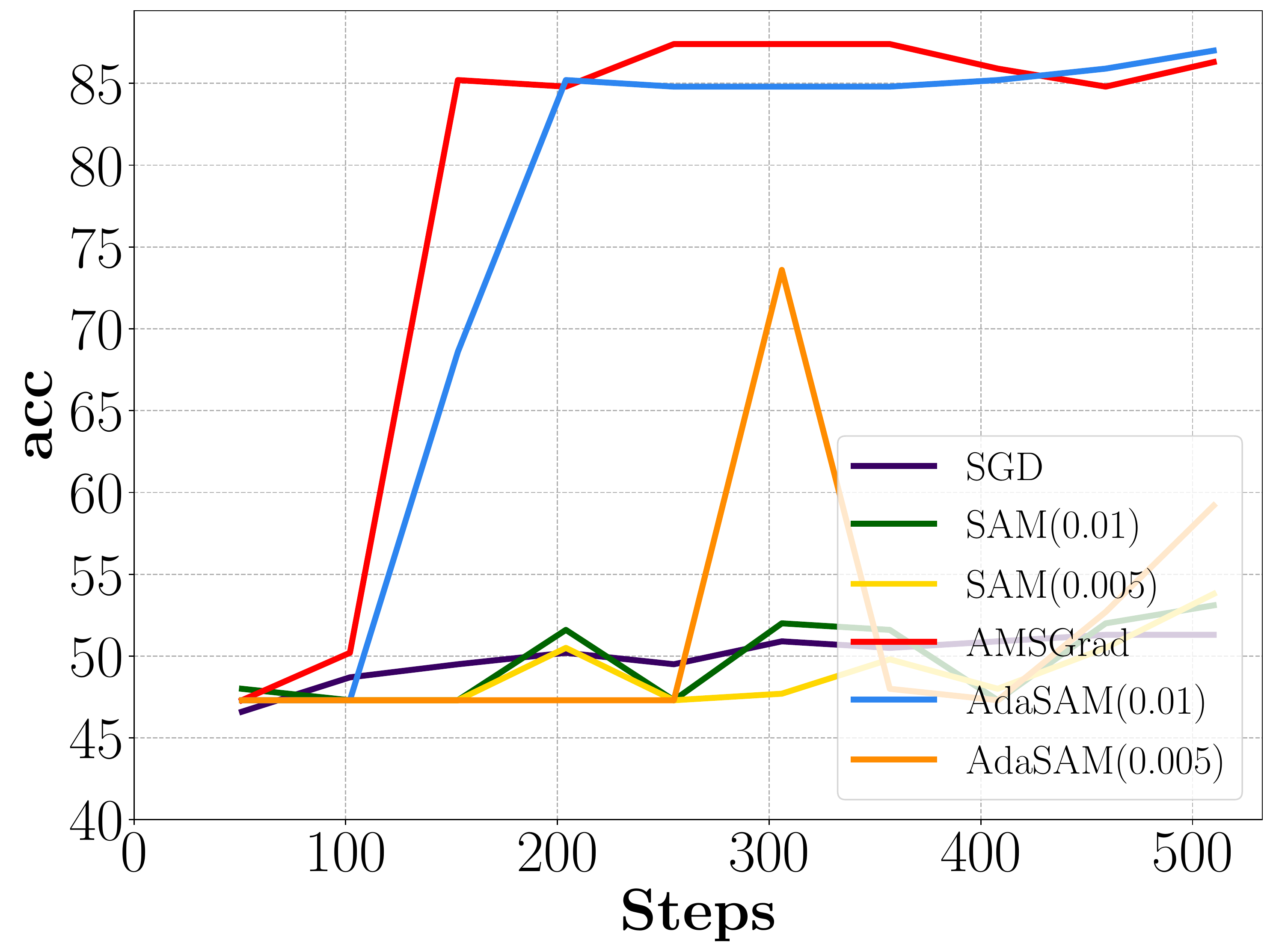}
    \caption{RTE}
   \end{subfigure}\!\!
    \begin{subfigure}{0.23\linewidth}
    \includegraphics[width=\linewidth]{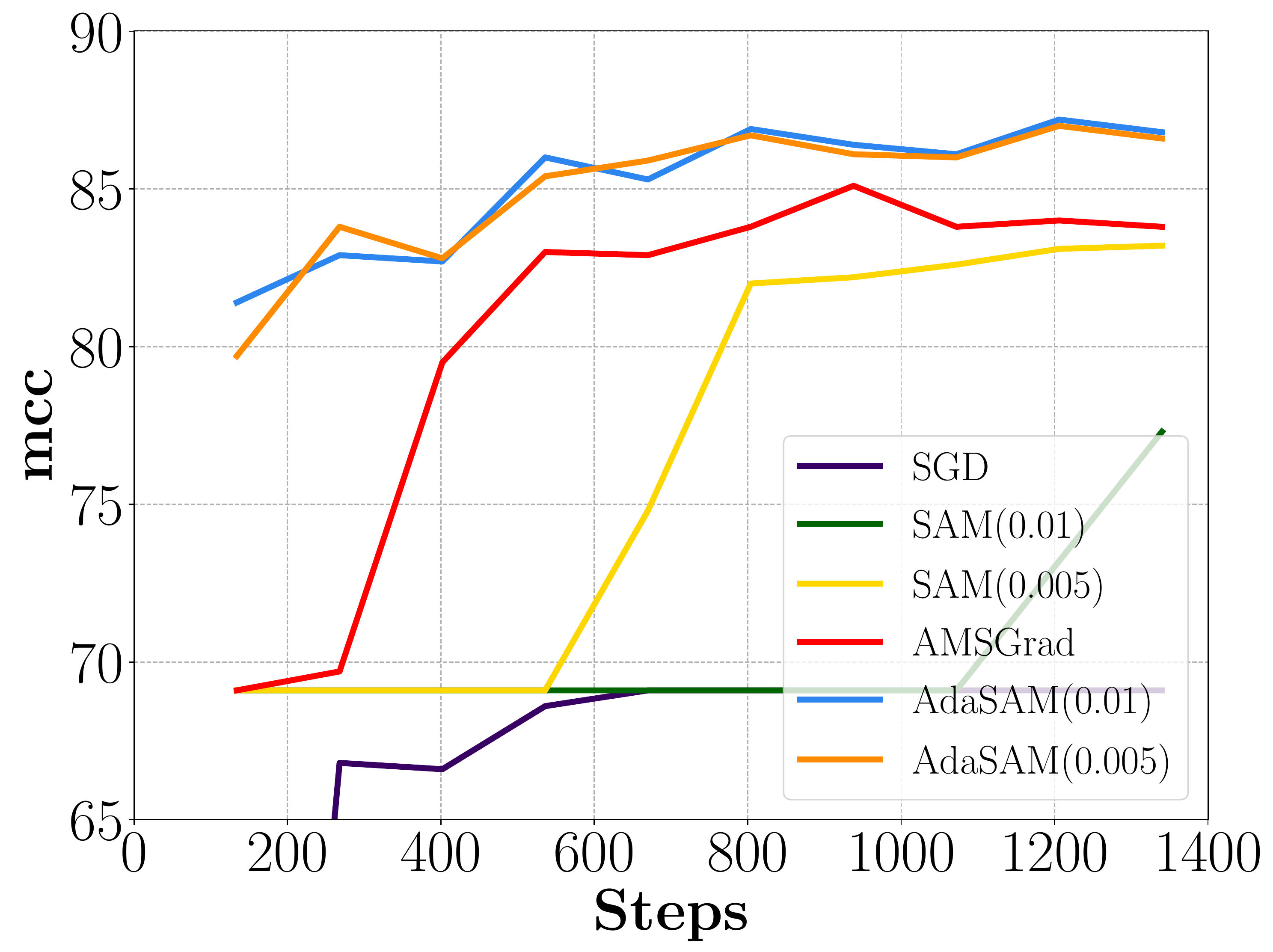}
    \caption{CoLA}
   \end{subfigure}
   \begin{subfigure}{0.23\linewidth}
    \includegraphics[width=\linewidth]{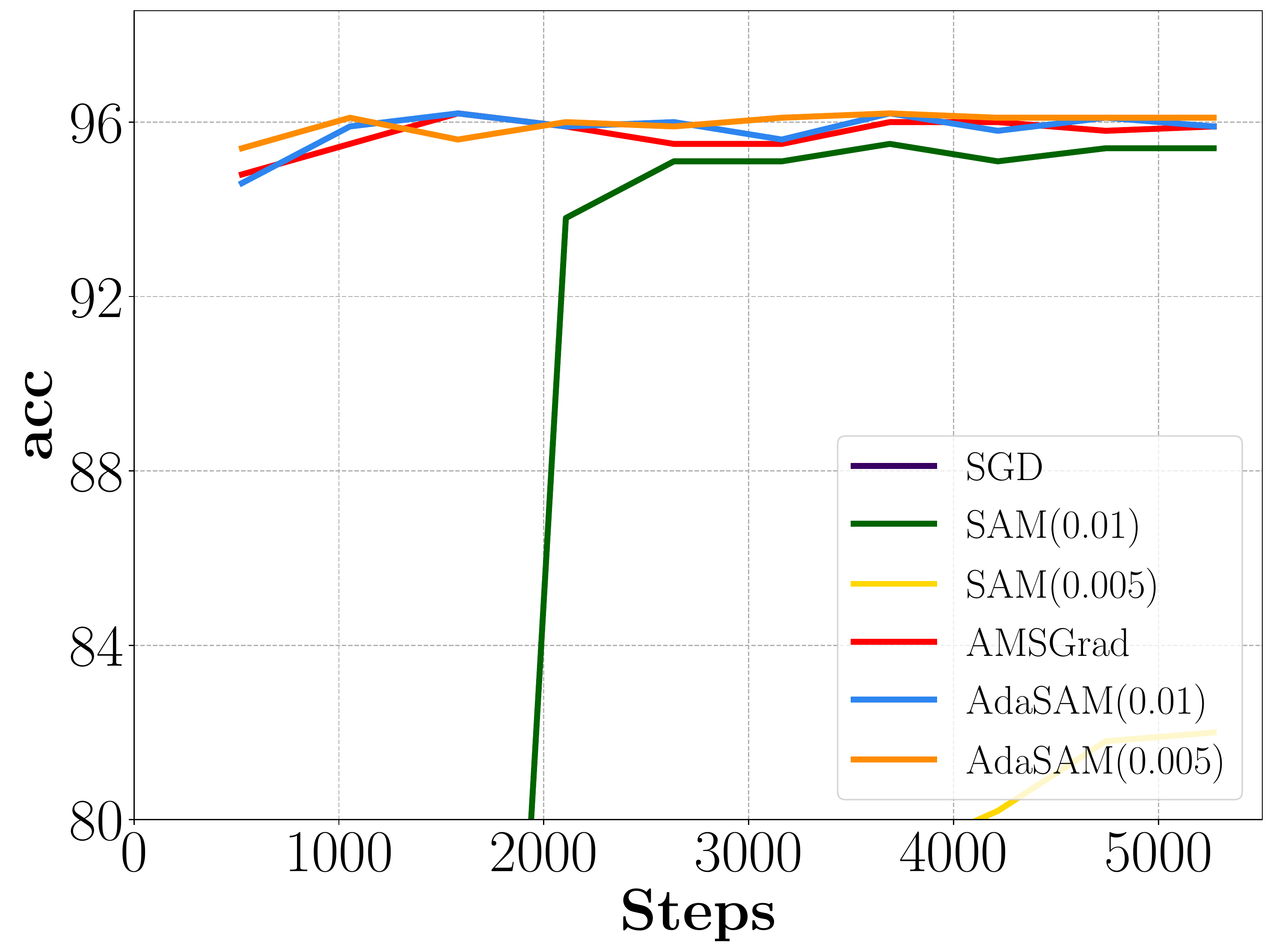}
    \caption{SST-2}
   \end{subfigure}\!\!

    \begin{subfigure}{0.23\linewidth}
    \includegraphics[width=\linewidth]{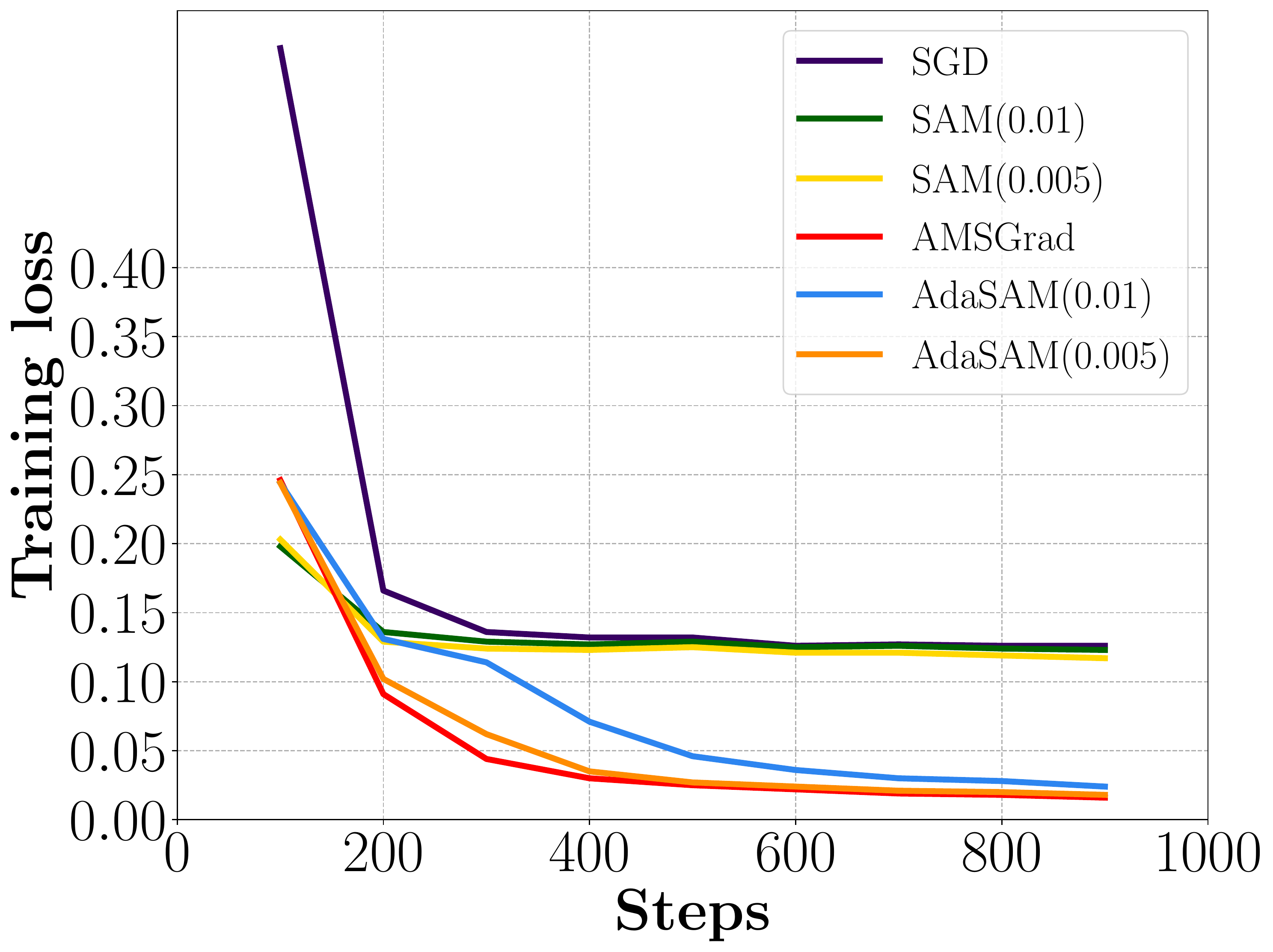}
    \caption{STS-B}
   \end{subfigure}\!\!
    \begin{subfigure}{0.23\linewidth}
    \includegraphics[width=\linewidth]{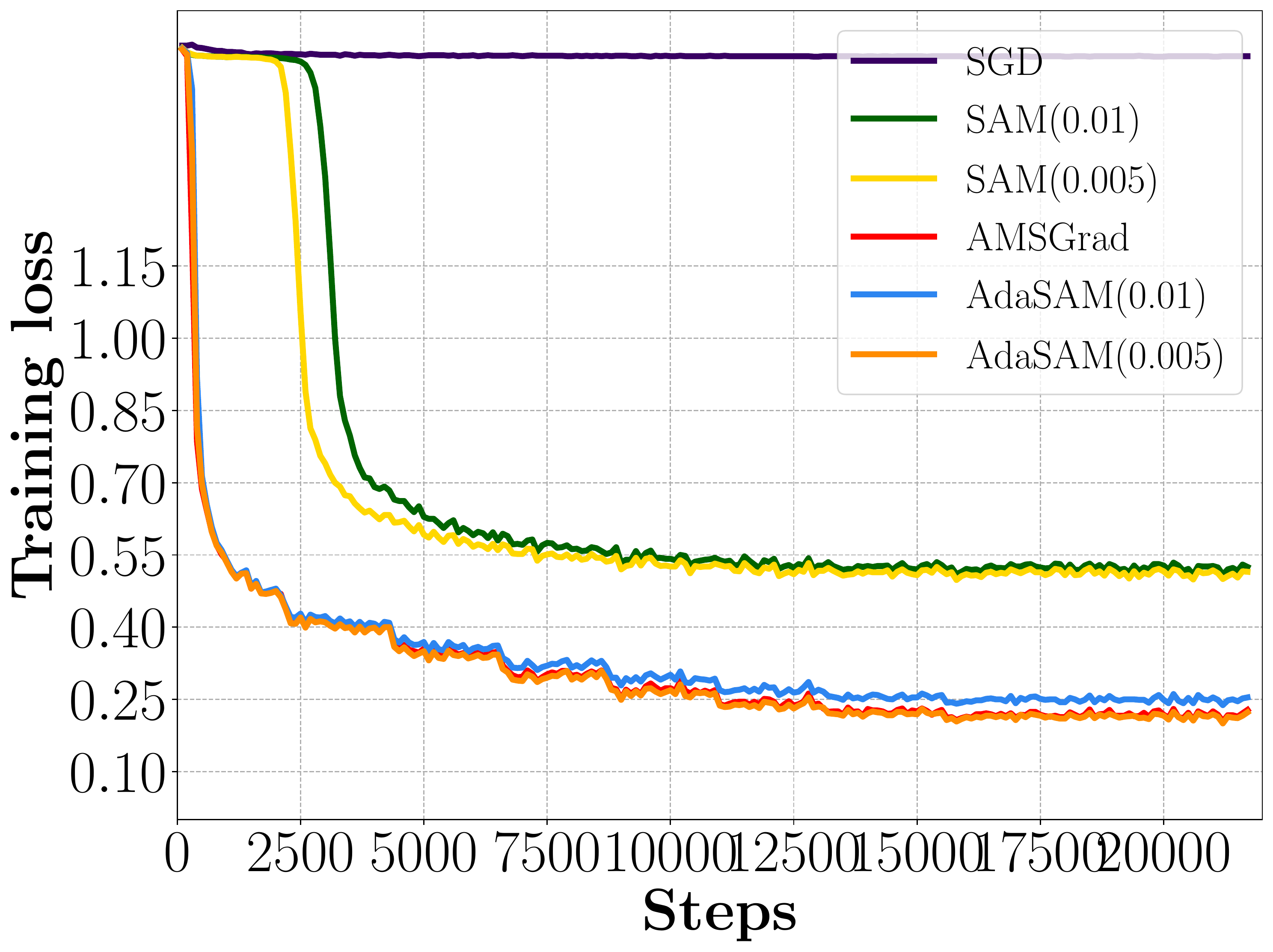}
    \caption{MNLI}
   \end{subfigure}\!\!
    \begin{subfigure}{0.23\linewidth}
    \includegraphics[width=\linewidth]{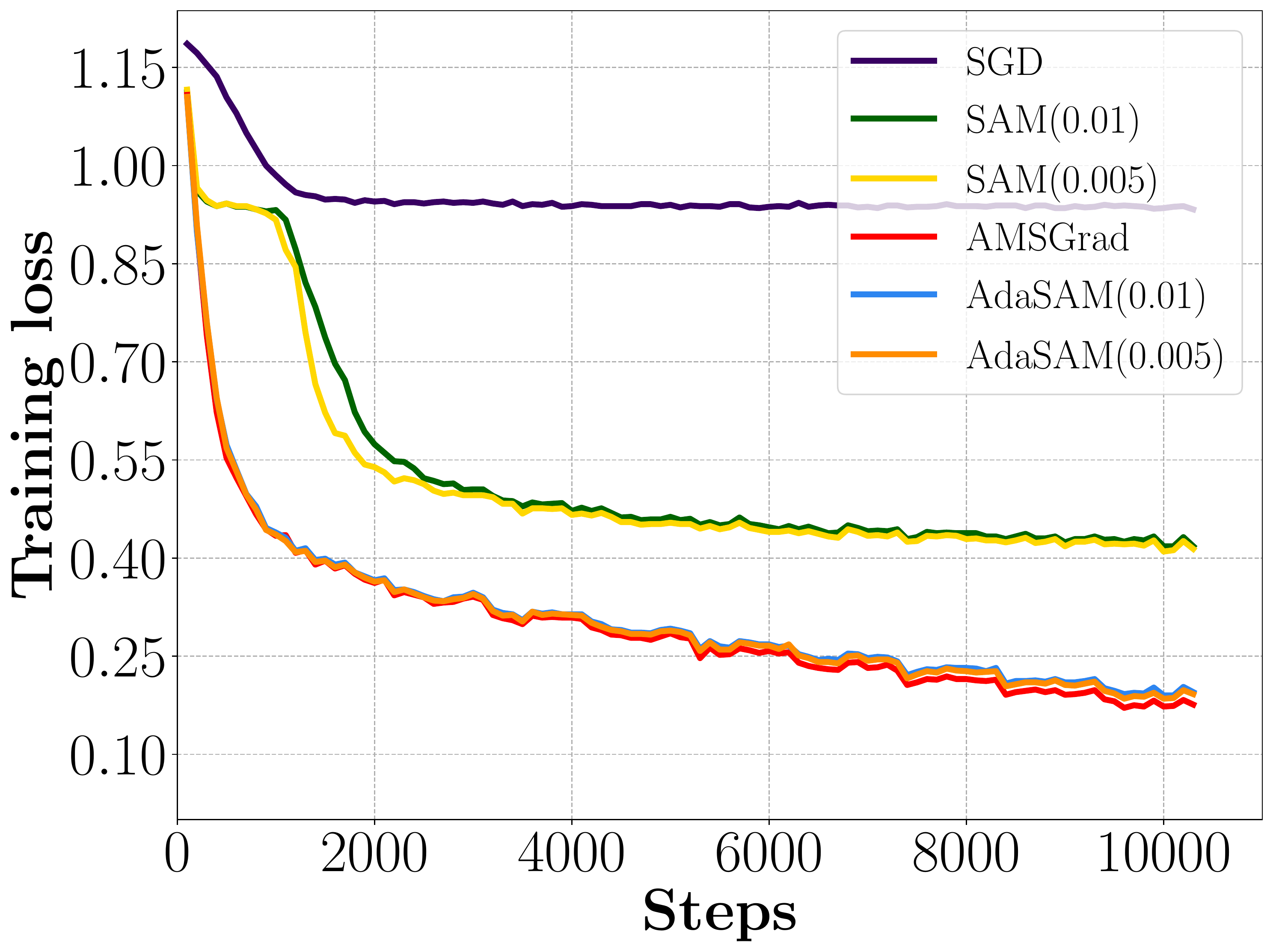}
    \caption{QQP}
  \end{subfigure}\!\!
    \begin{subfigure}{0.23\linewidth}
    \includegraphics[width=\linewidth]{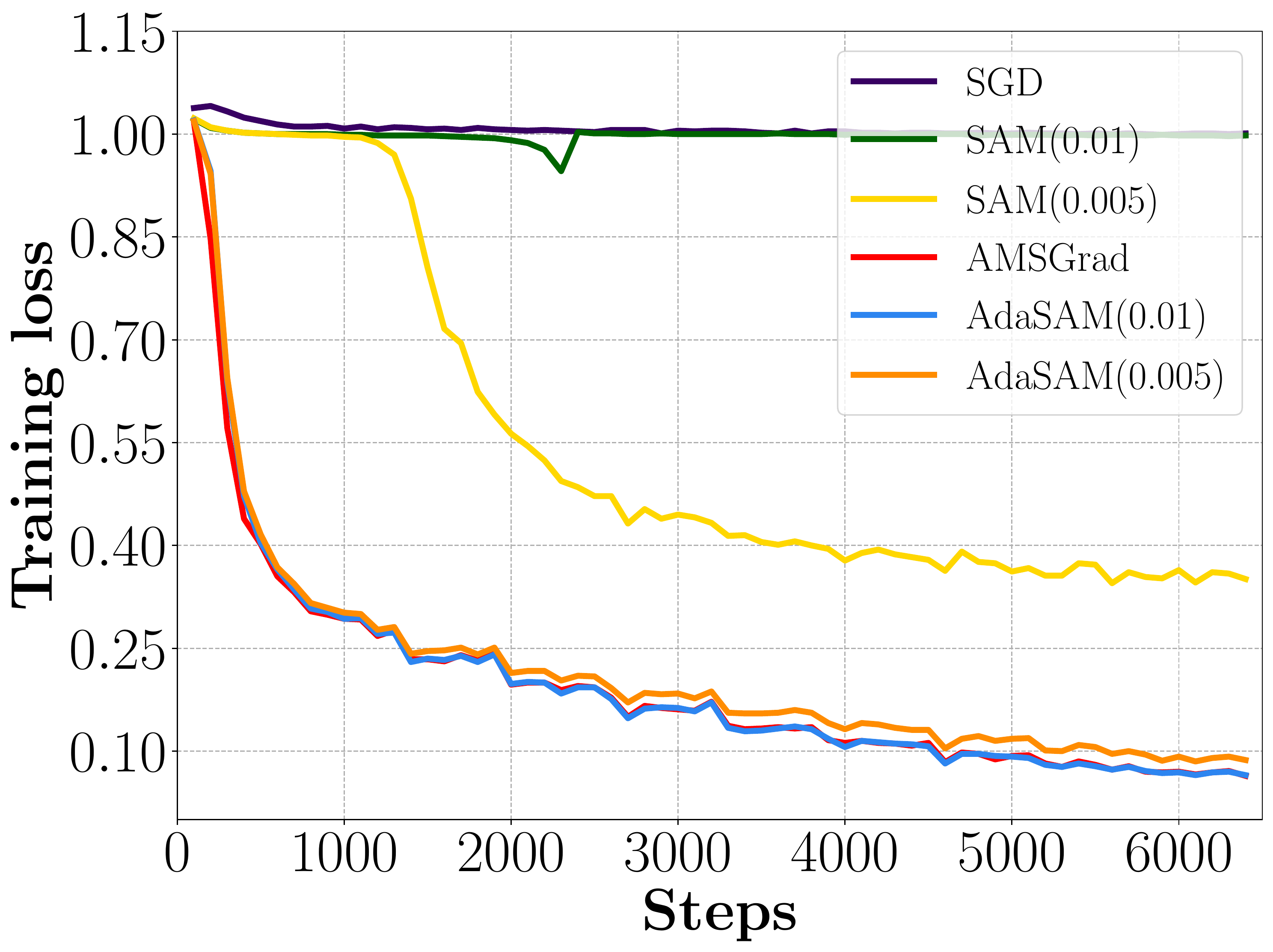}
    \caption{QNLI}
  \end{subfigure}\!\!
   
    \begin{subfigure}{0.23\linewidth}
    \includegraphics[width=\linewidth]{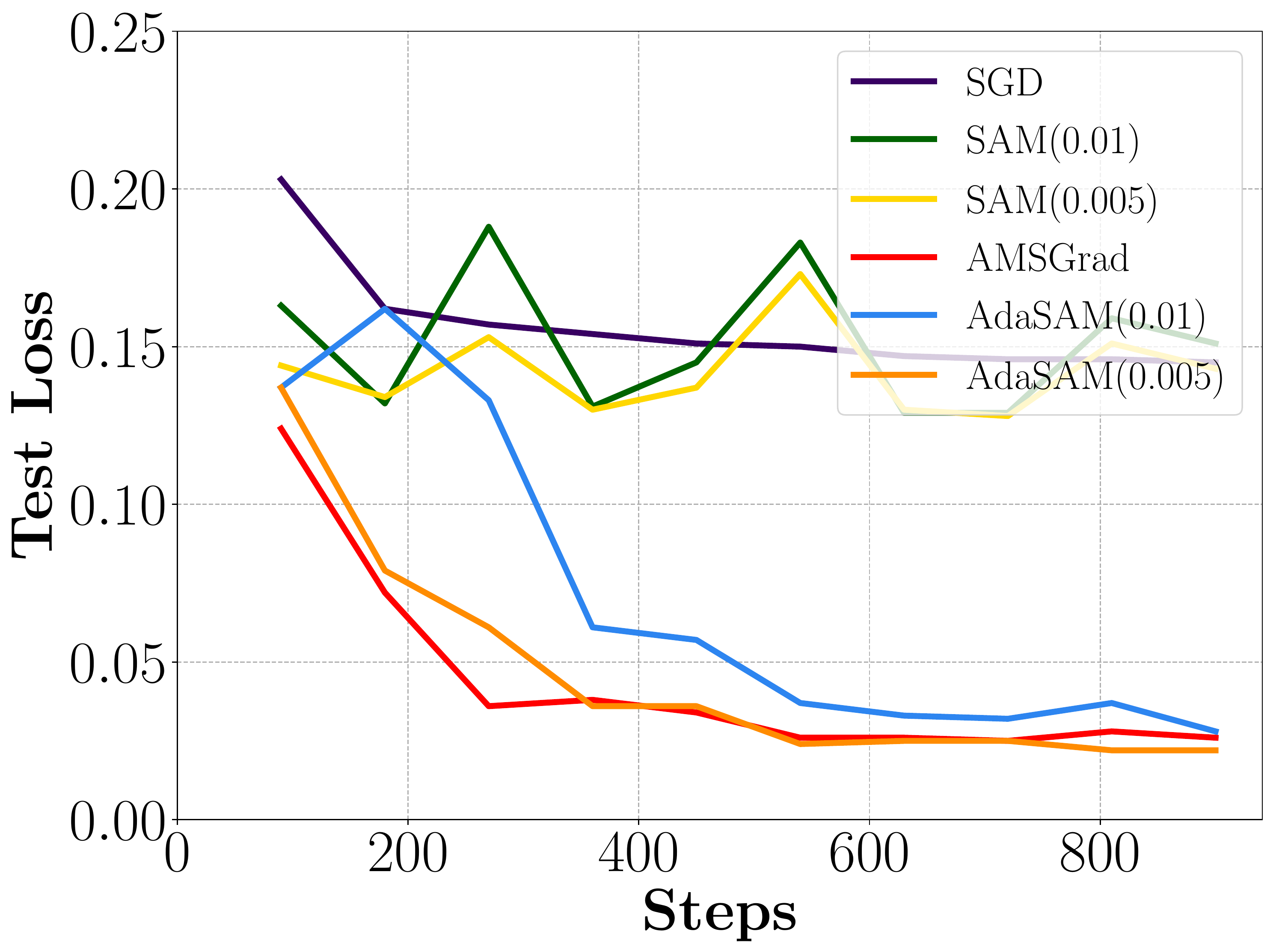}
    \caption{STS-B}
   \end{subfigure}\!\!
    \begin{subfigure}{0.23\linewidth}
    \includegraphics[width=\linewidth]{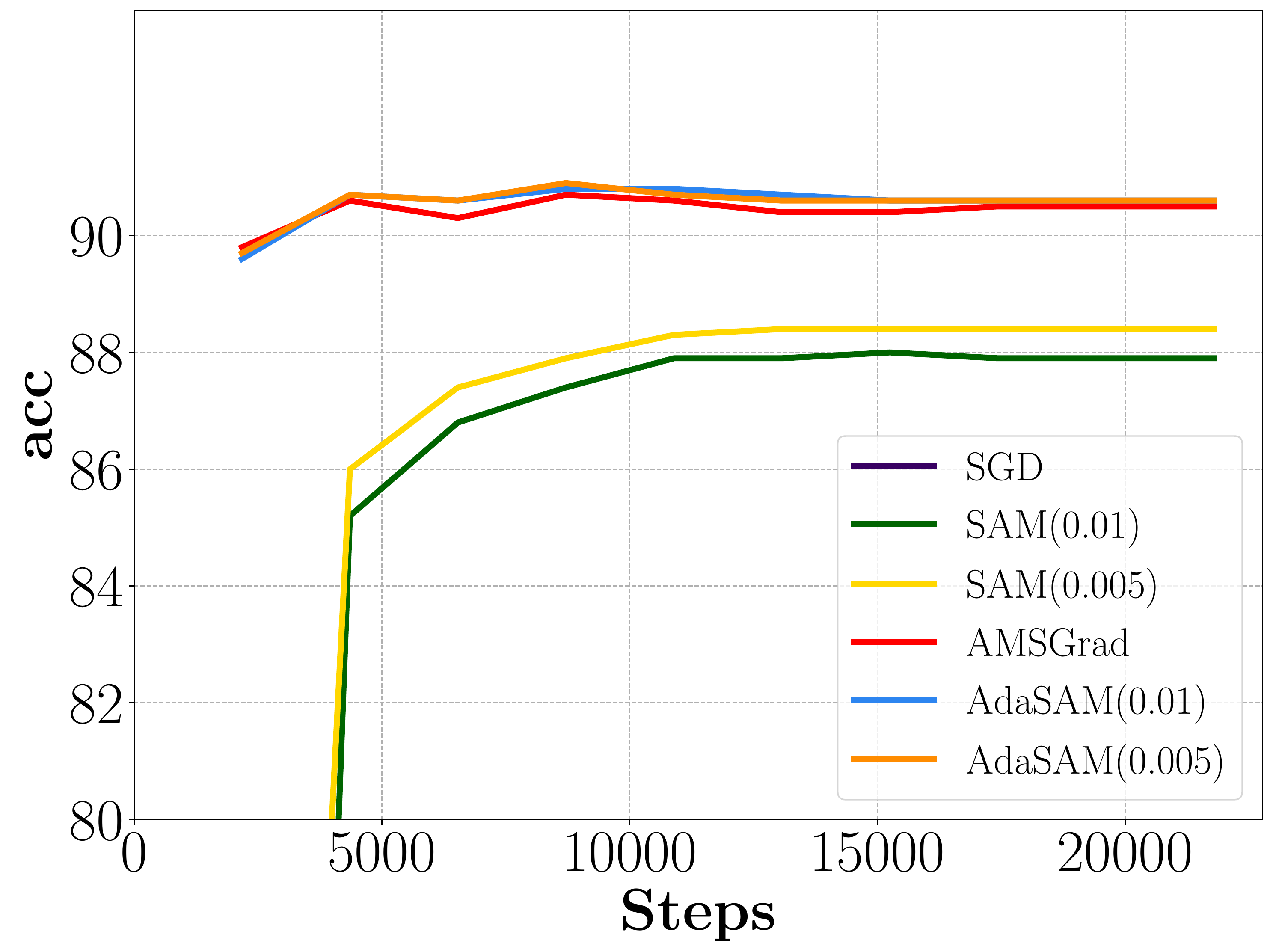}
    \caption{MNLI}
   \end{subfigure}\!\!
  \begin{subfigure}{0.23\linewidth}
    \includegraphics[width=\linewidth]{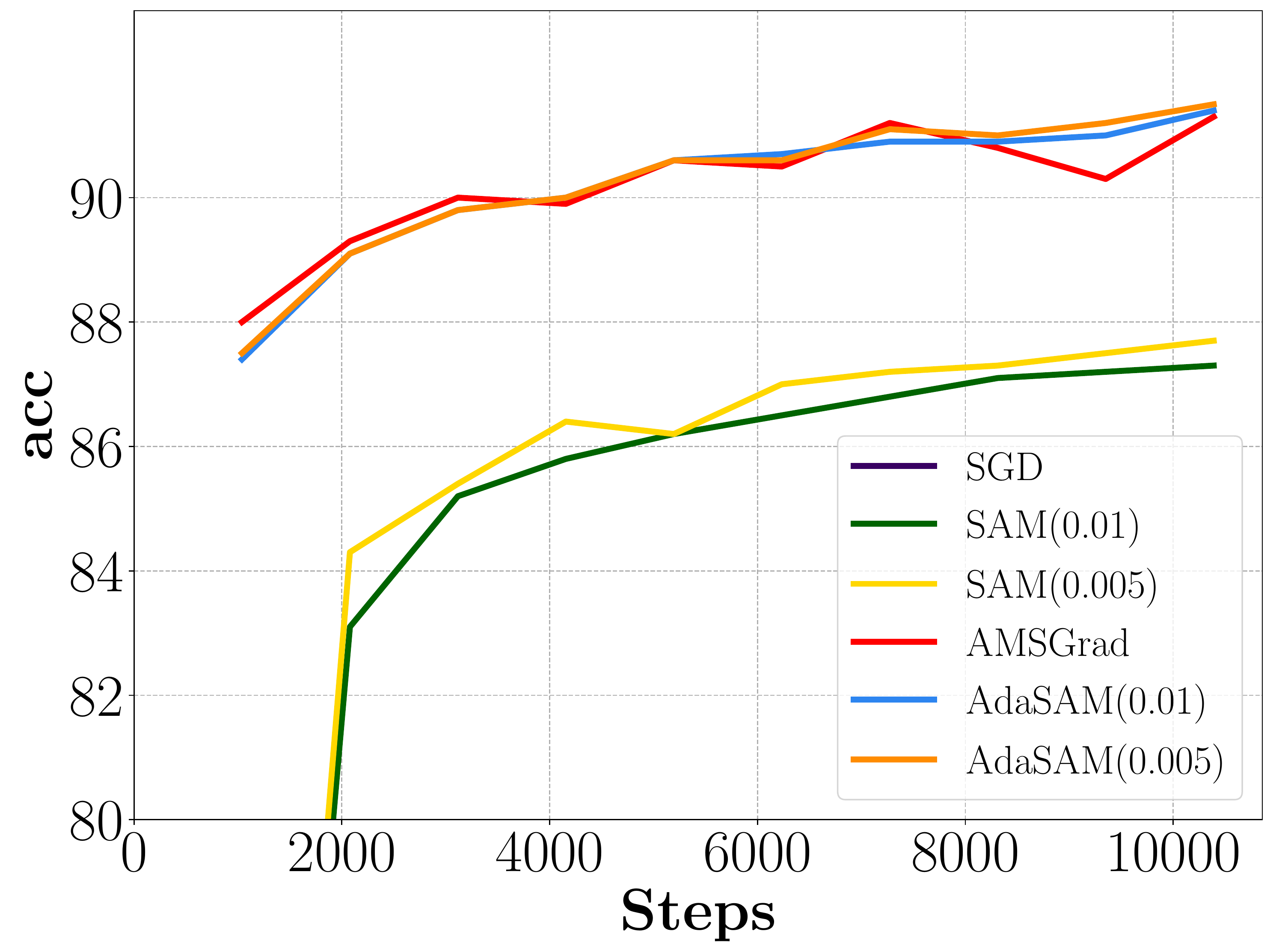}
    \caption{QQP}
  \end{subfigure}\!\!
    \begin{subfigure}{0.23\linewidth}
    \includegraphics[width=\linewidth]{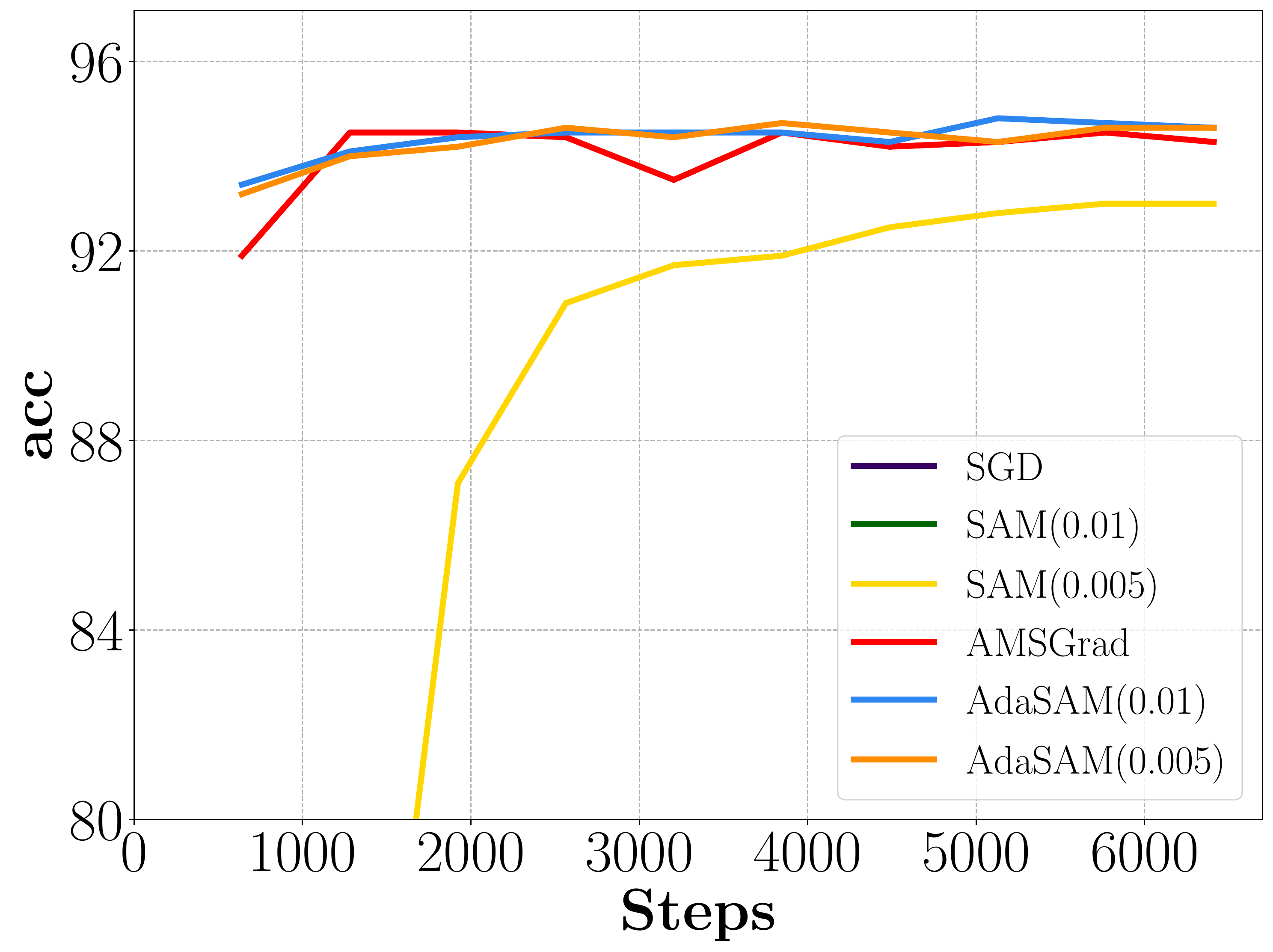}
    \caption{QNLI}
  \end{subfigure}\!\!
\caption{The loss and evaluation metric v.s. steps on MRPC, RTE, CoLA, SST-2, STS-B, MNLI, QQP and QNLI.($\beta_1 = 0$)}
\label{fig:s_exp_1}
 \vspace{-0.0cm}
\end{figure*}

In the ablation study, we conduct the experiments on the GLUE benchmark with AdaSAM, AMSGrad, SAM and SGD, respectively.
The optimizers do not have the momentum part ($\beta_1 = 0$).
As a supplement to Table \ref{tab:addlabel-3}, Figure \ref{fig:s_exp_1} show the detailed loss and evaluation metrics versus number of steps curves during training.
The loss curve of AdaSAM decreases faster than SAM and SGD in all tasks, and it has a similar decreasing speed as the AMSGrad.  The metric curve of AdaSAM and AMSGrad show that the adaptive learning rate method is better than SGD and SAM.
And AdaSAM decrease as faster as the AMSGrad in all tasks.

\end{document}